\documentclass[onefignum,onetabnum]{siamart171218}

\usepackage{lipsum}
\usepackage{amsfonts}
\usepackage{graphicx}
\usepackage{epstopdf}
\usepackage{algorithmic}
\ifpdf
  \DeclareGraphicsExtensions{.eps,.pdf,.png,.jpg}
\else
  \DeclareGraphicsExtensions{.eps}
\fi

\newsiamremark{remark}{Remark}
\newsiamremark{re}{Remark}
\newsiamremark{hypothesis}{Hypothesis}
\crefname{hypothesis}{Hypothesis}{Hypotheses}
\newsiamthm{claim}{Claim}
\newsiamthm{thm}{Theorem}

\headers{Deep Adaptive Dimension Reduction for Bayesian Inference}{Y. Wang, X. Wang, K. Tang, X. Wan, T. Zhou and C. Yang}

\title{Deep Adaptive Dimension Reduction for Bayesian Inference in Inverse Problems}

\author{
Yueyang Wang\thanks{School of Mathematical Sciences, Peking University. (\email{wangyueyang@stu.pku.edu.cn}, \email{xiliwang@stu.pku.edu.cn}, \email{chao\_yang@pku.edu.cn}). Correspondence to: Chao Yang.}
\and Xili Wang\footnotemark[1]
\and Kejun Tang\thanks{School of Sciences, Great Bay University. (\email{tangkj@gbu.edu.cn}).}
\and Xiaoliang Wan\thanks{Department of Mathematics, Louisiana State University. (\email{xlwan@lsu.edu}).}
\and Tao Zhou\thanks{SKLMS \& Institute of Computational Mathematics and Scientific/Engineering Computing, Academy of Mathematics and Systems Science, Chinese Academy of Sciences. (\email{tzhou@lsec.cc.ac.cn}).}
\and Chao Yang\footnotemark[1]
}
\usepackage{amsopn}

\usepackage[numbers]{natbib}
\usepackage{mathtools}
\usepackage{subfig}
\usepackage{booktabs}
\usepackage{multirow}
\usepackage{algorithm}

\newcommand{\mb}{\boldsymbol}

\newcommand{\comm}[1]{}

\newcommand{\squishlist}{
   \begin{list}{$\bullet$}
    { \setlength{\itemsep}{0pt}      \setlength{\parsep}{0pt}
      \setlength{\topsep}{-3pt}       \setlength{\partopsep}{0pt}
      \setlength{\listparindent}{-2pt}
      \setlength{\itemindent}{-5pt}
      \setlength{\leftmargin}{1em} \setlength{\labelwidth}{0em}
      \setlength{\labelsep}{0.5em} } }

\newcommand{\squishend}{
    \end{list}  }

\makeatletter
\newcommand*{\addFileDependency}[1]{
  \typeout{(#1)}
  \@addtofilelist{#1}
  \IfFileExists{#1}{}{\typeout{No file #1.}}
}
\makeatother

\ifpdf
\hypersetup{
  pdftitle={An Example Article},
  pdfauthor={D. Doe, P. T. Frank, and J. E. Smith}
}
\fi

\begin{document}

\maketitle

\begin{abstract}
    Solving high-dimensional PDE-governed inverse problems is often challenging due to complex non-Gaussian posterior distributions, expensive forward model evaluations, and misspecified prior information. To address these issues, we propose a deep adaptive dimension-reduction Bayesian inference framework based on the \emph{Variational Flow} (VF) model. Since standard normalizing flows are restricted by bijective mappings and cannot directly reduce dimensions, VF overcomes this limitation by integrating VAE-based nonlinear dimension reduction with dual normalizing flows for the latent prior and encoder. This design provides a strictly higher evidence lower bound than VAE and allows more flexible approximation of complex posterior distributions. We further introduce an iterative prior updating strategy that gradually moves the prior mean toward high-probability posterior regions, avoiding manual prior tuning. These components form a closed adaptive loop together with an adaptively fine-tuned Fourier Neural Operator (FNO) surrogate: VF generates posterior-concentrated samples to refine the surrogate, while the updated surrogate further improves posterior inference. Numerical experiments on a 100-dimensional Rosenbrock problem and three standard PDE-governed inverse problems show that our method delivers competitive or superior accuracy compared with MCMC, UKI, and SVGD baselines across all tested configurations, with the most pronounced advantages emerging in challenging scenarios such as high-noise observations and high-dimensional parameter spaces.

\end{abstract}



\section{Introduction}
Bayesian inverse problems governed by partial differential equations (PDEs) arise widely in science and engineering, including subsurface flow modeling, medical imaging, and climate science~\cite{stuart2010inverse, tarantola2005inverse}. The primary task is to recover unknown parameters from noisy, incomplete observations by characterizing the posterior distribution. In practice, the parameters are often high-dimensional, while the posterior is induced by a nonlinear and expensive PDE forward map and is typically non-Gaussian. This makes accurate and efficient posterior approximation a central challenge. Existing methods face clear limitations: MCMC methods~\cite{rudolf2018generalization, goodman2010ensemble, foreman2013emcee} are computationally expensive, Kalman-based methods such as EKI~\cite{iglesias2013ensemble} and UKI~\cite{huang2022iterated} rely on Gaussian approximations, and particle-based variational inference (VI) methods such as SVGD~\cite{liu2017stein} often degrade in high-dimensional spaces.

Forward evaluation is often the dominant computational bottleneck in Bayesian inverse problems. Neural operators such as FNO~\cite{li2020fourier} and DeepONet~\cite{lu2021learning} provide efficient surrogates, but models trained on prior samples can incur large out-of-distribution (OOD) errors in posterior regions. A related and often overlooked issue is the use of a static prior: in high-dimensional spaces lacking explicit domain knowledge, a static prior can rigidly confine inference to its initial subspace~\cite{kaipio2005statistical, dunlop2017hierarchical}. A narrow prior restricts the search space and may miss the true posterior, whereas an overly broad prior provides insufficient regularization, exacerbates ill-posedness, and can push PDE solvers into unphysical regimes. This motivates a posterior-informed view of prior construction: a useful prior should be compatible with the structure of the posterior induced by the data and the forward model. Recent adaptive surrogate methods~\cite{gao2024adaptive} partially address distribution shift, but they typically rely on Gaussian posterior approximations (via UKI) and expensive sample-selection procedures.

To address these challenges, we propose a \textbf{deep adaptive dimension-reduction Bayesian inference framework} with three components. First, we introduce \emph{Variational Flow} (VF), a latent-variable model that combines nonlinear dimension reduction with flow-based density modeling to approximate complex non-Gaussian posteriors. Second, we develop an \emph{iterative prior updating} strategy that progressively shifts the prior mean toward the posterior, improving stability and reducing prior misspecification. Third, we use \emph{adaptive surrogate fine-tuning} based on local perturbations around posterior-concentrated VF samples, avoiding expensive sample-selection heuristics. These components form a mutually reinforcing loop: VF guides the surrogate toward the posterior region, while the refined surrogate improves posterior inference.

\subsection{Related Work}

\paragraph{Bayesian inference for PDE-governed inverse problems}
Classical MCMC methods, such as preconditioned Crank--Nicolson (pCN)~\cite{rudolf2018generalization} and affine-invariant ensemble samplers~\cite{goodman2010ensemble, foreman2013emcee}, provide asymptotically exact posterior samples but are often prohibitively expensive for PDE-constrained inverse problems due to costly forward solves. Kalman-based methods, including the Ensemble Kalman Inversion (EKI)~\cite{iglesias2013ensemble} and the Unscented Kalman Inversion (UKI)~\cite{huang2022iterated}, are more efficient but rely on Gaussian posterior structure, limiting accuracy for multimodal or strongly nonlinear posteriors. Stein variational gradient descent (SVGD)~\cite{liu2016stein, liu2017stein} offers a flexible particle-based alternative, yet its reliance on gradient information and diminishing particle efficiency in high-dimensional spaces limit its applicability.

\paragraph{Neural operator surrogates and adaptive training}
Neural operators, including FNO~\cite{li2020fourier}, DeepONet~\cite{lu2021learning}, and related architectures~\cite{wang2021learning, kontolati2023learning}, provide efficient surrogates for PDE forward evaluations. However, surrogates trained only on prior samples often suffer from distribution shift and generalize poorly in posterior regions. Adaptive retraining or fine-tuning of the surrogate can mitigate this issue~\cite{gao2024adaptive, wang2024deep, cao2023residual, tang_das, li2023surrogate}. A closely related work~\cite{gao2024adaptive} initializes the surrogate on prior samples, then alternates between UKI-based posterior estimation and surrogate refinement by selecting an observation-matching anchor from posterior samples and greedily adding nearby yet prediction-diverse samples for retraining.

\paragraph{Deep generative models}
Deep generative models have been widely applied to complex problems such as image generation~\cite{dhariwal2021diffusion, bond2021deep}, protein conformation~\cite{janson2024transferable, wang2025estimating, wang2024protein}, statistical physics~\cite{wu2020stochastic, tang2022adaptive}, and probability approximation~\cite{tang2020deep, papamakarios2021normalizing}. Variational autoencoders (VAEs)~\cite{kingma2013auto} are scalable but can be limited by Gaussian assumptions. Consequently, many dimension-reduction approaches for inverse problems~\cite{xia2022bayesian, goh2019solving, xia2023vi} perform inference in a latent space following the VAE framework. Normalizing flows~\cite{rezende2015variational} support exact density evaluation and have also been applied to inverse problems~\cite{papamakarios2021normalizing}. Importantly, normalizing flows have been shown to be universal probability approximators~\cite{teshima2020coupling, kong2020expressive}, but their bijective structure generally does not provide intrinsic dimensionality reduction. The VAE-flow model~\cite{kingma2016improved} focuses on improving latent expressiveness with autoregressive flow. Our Variational Flow (VF) model leverages a dual-flow structure to enhance not only the latent distribution but also the encoder’s flexibility, and is compatible with our theoretical analysis.

\subsection{Main Contributions}
The main contributions of this paper are:
\begin{itemize}
    \item We propose \textbf{Variational Flow (VF)}, a flow-enhanced latent model for nonlinear dimension reduction and expressive posterior approximation. We not only theoretically prove that VF achieves strictly higher ELBO than VAE, but show it can capture complex non-Gaussian posteriors, including on a \textbf{100D Rosenbrock benchmark}.
    \vspace{0.5em}
    \item We propose an \textbf{iterative prior updating} strategy that adaptively refines the prior mean during inference, mitigating prior misspecification without requiring problem-specific prior design.
    \vspace{0.5em}
    \item We propose a \textbf{deep adaptive dimension-reduction Bayesian inference framework} alternating between VF updating and surrogate fine-tuning, which achieves superior performance on \textbf{PDE-governed inverse problems} like Darcy flow and Navier-Stokes equations.
\end{itemize}

\subsection{Organization}
The remainder of this paper is organized as follows. Section \ref{sec:preliminaries} introduces the preliminaries on Bayesian inverse problems. Section \ref{variational_flow} details the proposed Variational Flow model. In Section \ref{sec:method}, we present the deep adaptive Bayesian inference framework, including iterative prior updating and adaptive surrogate fine-tuning. Section \ref{sec:experiments} demonstrates the numerical performance of our framework. Finally, we conclude the paper in Section \ref{sec:conclusions}.

\section{Preliminaries}
\label{sec:preliminaries}
\subsection{Bayesian Inverse Problems}
Let $\Omega_s \subset \mathbb{R}^n$ be a bounded spatial domain and $\mb{\xi} \in \Omega_p \subset \mathbb{R}^d$ be the unknown parameter vector. We consider physical systems governed by parametric partial differential equations:
\begin{align}
	\mathcal{L}(u(\mb{x};m_{\mb{\xi}}(\mb{x})); m_{\mb{\xi}}(\mb{x})) &= s(\mb{x}), &&\forall (\mb{x}, \mb{\xi}) \in \Omega_s \times \Omega_p, \label{spdexi1}\\
	\mathcal{B}(u(\mb{x};m_{\mb{\xi}}(\mb{x})); m_{\mb{\xi}}(\mb{x})) &= g(\mb{x}), &&\forall (\mb{x}, \mb{\xi}) \in \partial \Omega_s \times \Omega_p, \label{spdexi2}
\end{align}
where $m_{\mb{\xi}}(\mb{x})$ is the coefficient field characterized by $\mb{\xi}$. Here, $\mathcal{L}$ and $\mathcal{B}$ denote the differential and boundary operators respectively, while $s$ and $g$ specify the source and boundary conditions. For any $\mb{\xi} \in \Omega_p$, we assume the system~\eqref{spdexi1}-\eqref{spdexi2} is well-posed, admitting a unique solution $u$ in a state space $\mathcal{U}$.

The goal of inverse problems is to infer the unknown parameters $\mb{\xi} \in \Omega_p$ from noisy measurements $\mb{y} \in \mathbb{R}^m$. Let $\mathcal{S}: \Omega_p \to \mathcal{U}$ be the forward parameter-to-state map, and $\mathcal{O}: \mathcal{U} \to \mathbb{R}^m$ be the observation operator. The composite parameter-to-observation map is defined as $\mathcal{G} = \mathcal{O} \circ \mathcal{S}$. We consider the additive Gaussian noise model:
\begin{equation}
    \mb{y} = \mathcal{G}(\mb{\xi}) + \mb{\eta}, \qquad \mb{\eta} \sim \mathcal{N}(\mb{0}, \mb{\Sigma}_{\eta}).
\end{equation}
The negative log-likelihood function is defined as:
\begin{equation}
    \Phi(\mb{\xi}, \mb{y}) := \frac{1}{2} \left\| \mb{\Sigma}_{\eta}^{-1/2} \left( \mb{y} - \mathcal{G}(\mb{\xi}) \right) \right\|^2.
    \label{eq:misfit}
\end{equation}
Since typically $m \ll d$, the inverse problem is ill-posed. The Bayesian framework regularizes this by modeling parameters as random variables and introducing a prior distribution, assumed here to be a Gaussian $\pi_0(\mb{\xi}) = \mathcal{N}(\mb{\mu}_0, \mb{\Sigma}_0)$ supported on $\Omega_p$. According to Bayes' theorem, the posterior distribution is:
\begin{equation}
\begin{split}
    \pi(\mb{\xi} | \mb{y}) &= \frac{1}{Z} \exp \left( -\Phi(\mb{\xi}, \mb{y}) \right) \pi_0(\mb{\xi}), \\
    &\propto \exp\left( 
    -\frac{1}{2} \left\| \mb{\Sigma}_\eta^{-\frac{1}{2}} (\mb{y} - \mathcal{G}(\mb{\xi})) \right\|^2 
    -\frac{1}{2} \left\| \mb{\Sigma}_0^{-\frac{1}{2}} (\mb{\xi} - \mb{\mu}_0) \right\|^2
    \right),
\end{split}
\label{eq:posterior}
\end{equation}
where $Z = \int_{\Omega_p} \exp(-\Phi(\mb{\xi}, \mb{y})) \pi_0(\mb{\xi}) d\mb{\xi}$ is the intractable normalization constant (model evidence).

To achieve numerical tractability, we parameterize the spatially dependent Gaussian random field $m_{\mb{\xi}}(\mb{x})$ using a truncated Karhunen-Lo\`{e}ve (KL) expansion, which maps a finite-dimensional parameter vector $\boldsymbol{\xi} = (\xi_1, \dots, \xi_d)^\top \in \mathbb{R}^d$ to the continuous field. The detailed formulation of the covariance operator and the KL expansion is provided in Appendix~\ref{appendix:kl}.

\section{Variational Flow}
\label{variational_flow}

\subsection{Background: VAEs and Normalizing Flows}

\subsubsection{Variational Autoencoder (VAE)}
To approximate the true distribution $p_{\mb{x}}$ of data $\mb{x} \in \mathbb{R}^d$, VAE~\cite{kingma2013auto} introduces a latent variable $\mb{z} \in \mathbb{R}^{k}$ ($k<d$ for dimension reduction) and optimizes the Evidence Lower Bound (ELBO):
\begin{equation}
L_{\mb{\theta},\mb{\phi}}(\mb{x})
= -D_{\mathrm{KL}}\!\left(q_{\mb{z}|\mb{x},\mb{\phi}} \,\big\|\, p_{\mb{z}}\right)
+ \int q_{\mb{z}|\mb{x},\mb{\phi}} \log p_{\mb{x}|\mb{z},\mb{\theta}} \, d\mb{z},
\label{eq:elbo}
\end{equation}
where the decoder, encoder, and prior of VAE are respectively defined as:
\begin{align}
p_{\mb{x}|\mb{z},\mb{\theta}} &= \mathcal{N}\!\left(\mb{\mu}_{\mathrm{de},\mb{\theta}}(\mb{z}), \operatorname{diag}(\mb{\sigma}^2_{\mathrm{de},\mb{\theta}}(\mb{z}))\right), \label{eq:vae_decoder}\\
q_{\mb{z}|\mb{x},\mb{\phi}} &= \mathcal{N}\!\left(\mb{\mu}_{\mathrm{en},\mb{\phi}}(\mb{x}), \operatorname{diag}(\mb{\sigma}^2_{\mathrm{en},\mb{\phi}}(\mb{x}))\right), \label{eq:vae_encoder}\\
p_{\mb{z}} &= \mathcal{N}(\mb{0},\mb{I}), \label{eq:vae_prior}
\end{align}
with mean and variance $\mb{\mu}_{\mathrm{de},\mb{\theta}}$, $\mb{\sigma}^2_{\mathrm{de},\mb{\theta}}$, $\mb{\mu}_{\mathrm{en},\mb{\phi}}$, $\mb{\sigma}^2_{\mathrm{en},\mb{\phi}}$ parameterized by neural networks.

While these structural choices enable tractable optimization, they fundamentally limit the model's expressiveness: the standard Gaussian prior distribution assigned to the latent variable $\mb{z}$ cannot capture complex structures in the latent space, and the diagonal Gaussian encoder cannot represent multimodal or highly correlated distributions, both common in high-dimensional Bayesian inference.

\subsubsection{Normalizing Flow}
A normalizing flow~\cite{rezende2015variational, papamakarios2021normalizing} constructs a complex probability distribution through an invertible, differentiable transformation $f: \mathbb{R}^k \to \mathbb{R}^k$ that maps variable $\mb{z} \in \mathbb{R}^k$ to a simple base variable $\mb{v} = f(\mb{z})$ with known density $p_{\mb{v}}(\mb{v})$ (typically $\mathcal{N}(\mb{0}, \mb{I})$). By the change-of-variables formula, the density of $\mb{z}$ is:
\begin{equation}
    p_{\mb{z}}(\mb{z}) = p_{\mb{v}}\!\left(f(\mb{z})\right) \left|\det \nabla_{\mb{z}} f(\mb{z})\right|.
\end{equation}
In practice, $f$ is composed of $K$ invertible layers $f = f_K \circ \cdots \circ f_1$, each designed so that the Jacobian determinant is efficiently computable. This yields:
\begin{equation}
    p_{\mb{z}}(\mb{z}) = p_{\mb{v}}(\mb{v}_K) \prod_{k=1}^{K} \left|\det \nabla_{\mb{v}_{k-1}} f_k(\mb{v}_{k-1})\right|,
\end{equation}
where $\mb{v}_0 = \mb{z}$ and $\mb{v}_k = f_k(\mb{v}_{k-1})$.

\paragraph{Affine coupling layers}
A key building block of normalizing flows is the affine coupling layer~\cite{dinh2014nice, dinh2016density}. Given an input $\mb{v} \in \mathbb{R}^k$, it is split into two disjoint partitions $\mb{v} = [\mb{v}^A, \mb{v}^B]$, and the transformation is defined as:
\begin{equation}
    \mb{v}^A_{\mathrm{out}} = \mb{v}^A, \quad \mb{v}^B_{\mathrm{out}} = \mb{v}^B \odot \exp\!\left(\mb{s}(\mb{v}^A)\right) + \mb{t}(\mb{v}^A),
\end{equation}
where $\mb{s}(\cdot)$ and $\mb{t}(\cdot)$ are scale and translation functions parameterized by neural networks, and $\odot$ denotes element-wise multiplication. Since $\mb{v}^A$ is unchanged, the Jacobian of this transformation is lower-triangular, and its log-determinant reduces to $\sum_j s_j(\mb{v}^A)$. More expressive variants, such as Glow~\cite{kingma2018glow}, further incorporate invertible $1 \times 1$ convolutions and actnorm layers to improve flexibility.

\paragraph{Universal approximation}
Crucially, normalizing flows have been shown to be \emph{universal probability approximators}: given sufficient depth and width, coupling-based invertible networks can approximate any diffeomorphism and thus any continuous target distribution to arbitrary precision~\cite{teshima2020coupling, kong2020expressive}. This universal approximation property is the theoretical foundation underlying the expressiveness guarantees of the Variational Flow model.

\subsubsection{Conditional Normalizing Flow}
A Conditional Normalizing Flow \allowbreak(CNF)~\cite{zhu2019physics} extends the standard normalizing flow by conditioning each invertible transformation on an external variable, here the data $\mb{x}$.

\paragraph{Architecture}
The CNF is constructed by composing $K$ invertible conditional layers. Starting from $\mb{z}_0 = \mb{z}$, each layer applies
\begin{equation}
    \mb{z}_k = f^{(k)}(\mb{z}_{k-1};\mb{x}), \quad k = 1, \dots, K.
\end{equation}
Each $f^{(k)}(\cdot\,;\mb{x})$ is invertible with respect to $\mb{z}_{k-1}$ for any fixed $\mb{x}$. Typical implementations use conditional affine coupling layers, in which the scale and translation functions $\mb{s}(\cdot\,;\mb{x})$ and $\mb{t}(\cdot\,;\mb{x})$ take both the passthrough partition and $\mb{x}$ as inputs.

\paragraph{Base distribution}
The final output $\mb{z}_K = \mb{u}$ follows a conditional Gaussian base distribution:
\begin{equation}
p_{\mb{u}}(\mb{u};\mb{x}) = \mathcal{N}\!\left(\mb{\mu}(\mb{x}),\operatorname{diag}(\mb{\sigma}^2(\mb{x}))\right),
\end{equation}
where the mean $\mb{\mu}(\mb{x})$ and variance $\mb{\sigma}^2(\mb{x})$ are parameterized by neural networks. This data-dependent base distribution provides additional flexibility beyond a fixed standard Gaussian.

\paragraph{Density computation}
By iteratively applying the change-of-variables formula, the conditional density is:
\begin{equation}
q_{\mb{z}|\mb{x},\mb{\alpha}}(\mb{z}|\mb{x})
= p_{\mb{u}}(\mb{u};\mb{x})
\prod_{k=1}^{K}
\left|\det \frac{\partial f^{(k)}(\mb{z}_{k-1};\mb{x})}{\partial \mb{z}_{k-1}}\right|.
\end{equation} Since each conditional layer preserves the triangular Jacobian structure of affine coupling, the log-determinant remains computable.

\subsection{Variational Flow Model}
\label{sec:vf_model}

Normalizing flows provide expressive density models with exact likelihood evaluation, but their bijective structure requires equal input and output dimensions and thus does not support intrinsic dimensionality reduction. VAEs naturally address this issue through latent-variable compression, but standard Gaussian priors and encoders can be restrictive. We therefore propose Variational Flow (VF), which enriches both the latent prior and the encoder using normalizing flows.

\paragraph{Generalizing the prior}
We replace the fixed Gaussian prior $p_{\mb{z}} = \mathcal{N}(\mb{0}, \mb{I})$ with a normalizing flow prior. Specifically, we define a normalizing flow $f_{\mathrm{pr},\mb{\beta}}: \mathbb{R}^k \to \mathbb{R}^k$ with trainable parameters $\mb{\beta}$, which maps the latent variable $\mb{z}$ to a base variable $\mb{v} = f_{\mathrm{pr},\mb{\beta}}(\mb{z})$ with $\mb{v} \sim p_{\mb{v}} = \mathcal{N}(\mb{0}, \mb{I})$. By the change-of-variables formula, the induced prior density on $\mb{z}$ is:
\begin{equation}
p_{\mb{z},\mb{\beta}}(\mb{z}) = p_{\mb{v}}\!\big(f_{\mathrm{pr},\mb{\beta}}(\mb{z})\big)\left|\det \nabla_{\mb{z}} f_{\mathrm{pr},\mb{\beta}}(\mb{z})\right|.
\end{equation}



\paragraph{Generalizing the encoder}
We replace the diagonal Gaussian encoder $q_{\mb{z}|\mb{x},\mb{\phi}}$ with a CNF. Specifically, we define a CNF $f_{\mathrm{en},\mb{\alpha}}: \mathbb{R}^k \to \mathbb{R}^k$ with trainable parameters $\mb{\alpha}$, which maps the latent variable $\mb{z}$ to a base variable $\mb{u} = f_{\mathrm{en},\mb{\alpha}}(\mb{z};\mb{x})$ with $\mb{u} \sim p_{\mb{u}}(\mb{u};\mb{x})$. By the change-of-variables formula, the encoder density becomes:
\begin{equation}
q_{\mb{z}|\mb{x},\mb{\alpha}}(\mb{z}|\mb{x})
= p_{\mb{u}}\!\left(f_{\mathrm{en},\mb{\alpha}}(\mb{z};\mb{x});\mb{x}\right)
\left|\det \nabla_{\mb{z}} f_{\mathrm{en},\mb{\alpha}}(\mb{z};\mb{x})\right|.
\label{eq:cnf_encoder}
\end{equation}

\paragraph{VF definition}
Combining the above generalizations, the Variational Flow is defined by:
\begin{align}
p_{\mb{z},\mb{\beta}}(\mb{z})&=p_{\mb{v}}\!\big(f_{\mathrm{pr},\mb{\beta}}(\mb{z})\big)\left|\det \nabla_{\mb{z}} f_{\mathrm{pr},\mb{\beta}}(\mb{z})\right|, \label{vf_prior}\\
q_{\mb{z}|\mb{x},\mb{\alpha}}(\mb{z}|\mb{x})&=p_{\mb{u}}\!\left(f_{\mathrm{en},\mb{\alpha}}(\mb{z};\mb{x});\,\mb{x}\right) \left|\det \nabla_{\mb{z}} f_{\mathrm{en},\mb{\alpha}}(\mb{z};\mb{x})\right|,\label{vf_encoder}\\
p_{\mb{x}|\mb{z},\mb{\theta}}(\mb{x}|\mb{z})&=\mathcal{N}\!\left(\mb{\mu}_{\mathrm{de},\mb{\theta}}(\mb{z}), \operatorname{diag}(\mb{\sigma}^2_{\mathrm{de},\mb{\theta}}(\mb{z}))\right).\label{vf_decoder}
\end{align}
This ``dual-flow" architecture addresses the main limitation of standard normalizing flows -- the bijectivity constraint that prevents dimensionality reduction. It achieves this by using the encoder-decoder structure of VAE for latent compression and applying flows to model the density in both the prior and encoder.
\begin{thm}
\label{thm:vf_elbo_extended}
Suppose the decoder $p_{\mb{x}|\mb{z},\mb{\theta}}$ in VAE is parameterized by nonlinear neural networks. If
\begin{enumerate}
\item[(a)] the family of flow priors $\{p_{\mb{z},\mb{\beta}}\}_{\mb{\beta}}$ induced by the bijection $f_{\mathrm{pr},\boldsymbol{\beta}}$ can approximate any $k$-dimensional probability density function to arbitrary precision;
\item[(b)] the family of conditional normalizing flow encoders $\{q_{\mb{z}|\mb{x},\boldsymbol{\alpha}}\}_{\boldsymbol{\alpha}}$ can approximate any conditional distribution to arbitrary precision;
\end{enumerate}
suppose $(\mb{\phi}^*,\mb{\theta}^*)$ are the optimal parameters for standard VAE:
\begin{equation}
(\mb{\phi}^*,\mb{\theta}^*) = \mathop{\arg\max}_{\mb{\phi},\, \mb{\theta}} \; \mathbb{E}_{p_{\mb{x}}}[L_{\mb{\phi},\mb{\theta}}(\mb{x})],
\end{equation}
then the Variational Flow (VF) model achieves a strictly higher Evidence Lower Bound (ELBO) than standard VAE if either of the following conditions holds:
\begin{enumerate}
\item[(i)] the aggregated posterior of the optimal encoder in the standard VAE does not equal the prior, i.e., $q_{\mb{z},\boldsymbol{\phi}^*}= \int q_{\mb{z}|\mb{x},\mb{\phi}^*} p_{\mb{x}}d\mb{x} \neq p_{\mb{z}}$;
\item[(ii)] there exists a set of positive measure with respect to $p_{\mb{x}}$ such that the model posterior $p_{\mb{z}|\mb{x},\boldsymbol{\theta}^*}=p_{\mb{x}|\mb{z},\mb{\theta}^*}p_{\mb{z}}/ \int p_{\mb{x}|\mb{z},\mb{\theta}^*}p_{\mb{z}}d\mb{z}$ of VAE is not a diagonal Gaussian distribution.
\end{enumerate}
Furthermore, when conditions (i) and (ii) hold simultaneously, the flow prior and the conditional normalizing flow encoder each provide a strictly positive improvement to the ELBO.
\end{thm}
\begin{proof}

Let 
\begin{equation}
E_0 \coloneqq \mathbb{E}_{p_{\mb{x}}}[L_{\mb{\phi}^*,\mb{\theta}^*}(\mb{x})]
= -\mathbb{E}_{p_{\mb{x}}}[D_{\mathrm{KL}}(q_{\mb{z}|\mb{x},\mb{\phi}^*} \| p_{\mb{z}})]
+ \mathbb{E}_{p_{\mb{x}}\, q_{\mb{z}|\mb{x},\mb{\phi}^*}}[\log p_{\mb{x}|\mb{z},\mb{\theta}^*}].
\end{equation}

\paragraph{Generalize the prior}
Fixing the encoder $q_{\mb{z}|\mb{x},\mb{\phi}^*}$ and decoder $p_{\mb{x}|\mb{z},\mb{\theta}^*}$ and replacing the prior $p_{\mb{z}}$ with the normalizing flow prior $p_{\mb{z},\mb{\beta}}$, the ELBO becomes
\begin{equation}
\mathbb{E}_{p_{\mb{x}}}[L_{\mb{\phi}^*,\mb{\theta}^*,\mb{\beta}}(\mb{x})]
= -\mathbb{E}_{p_{\mb{x}}}[D_{\mathrm{KL}}(q_{\mb{z}|\mb{x},\mb{\phi}^*} \| p_{\mb{z},\mb{\beta}})]
+ \mathbb{E}_{p_{\mb{x}}\, q_{\mb{z}|\mb{x},\mb{\phi}^*}}[\log p_{\mb{x}|\mb{z},\mb{\theta}^*}].
\end{equation}
Subtracting $E_0$ yields
\begin{equation}
\mathbb{E}_{p_{\mb{x}}}[L_{\mb{\phi}^*,\mb{\theta}^*,\mb{\beta}}(\mb{x})] - E_0
= \mathbb{E}_{p_{\mb{x}}}[D_{\mathrm{KL}}(q_{\mb{z}|\mb{x},\mb{\phi}^*} \| p_{\mb{z}})]
- \mathbb{E}_{p_{\mb{x}}}[D_{\mathrm{KL}}(q_{\mb{z}|\mb{x},\mb{\phi}^*} \| p_{\mb{z},\mb{\beta}})].
\label{eq:diff_prior}
\end{equation}
Simplifying the right-hand side of the above equation, we note that
\begin{align}
\mathbb{E}_{p_{\mb{x}}}[D_{\mathrm{KL}}(q_{\mb{z}|\mb{x},\mb{\phi}^*} \| p)]
&= \int q_{\mb{z}|\mb{x},\mb{\phi}^*}(\mb{z}|\mb{x})\, p_{\mb{x}}(\mb{x}) \log \frac{q_{\mb{z}|\mb{x},\mb{\phi}^*}(\mb{z}|\mb{x})}{p(\mb{z})}\, d\mb{z}\, d\mb{x} \nonumber\\
&= \int q_{\mb{z}|\mb{x},\mb{\phi}^*}\, p_{\mb{x}} \log q_{\mb{z}|\mb{x},\mb{\phi}^*}\, d\mb{z}\, d\mb{x}
- \int q_{\mb{z},\mb{\phi}^*}(\mb{z}) \log p(\mb{z})\, d\mb{z},
\end{align}
where $q_{\mb{z},\mb{\phi}^*}(\mb{z}) = \int q_{\mb{z}|\mb{x},\mb{\phi}^*}(\mb{z}|\mb{x})\, p_{\mb{x}}(\mb{x})\, d\mb{x}$ is the aggregated posterior. Thus, \eqref{eq:diff_prior} simplifies to
\begin{equation}
\mathbb{E}_{p_{\mb{x}}}[L_{\mb{\phi}^*,\mb{\theta}^*,\mb{\beta}}(\mb{x})] - E_0
= D_{\mathrm{KL}}(q_{\mb{z},\mb{\phi}^*} \| p_{\mb{z}}) - D_{\mathrm{KL}}(q_{\mb{z},\mb{\phi}^*} \| p_{\mb{z},\mb{\beta}}).
\label{eq:prior_improvement}
\end{equation}
By assumption (a), the flow prior $p_{\mb{z},\mb{\beta}}$ can approximate $q_{\mb{z},\mb{\phi}^*}$ to arbitrary precision. If condition (i) holds, i.e., $q_{\mb{z},\mb{\phi}^*} \neq p_{\mb{z}}$, then $D_{\mathrm{KL}}(q_{\mb{z},\mb{\phi}^*} \| p_{\mb{z}}) > 0$. By choosing $\tilde{\mb{\beta}}$ such that $D_{\mathrm{KL}}(q_{\mb{z},\mb{\phi}^*} \| p_{\mb{z},\tilde{\mb{\beta}}}) < D_{\mathrm{KL}}(q_{\mb{z},\mb{\phi}^*} \| p_{\mb{z}})$, we have
\begin{equation}
E_1 \coloneqq \mathbb{E}_{p_{\mb{x}}}[L_{\mb{\phi}^*,\mb{\theta}^*,\tilde{\mb{\beta}}}(\mb{x})] > E_0.
\label{eq:E1_gt_E0}
\end{equation}

\paragraph{Generalize the encoder} 
Fixing the decoder $p_{\mb{x}|\mb{z},\mb{\theta}^*}$ and the prior $p_{\mb{z},\tilde{\mb{\beta}}}$, we extend the encoder from a diagonal Gaussian to a conditional normalizing flow $q_{\mb{z}|\mb{x},\boldsymbol{\alpha}}$. Utilizing the following decomposition of the ELBO: 
\begin{align}
\mathbb{E}_{p_{\mb{x}}}[L_{\mb{\alpha},\mb{\theta}^*,\tilde{\mb{\beta}}}(\mb{x})]
&= \mathbb{E}_{p_{\mb{x}}}\left[\int q_{\mb{z}|\mb{x},\mb{\alpha}}\left[\log\frac{p_{\mb{x}|\mb{z},\mb{\theta}^*}p_{\mb{z},\tilde{\mb{\beta}}}}{q_{\mb{z}|\mb{x},\mb{\alpha}}}\right]d\mb{z}\right] \nonumber\\
&= \mathbb{E}_{p_{\mb{x}}}\left[\int q_{\mb{z}|\mb{x},\mb{\alpha}}\left[\log\frac{p_{\mb{x},\mb{\theta}^*,\tilde{\mb{\beta}}}p_{\mb{z}|\mb{x},\mb{\theta}^*,\tilde{\mb{\beta}}}}{q_{\mb{z}|\mb{x},\mb{\alpha}}}\right]d\mb{z}\right] \nonumber\\
&=\mathbb{E}_{p_{\mb{x}}}[\log p_{\mb{x},\mb{\theta}^*,\tilde{\mb{\beta}}}(\mb{x})]
- \mathbb{E}_{p_{\mb{x}}}[D_{\mathrm{KL}}(q_{\mb{z}|\mb{x},\mb{\alpha}} \| p_{\mb{z}|\mb{x},\mb{\theta}^*,\tilde{\mb{\beta}}})],\label{eq:elbo_decomp}
\end{align}
where
\begin{equation}
p_{\mb{x},\mb{\theta}^*,\tilde{\mb{\beta}}}(\mb{x}) = \int p_{\mb{x}|\mb{z},\mb{\theta}^*}(\mb{x}|\mb{z})\, p_{\mb{z},\tilde{\mb{\beta}}}(\mb{z})\, d\mb{z}
\end{equation}
is the model evidence (independent of the encoder parameters), and
\begin{equation}
p_{\mb{z}|\mb{x},\mb{\theta}^*,\tilde{\mb{\beta}}}(\mb{z}|\mb{x})
= \frac{p_{\mb{x}|\mb{z},\mb{\theta}^*}(\mb{x}|\mb{z})\, p_{\mb{z},\tilde{\mb{\beta}}}(\mb{z})}{p_{\mb{x},\mb{\theta}^*,\tilde{\mb{\beta}}}(\mb{x})}
\end{equation}
is the true model posterior under the decoder $p_{\mb{x}|\mb{z},\mb{\theta}^*}$ and the prior $p_{\mb{z},\tilde{\mb{\beta}}}$.
From the decomposition~\eqref{eq:elbo_decomp}, maximizing the ELBO with respect to the encoder is equivalent to minimizing $\mathbb{E}_{p_{\mb{x}}}[D_{\mathrm{KL}}(q_{\mb{z}|\mb{x},\mb{\alpha}} \| p_{\mb{z}|\mb{x},\mb{\theta}^*,\tilde{\mb{\beta}}})]$. In prior generalization, the encoder remains a diagonal Gaussian $q_{\mb{z}|\mb{x},\mb{\phi}^*}$, and the corresponding ELBO is
\begin{equation}
E_1 = \mathbb{E}_{p_{\mb{x}}}[\log p_{\mb{x},\mb{\theta}^*,\tilde{\mb{\beta}}}(\mb{x})]
- \mathbb{E}_{p_{\mb{x}}}[D_{\mathrm{KL}}(q_{\mb{z}|\mb{x},\mb{\phi}^*} \| p_{\mb{z}|\mb{x},\mb{\theta}^*,\tilde{\mb{\beta}}})].
\end{equation}
Since the decoder mean $\mb{\mu}_{\mathrm{de},\mb{\theta}^*}(\mb{z})$ is a nonlinear function of $\mb{z}$, the log model posterior
$$
\log p_{\mb{z}|\mb{x},\mb{\theta}^*,\tilde{\mb{\beta}}}(\mb{z}|\mb{x})
= \log p_{\mb{x}|\mb{z},\mb{\theta}^*}(\mb{x}|\mb{z}) + \log p_{\mb{z},\tilde{\mb{\beta}}}(\mb{z}) - \log p_{\mb{x},\mb{\theta}^*,\tilde{\mb{\beta}}}(\mb{x})
$$
is generally not a quadratic function of $\mb{z}$; thus, $p_{\mb{z}|\mb{x},\mb{\theta}^*,\tilde{\mb{\beta}}}$ is generally not a diagonal Gaussian distribution. If condition (ii) holds, the diagonal Gaussian encoder cannot accurately match the model posterior, i.e.,
\begin{equation}
\mathbb{E}_{p_{\mb{x}}}[D_{\mathrm{KL}}(q_{\mb{z}|\mb{x},\mb{\phi}^*} \| p_{\mb{z}|\mb{x},\mb{\theta}^*,\tilde{\mb{\beta}}})] > 0.
\label{eq:gauss_gap}
\end{equation}
By assumption (b), the conditional normalizing flow encoder $q_{\mb{z}|\mb{x},\mb{\alpha}}$ can approximate any conditional distribution to arbitrary precision. Therefore, there exists $\tilde{\mb{\alpha}}$ such that
\begin{equation}
\mathbb{E}_{p_{\mb{x}}}[D_{\mathrm{KL}}(q_{\mb{z}|\mb{x},\tilde{\mb{\alpha}}} \| p_{\mb{z}|\mb{x},\mb{\theta}^*,\tilde{\mb{\beta}}})]
< \mathbb{E}_{p_{\mb{x}}}[D_{\mathrm{KL}}(q_{\mb{z}|\mb{x},\mb{\phi}^*} \| p_{\mb{z}|\mb{x},\mb{\theta}^*,\tilde{\mb{\beta}}})].
\label{eq:encoder_improvement}
\end{equation}
From \eqref{eq:elbo_decomp} and \eqref{eq:encoder_improvement}, we have
\begin{equation}
E_2 \coloneqq \mathbb{E}_{p_{\mb{x}}}[L_{\tilde{\mb{\alpha}},\mb{\theta}^*,\tilde{\mb{\beta}}}(\mb{x})] > E_1.
\label{eq:E2_gt_E1}
\end{equation}

\paragraph{Combining the two steps}
From \eqref{eq:E1_gt_E0} and \eqref{eq:E2_gt_E1}, when both conditions (i) and (ii) hold,
\begin{equation}
E_2 > E_1 > E_0.
\end{equation}
Since the optimal ELBO of the VF model is no lower than the ELBO value at any specific parameters, we have
\begin{equation}
\sup_{\mb{\alpha},\, \mb{\theta},\, \mb{\beta}} \mathbb{E}_{p_{\mb{x}}}[L_{\mb{\alpha},\mb{\theta},\mb{\beta}}(\mb{x})]
\geq E_2 > E_0
= \sup_{\mb{\phi},\, \mb{\theta}} \mathbb{E}_{p_{\mb{x}}}[L_{\mb{\phi},\mb{\theta}}(\mb{x})].
\end{equation}
When only condition (ii) holds, one can directly choose $\tilde{\mb{\beta}}$ in encoder generalization step to correspond to the identity mapping (i.e., the prior remains $p_{\mb{z}}$), in which case the encoder improvement still provides a strict improvement. In summary, if either condition (i) or (ii) holds, a strict improvement in the ELBO is guaranteed.
\end{proof}

\begin{remark}
Condition (i) is expected to hold for most non-degenerate VAEs in practice: if $q_{\mb{z},\boldsymbol{\phi}^*} = p_{\mb{z}} = \mathcal{N}(\mb{0},\mb{I})$, it implies that the marginal distribution of the latent variables across all training samples is exactly the standard Gaussian, which is almost impossible to achieve precisely in practice. 
Condition (ii) typically holds when the nonlinear decoder induces a non-quadratic log-likelihood in the latent variable: since the model posterior satisfies $p_{\mb{z}|\mb{x},\boldsymbol{\theta}^*}(\mb{z}|\mb{x}) \propto p_{\mb{x}|\mb{z},\boldsymbol{\theta}^*}(\mb{x}|\mb{z})\, p_{\mb{z}}(\mb{z})$, and $\log p_{\mb{x}|\mb{z},\boldsymbol{\theta}^*}$ is generally non-quadratic with respect to $\mb{z}$ (as the decoder mean $\boldsymbol{\mu}_{\mathrm{de},\boldsymbol{\theta}^*}(\mb{z})$ is a nonlinear function), the model posterior is generally non-Gaussian.
\end{remark}

\begin{remark}
The proof reveals that ELBO improvement stems from two complementary mechanisms:
\begin{itemize}

\item \textbf{Prior matching}: The flow prior eliminates the mismatch between the aggregated posterior $q_{\mb{z},\mb{\phi}^*}$ and the prior $p_{\mb{z}}$, with a contribution of $D_{\mathrm{KL}}(q_{\mb{z},\mb{\phi}^*} \| p_{\mb{z}}) - D_{\mathrm{KL}}(q_{\mb{z},\mb{\phi}^*} \| p_{\mb{z},\tilde{\mb{\beta}}}) > 0$;

\item \textbf{Posterior approximation}: The CNF encoder narrows the variational gap, i.e., the difference in KL divergence between the encoder and the true model posterior: $\mathbb{E}_{p_{\mb{x}}}[D_{\mathrm{KL}}(q_{\mb{z}|\mb{x},\mb{\phi}^*} \| p_{\mb{z}|\mb{x},\cdot})] - \mathbb{E}_{p_{\mb{x}}}[D_{\mathrm{KL}}(q_{\mb{z}|\mb{x},\tilde{\mb{\alpha}}} \| p_{\mb{z}|\mb{x},\cdot})] > 0$.
\end{itemize}
The two work synergistically: the flow prior absorbs the burden of matching the aggregated posterior, freeing the CNF encoder to capture complex model posteriors (multimodality, asymmetry, etc.).
\end{remark}
\subsection{Approximating Unnormalized Densities via Variational Flow}
\label{sec:vf_posterior}

Beyond standard generative modeling where VF is trained on observed data by maximizing the ELBO (see Appendix~\ref{appendix:vf_case1}), VF can also approximate a target density $\hat{p}(\mb{x})$ that is known only up to a normalization constant $C = \int \hat{p}(\mb{x})\, d\mb{x}$. We minimize the KL divergence between the VF joint and a target joint constructed with the encoder:
\begin{equation}
\label{vf_object}
\begin{aligned}
D_{\mathrm{KL}}\!\left(p_{\mb{x}|\mb{z},\mb{\theta}}\,p_{\mb{z},\mb{\beta}}\,\Big\Vert\, q_{\mb{z}|\mb{x},\mb{\alpha}}\,\frac{\hat{p}(\mb{x})}{C}\right)
=& \int p_{\mb{x}|\mb{z},\mb{\theta}}\,p_{\mb{z},\mb{\beta}} \log\!\left(\frac{p_{\mb{x}|\mb{z},\mb{\theta}}\,p_{\mb{z},\mb{\beta}}}{q_{\mb{z}|\mb{x},\mb{\alpha}}\,\hat{p}(\mb{x})}\right) d\mb{z}\,d\mb{x} \\
&+ \log C.
\end{aligned}
\end{equation}
Since $\log C$ is independent of the model parameters $(\mb{\beta},\mb{\alpha},\mb{\theta})$, we minimize only the first term via Monte Carlo:
\begin{equation}
\mathcal{L}_{\mathrm{VF}}(\mb{\beta},\mb{\alpha},\mb{\theta}) \approx \frac{1}{M}\sum_{i=1}^{M}\log\!\left(\frac{p_{\mb{x}|\mb{z},\mb{\theta}}(\mb{x}^{(i)}|\mb{z}^{(i)})\,p_{\mb{z},\mb{\beta}}(\mb{z}^{(i)})}{q_{\mb{z}|\mb{x},\mb{\alpha}}(\mb{z}^{(i)}|\mb{x}^{(i)})\,\hat{p}(\mb{x}^{(i)})}\right),
\label{eq:vf_loss}
\end{equation}
where $\mb{z}^{(i)} \sim p_{\mb{z},\mb{\beta}}$ and $\mb{x}^{(i)} \sim p_{\mb{x}|\mb{z}^{(i)},\mb{\theta}}$. In this work, we set $\hat{p}(\mb{x}) = \exp\!\left(-\Phi(\mb{\xi},\mb{y})\right)\pi_0(\mb{\xi})$ as the unnormalized posterior from Eq.~\eqref{eq:posterior} by identifying the VF data variable $\mb{x}$ with the unknown parameter vector $\mb{\xi}$. Once trained, VF efficiently generates posterior samples $\mb{\xi} \sim p_{\mathrm{VF}}(\mb{\xi})$ via 
\begin{equation}
\mb{z}=f^{-1}_{\mathrm{pr},\mb{\beta}}(\mb{v}),\; \mb{v}\sim \mathcal{N}(\mb{0},\mb{I})\quad \Rightarrow \quad \mb{\xi}= \mb{\mu}_{\mathrm{de},\mb{\theta}}(\mb{z}) + \mb{\sigma}_{\mathrm{de},\mb{\theta}}(\mb{z}) \odot \mb{\epsilon}, \; \mb{\epsilon}\sim \mathcal{N}(\mb{0},\mb{I}),
\end{equation}
which are used to adaptively update the surrogate model as described in the next section.

\section{Deep Adaptive Bayesian Inference}
\label{sec:method}

\subsection{Overall Framework}
\label{sec:framework}

The deep adaptive framework addresses two coupled challenges: the out-of-distribution (OOD) inaccuracy of the surrogate and the difficulty of specifying an appropriate prior mean. We tackle both through an alternating optimization loop between the VF posterior approximation and the surrogate model fine-tuning, alongside a dynamic update of the prior mean.

At each stage $k$, the framework executes three sequential steps:
\begin{itemize}

    \item \textbf{VF training}: approximate the current unnormalized posterior $$\hat{\pi}(\mb{\xi}) \propto \exp(-\Phi(\mb{\xi},\mb{y}))\,\pi_0^{(k)}(\mb{\xi})$$ via~\eqref{eq:vf_loss} and update the prior mean iteratively (see Section~\ref{sec:iter_prior}).
    
    \item \textbf{Data replacement}: sample $\{\mb{\xi}_{\mathrm{post}}^{(j)}\}_{j=1}^M \sim p_{\mathrm{VF}}(\mb{\xi})$, query the exact PDE solver for these samples, and construct a newly generated local dataset $\mathcal{D}_k$ as demonstrated in Subsection~\ref{sec:adaptive_surrogate}, discarding the historical data.
    
    \item \textbf{Surrogate fine-tuning}: fine-tune the surrogate model $\mathcal{F}_{\mb{\vartheta}}$ exclusively on the new dataset $\mathcal{D}_k$ to align its accuracy with the current posterior region.
\end{itemize}

The loop terminates when the relative change in data misfit evaluated at the prior mean,
$\bigl| \Phi(\mb{\mu}_{\mathrm{prior}}^{(k-1)},\mb{y}) - \Phi(\mb{\mu}_{\mathrm{prior}}^{(k)},\mb{y}) \bigr|
/ \Phi(\mb{\mu}_{\mathrm{prior}}^{(k-1)},\mb{y})$ falls below a threshold $\epsilon$. Algorithm~\ref{alg:dabai} summarizes the procedure. In our implementation, we employ the Fourier Neural Operator (FNO) as the surrogate model. Details regarding its pre-training are provided in Appendix~\ref{appendix:fno}.

\begin{algorithm}[ht]
\caption{Deep Adaptive Bayesian Inference}
\label{alg:dabai}
\begin{algorithmic}[1]
\REQUIRE Pre-trained FNO $\mathcal{F}_{\mb{\vartheta}^{(0)}}$, initial dataset $\mathcal{D}_0$ sampled from $\pi_0(\mb{\xi}) = \mathcal{N}(\mb{\mu}_0, \mb{\Sigma}_0)$, initial prior mean $\mb{\mu}_{\mathrm{prior}}^{(0)} = \mb{\mu}_{0}$, VF parameters $\mb{\psi}^{(0)}$, max stages $K$, inner epochs $N_e$, samples per stage $M$, coefficient $\alpha$, perturbation strength $\gamma$, tolerance $\epsilon$.
\FOR{$k = 1, 2, \ldots, K$}
    \FOR{$i = 1, 2, \ldots, N_e$}
        \STATE Compute current VF mean estimate $\mb{\mu}_{\mathrm{post}}^{(k,i)}$; set $\mb{\mu}_{\mathrm{prior}}^{(k,i)} = \alpha\,\mb{\mu}_{\mathrm{post}}^{(k,i)} + (1-\alpha)\,\mb{\mu}_{\mathrm{prior}}^{(k-1)}$.
        \STATE Set prior $\pi_0^{(k,i)} = \mathcal{N}(\mb{\mu}_{\mathrm{prior}}^{(k,i)}, \mb{\Sigma}_0)$.
        \STATE Update VF to approximate $\hat{\pi}(\mb{\xi}) \propto \exp(-\Phi(\mb{\xi},\mb{y}))\,\pi_0^{(k,i)}(\mb{\xi})$ by minimizing $\mathcal{L}_{\mathrm{VF}}$~\eqref{eq:vf_loss}.
    \ENDFOR
    \STATE Update prior mean: $\mb{\mu}_{\mathrm{prior}}^{(k)} \leftarrow \mb{\mu}_{\mathrm{prior}}^{(k, N_e)}$.
    \IF{$\bigl| \Phi(\mb{\mu}_{\mathrm{prior}}^{(k-1)},\mb{y}) - \Phi(\mb{\mu}_{\mathrm{prior}}^{(k)},\mb{y}) \bigr| \,/\, \Phi(\mb{\mu}_{\mathrm{prior}}^{(k-1)},\mb{y}) < \epsilon$}
        \STATE \textbf{break}
    \ENDIF
    \STATE Sample $\{\mb{\xi}_{\mathrm{post}}^{(j)}\}_{j=1}^M \sim p_{\mathrm{VF}}(\mb{\xi})$; add perturbations $\hat{\mb{\xi}}^{(j)} = \mb{\xi}_{\mathrm{post}}^{(j)} + \gamma\,\mb{\nu}^{(j)}$, $\mb{\nu}^{(j)} \sim \mathcal{N}(\mb{0},\mb{I})$.
    \STATE Solve PDE for each $\hat{\mb{\xi}}^{(j)}$; update $\mathcal{D}_k \leftarrow \{(m_{\hat{\mb{\xi}}^{(j)}}, u(\cdot,\hat{\mb{\xi}}^{(j)}))\}_{j=1}^M$.
    \STATE Fine-tune FNO: $\mb{\vartheta}^{(k)} \leftarrow \mathop{\arg\min}_{\mb{\vartheta}}\, \mathcal{L}_{\mathrm{FNO}}(\mb{\vartheta};\mathcal{D}_k)$~\eqref{eq:fno_loss}.
\ENDFOR
\ENSURE Trained VF $\mb{\psi}^*$, FNO $\mb{\vartheta}^*$; posterior samples $\{\mb{\xi}\} \sim p_{\mathrm{VF}}(\mb{\xi})$, $\mb{\mu}_{\mathrm{post}}=\mathbb{E}_{p_{\mathrm{VF}}(\mb{\xi})}[\mb{\xi}]$.
\end{algorithmic}
\end{algorithm}
\subsection{Iterative Prior Updating}
\label{sec:iter_prior}

To overcome the limitation of a narrow initial guess and prevent the surrogate from querying unphysical regimes caused by overly broad priors, we seamlessly embed a dynamic prior updating strategy directly into the Variational Flow (VF) training loop. We dynamically recalibrate the mean of Gaussian prior $\pi_0(\mb{\xi})$ appearing in the posterior distribution~\eqref{eq:posterior} before the $i$-th epoch of stage $k$ via a momentum-based moving average:
\begin{equation}
    \mb{\mu}_{\mathrm{prior}}^{(k, i)} = \alpha\,\mb{\mu}_{\mathrm{post}}^{(k,i)} + (1-\alpha)\,\mb{\mu}_{\mathrm{prior}}^{(k-1)}.
    \label{eq:iter_prior}
\end{equation}
This scheme balances adaptability and stability across two time scales. \textbf{At the epoch level}, the posterior mean $\mb{\mu}_{\mathrm{post}}^{(k,i)}$ is estimated by averaging $M$ samples drawn from the current VF model $p_{\mathrm{VF}}^{(k,i-1)}$ (i.e., the VF obtained after epoch $i-1$ of stage $k$):
\begin{equation}
    \mb{\mu}_{\mathrm{post}}^{(k,i)} = \frac{1}{M} \sum_{j=1}^{M} \mb{\xi}^{(j)}, \quad \mb{\xi}^{(j)} \sim p_{\mathrm{VF}}^{(k,i-1)}(\mb{\xi}).
\end{equation}
\textbf{At the stage level}, the anchor $\mb{\mu}_{\mathrm{prior}}^{(k-1)}$ remains fixed throughout stage $k$ and is updated only at the end of the stage via $\mb{\mu}_{\mathrm{prior}}^{(k)} \leftarrow \mb{\mu}_{\mathrm{prior}}^{(k, N_{e})}$. Accordingly, the effective prior is $\pi_0^{(k,i)}(\mb{\xi}) = \mathcal{N}(\mb{\mu}_{\mathrm{prior}}^{(k,i)}, \mb{\Sigma}_0)$, where $\alpha \in (0,1]$ controls the update strength. As noted in \cite{huang2022iterated}, smaller $\alpha$ yields stronger regularization, whereas $\alpha = 1$ removes this effect; that study further found that $\alpha \approx 0.5$ is close to optimal in general. Keeping the covariance fixed at $\mb{\Sigma}_0$ further prevents mode collapse and promotes stable exploration. We initialize with a standard isotropic Gaussian, i.e., $\mb{\mu}_{0} = \mb{0}$ and $\mb{\Sigma}_0 = \mb{I}$, and then use Eq.~\eqref{eq:iter_prior} to gradually shift the prior toward posterior region.

\subsection{Adaptive Surrogate Fine-tuning}
\label{sec:adaptive_surrogate}
Following each VF training stage, the FNO surrogate must be rectified to match the newly located posterior region. We draw a set of high-fidelity parameter samples $\{\mb{\xi}_{\mathrm{post}}^{(j)}\}_{j=1}^M$ from the updated VF model $p_{\mathrm{VF}}(\mb{\xi})$. Because these samples tightly concentrate around the inferred posterior, we inject a Gaussian perturbation with controlled strength to maintain local support coverage and prevent the surrogate from overfitting to a degenerate region:
\begin{equation}
    \hat{\mb{\xi}}^{(j)} = \mb{\xi}_{\mathrm{post}}^{(j)} + \gamma\,\mb{\nu}^{(j)}, \quad \mb{\nu}^{(j)} \sim \mathcal{N}(\mb{0}, \mb{I}),
    \label{eq:perturbation}
\end{equation}
where $\gamma > 1$ controls the perturbation strength. This perturbation is applied directly in the Karhunen-Lo\`{e}ve coefficient space, and its scale is defined relative to the prior standard deviation, since the prior is standard normal in this space. 

Instead of accumulating a growing historical dataset or relying on expensive greedy filtering, we adopt an aggressive replacement strategy. We generate new PDE solutions strictly for these perturbed samples and fine-tune the surrogate model exclusively on this newly generated data. This dynamic replacement directly rectifies the OOD problem by continuously shifting the surrogate's effective training distribution from the initial broad prior to the highly localized posterior support, ensuring local surrogate accuracy with minimal computational overhead.

\section{Numerical Experiments}
\label{sec:experiments}

\subsection{Classes of Problem Studied}
The performance of the proposed framework is evaluated on two types of problems. First, to specifically assess the posterior approximation capability of the VF model, we consider a challenging high-dimensional synthetic distribution. Second, to validate the complete framework, which incorporates both the adaptive surrogate fine-tuning and iterative prior updating, we test it on three PDE-governed inverse problems of increasing complexity.

\begin{itemize}
    \item[\textbf{1.}] \textbf{A 100-dimensional Rosenbrock inverse problem:} In Subsection~\ref{subsec_rosen}, we utilize this problem to evaluate the capability of various methods in approximating complex, non-Gaussian posterior distributions. To isolate the generative performance of the models, the iterative prior updating module is not employed in this example. The numerical results demonstrate that our approach achieves superior posterior approximation accuracy compared to the baseline methods. Notably, it significantly outperforms the vanilla VAE, which empirically corroborates our theoretical result.

    \item[\textbf{2.}] \textbf{PDE-governed inverse problem:} In Subsections~\ref{subsec:1d_darcy}, \ref{subsec:2d_darcy}, and \ref{subsec:2d_ns}, we evaluate our framework on three PDE-governed inverse problems, where observations are corrupted by relative Gaussian noise:
    \begin{equation}
        \mb{y} = \mathcal{G}(\mb{\xi}_{\mathrm{ref}}) + \delta \max\bigl\{|\mathcal{G}(\mb{\xi}_{\mathrm{ref}})|\bigr\}\,\boldsymbol{\eta}, \quad \boldsymbol{\eta} \sim \mathcal{N}(\mb{0}, \mb{I}),
        \label{eq:noise_model}
    \end{equation}
    with noise amplitude $\delta \in \{1\%, 5\%, 10\%\}$. The unknown field $m_{\mb{\xi}}(\mb{x})$ is modeled as a Gaussian random field parameterized via a truncated Karhunen-Lo\`{e}ve (KL) expansion. To evaluate robustness, the reference field $m_{\mb{\xi}_\mathrm{ref}}$ is generated using 256 KL modes with coefficients i.i.d. drawn from $\mathcal{U}[-10,10]$ (rather than the Gaussian prior), while the inversion targets only the first $d \in \{32, 64\}$ modes. Since the FNO surrogate is pre-trained on standard Gaussian prior samples, adaptive fine-tuning is essential to capture the shifted posterior. We compare our approach against pCN with the exact Finite Difference Method (FDM), UKI-FDM, UKI-FNO, and SVGD-FNO. For both UKI and our iterative prior updating scheme, the scaling parameter is set to $\alpha = 0.5$. Each experiment is repeated 3 times and the reported results correspond to the averaged outcomes. Further experimental details are provided in Appendix~\ref{appendix:exp_details}.

    Finally, the performance is assessed using two quantitative metrics. The first is the relative inversion error, which measures the accuracy of the reconstructed parameter field:
\begin{equation}
    e_{\mathcal{I}} = \frac{\|m_{\mb{\mu}_{\mathrm{post}}} - m_{\mb{\xi}_\mathrm{ref}}\|_{L^2(\Omega_s)} }{\|m_{\mb{\xi}_\mathrm{ref}}\|_{L^2(\Omega_s)}},
    \label{eq:error_1}
\end{equation}
where $\mb{\mu}_\mathrm{post}$ denotes the posterior mean estimate obtained from Algorithm~\ref{alg:dabai}. The second metric is the relative surrogate fitting error $e_{\mathcal{S}}$, designed to evaluate the accuracy of the FNO surrogate in the high-probability region surrounding the true solution. We construct a local test set of $N=100$ samples by applying standard Gaussian perturbations to the exact parameter, $\mb{\xi}^{(i)} = \mb{\xi}_{\mathrm{ref}} + \boldsymbol{\eta}^{(i)}$ with $\boldsymbol{\eta}^{(i)} \sim \mathcal{N}(\mb{0}, \mb{I})$. Letting $u(\cdot; m_{\mb{\xi}^{(i)}})$ denote the corresponding high-fidelity state field obtained via the FDM, the surrogate error is computed as the average relative $L^2$ discrepancy:
\begin{equation}
    e_{\mathcal{S}} = \frac{1}{N} \sum_{i=1}^{N}
    \frac{\bigl\|\mathcal{F}_{\mb{\vartheta}}(m_{\mb{\xi}^{(i)}}) - u(\cdot; m_{\mb{\xi}^{(i)}})\bigr\|_{L^2(\Omega_s)}}
         {\bigl\|u(\cdot; m_{\mb{\xi}^{(i)}})\bigr\|_{L^2(\Omega_s)}}.
\end{equation}

    \end{itemize}

\subsection{Posterior Approximation: Rosenbrock Problem}
\label{subsec_rosen}

\begin{figure}[!htp]
    \centering
    \subfloat[Reference\label{rosen_marginal}]{\includegraphics[width=.3\textwidth]{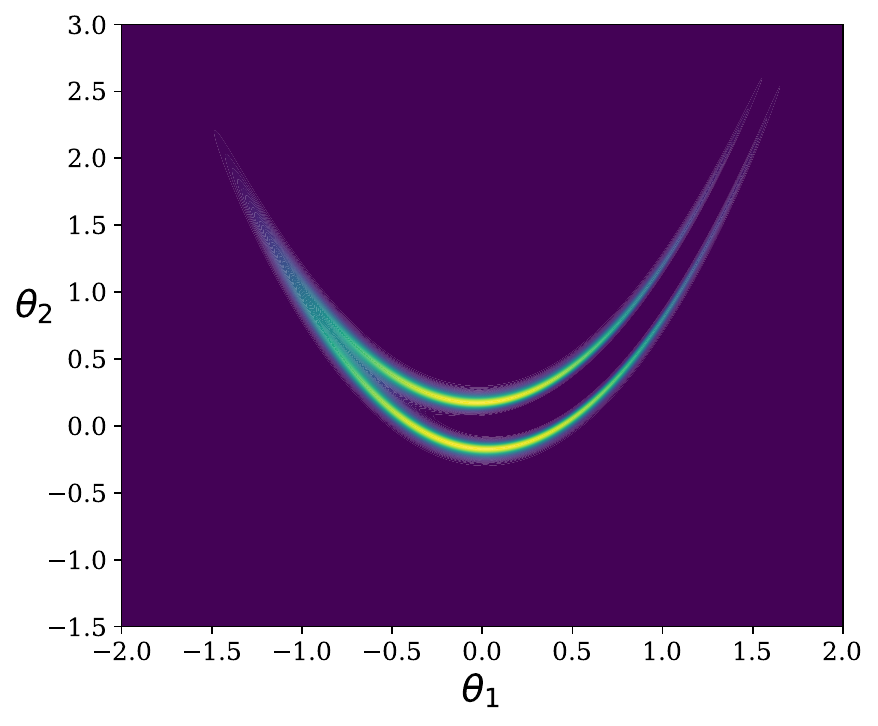}}
    \subfloat[VF (ours) \label{fig:rosen_vaekrnet}]{\includegraphics[width=.3\textwidth]{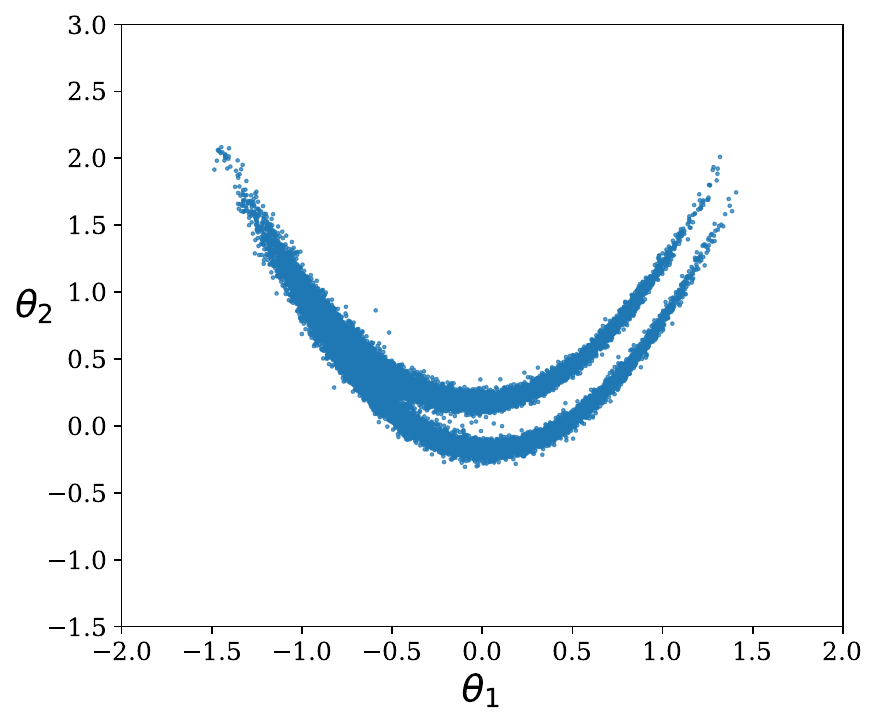}}
    \subfloat[VAE\label{fig:rosen_vae}]{\includegraphics[width=.3\textwidth]{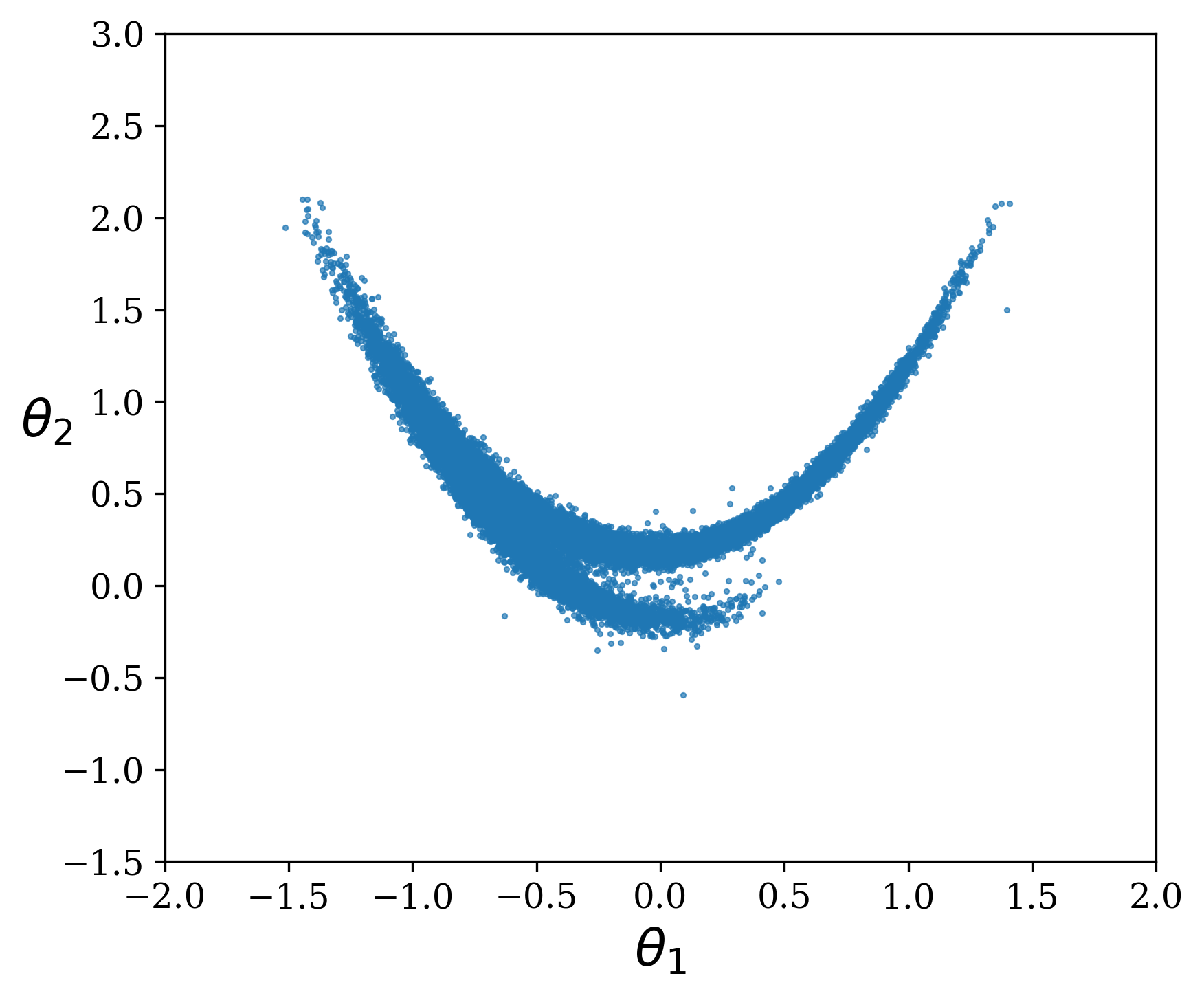}} \\ \vspace{-1em}
    \subfloat[MCMC \label{fig:rosen_mcmc}]{\includegraphics[width=.3\textwidth]{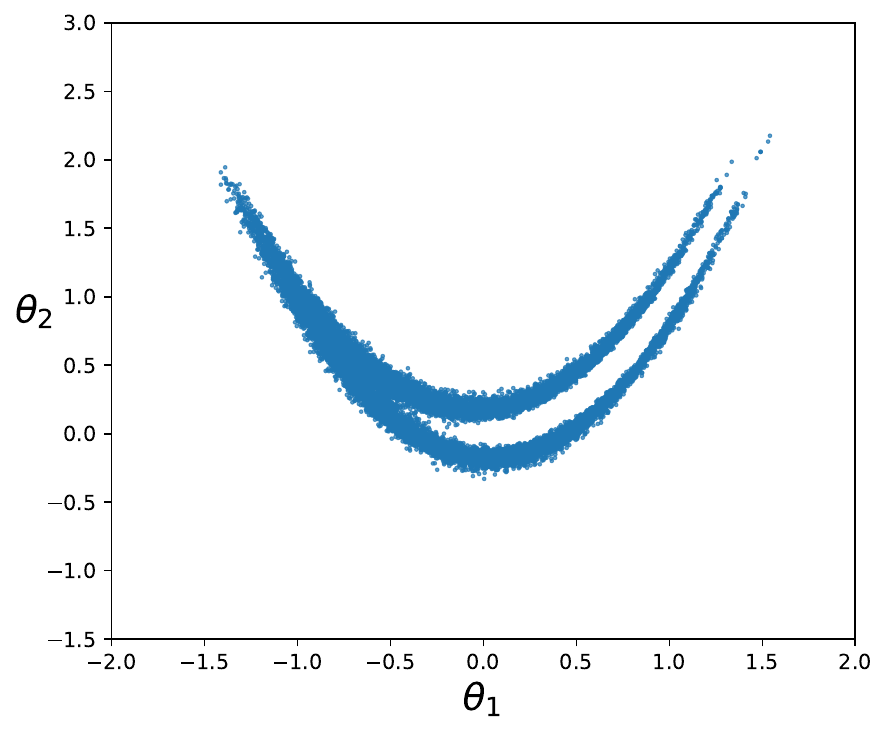}}
    \subfloat[SVGD \label{fig:rosen_svgd}]{\includegraphics[width=.3\textwidth]{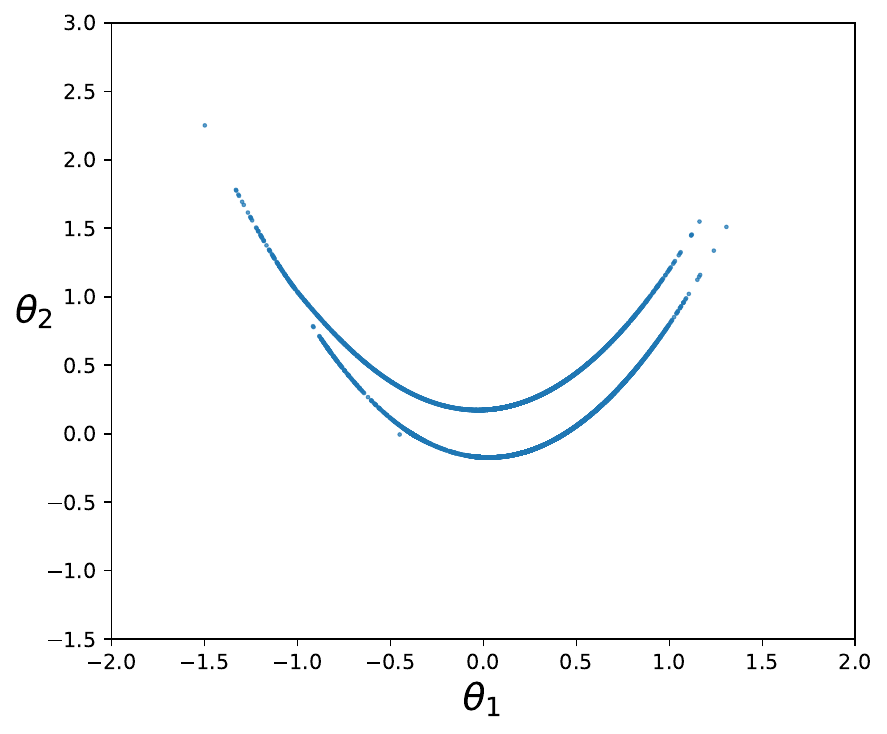}}
    \subfloat[UKI \label{fig:rosen_uki}]{\includegraphics[width=.3\textwidth]{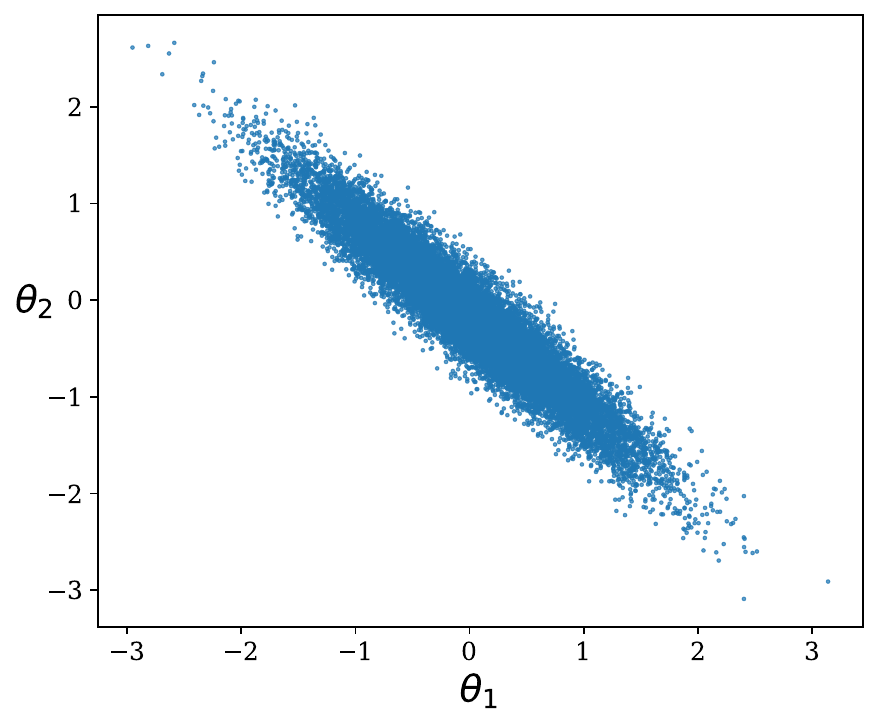}}
    
    \caption{Comparison of the 2D marginal posterior distributions on the $\xi_{1}$-$\xi_{2}$ plane for the Rosenbrock inverse problem. Panel (a) shows the ground truth reference density, while panels (b)-(f) display the sampling results obtained by VF, VAE, MCMC, SVGD, and UKI, respectively.}
    \label{fig:rosenbrock_result}
\end{figure}

To evaluate the representational capacity of the VF model for complex distributions, we consider a 100D inverse problem with a bimodal, banana-shaped posterior~\cite{che2025stable}. Let
$
\mb{\xi} = [\xi_1, \xi_2, \mb{\xi}^c] \in \mathbb{R}^{100}
$
where $\xi_1,\xi_2 \in \mathbb{R}$ are the primary parameters of interest and $\mb{\xi}^c \in \mathbb{R}^{98}$ denotes the complementary high-dimensional parameters. The forward map and observation are defined by
$$
\mathbf{F}(\xi_1,\xi_2) = 
\begin{bmatrix}
\displaystyle \log\left(100(\xi_2 - \xi_1^2)^2 + (1 - \xi_1)^2\right) / 0.3 \\
\xi_1 \\
\xi_2
\end{bmatrix}
\quad \text{and} \quad
\mb{y} = 
\begin{bmatrix}
\log(101) \\
0 \\
0
\end{bmatrix}.
$$
We consider the regularized least-squares objective:
\begin{equation}
\Phi(\boldsymbol{\xi}) = \frac{1}{2} \| \mathbf{F}(\xi_1, \xi_2) - \mb{y} \|_2^2 + \frac{1}{2} \| \boldsymbol{\xi}^c - K \boldsymbol{\xi} \|_2^2,
\end{equation}
where $K \in \mathbb{R}^{98 \times 100}$ is the all-ones matrix. In this problem, we compare our proposed model against several baselines. Notably, we include a vanilla VAE baseline derived by ablating the novel components of our architecture, ensuring an equivalent parameter count. We additionally compare against affine-invariant MCMC utilizing the \texttt{emcee} package~\cite{foreman2013emcee}, SVGD, and UKI. Further implementation details are deferred to Appendix~\ref{sec:rosen_details}. We draw 20,000 samples for each method, with the comparative results depicted in Figure~\ref{fig:rosenbrock_result}.

As illustrated in Figure~\ref{rosen_marginal}, the posterior distribution is bimodal and banana-shaped. Since UKI relies on a Gaussian approximation, it fails to capture the true posterior geometry. In contrast, VF successfully recovers the curved bimodal structure and produces a posterior approximation that closely matches the MCMC result, which serves as a high-fidelity sampling baseline, and it better captures the complex posterior geometry than the SVGD baseline. Furthermore, unlike standard VAE which struggles to capture the complete curved bimodal structure, VF demonstrates superior fitting capacity through its flow-enhanced architecture.

\subsection{1D Darcy Flow}
\label{subsec:1d_darcy}

We consider the steady-state 1D Darcy Flow on the domain $\Omega = [0,1]$. The
forward model is to find the pressure field $p(\mb{x})$ in a porous medium defined by a positive permeability
field $m_{\mb{\xi}}(x)$:
\begin{align}
    - \frac{d}{d x} \Big( \exp(m_{\mb{\xi}}(x)) \frac{d p}{d x} \Big) &= f(x), \quad x \in (0,1), \\
    p(0) = p(1) &= 0,
\end{align}
with constant source term $f(x) \equiv 1$.  
The log-permeability field $m_{\mb{\xi}}(x)$ is modeled as a Gaussian random field with covariance operator $\mathcal{C} = \sigma^2 (-\Delta + \tau^2)^{-l}$, where $\Delta$ is the 1D Laplacian with homogeneous Neumann boundary conditions. We set the parameters as $\sigma = 2$, $\tau = 1$, and $l = 2$. The truncated Karhunen-Loève expansion of $m_{\mb{\xi}}(x)$ is then
\begin{equation}
    m_{\mb{\xi}}(x) = \sum_{k=1}^{d} \sqrt{\lambda_k}\, \xi_k\, \psi_k(x),
\end{equation}
with the eigenfunctions and eigenvalues given explicitly by
\begin{equation}
    \psi_k(x) = \sqrt{2} \cos(k \pi x), \qquad \lambda_k = \sigma^2 (k^2 \pi^2 + \tau^2)^{-l}.
\end{equation}
The vector $\boldsymbol{\xi} \in \mathbb{R}^d$ forms the finite-dimensional parameter space for inference. In our experiments, we infer the log-permeability field from 31 uniformly spaced pressure observations across the domain. The high-fidelity forward solver is implemented using the Finite Difference Method (FDM) on a 1024-point spatial grid.

\begin{table}[t]
\centering
\small
\renewcommand{\arraystretch}{1.12}
\setlength{\tabcolsep}{7pt}
\caption{Relative inversion error $e_{\mathcal{I}}$ of 1D Darcy Flow problem under different noise amplitudes and $d$ values. Lower is better. Results are averaged over 3 runs per experiment.}
\label{tab:comparison_1d_darcy}

\resizebox{0.9\textwidth}{!}{%
\begin{tabular}{lccc@{\hspace{1.2em}}ccc}
\toprule
\multirow{2}{*}{\textbf{Method}}
& \multicolumn{3}{c}{$d=32$}
& \multicolumn{3}{c}{$d=64$} \\
\cmidrule(lr){2-4} \cmidrule(lr){5-7}
& \textbf{1\%} & \textbf{5\%} & \textbf{10\%}
& \textbf{1\%} & \textbf{5\%} & \textbf{10\%} \\
\midrule
pCN        & 0.3164 & 0.5036 & 0.5301 & 0.3160 & 0.5199 & 0.5282 \\
SVGD-FNO   & 0.3092 & 0.5188 & 0.5275 & 0.3072 & 0.5189 & 0.5274 \\
UKI-FDM    & 0.3306 & 0.4924 & 0.4514 & 0.2843 & 0.4809 & 0.4633 \\
UKI-FNO    & 0.2525 & 0.4444 & 0.4441 & 0.2982 & 0.4473 & 0.4487 \\
\textbf{Ours} & \textbf{0.2292} & \textbf{0.3569} & \textbf{0.4076} & \textbf{0.2073} & \textbf{0.3569} & \textbf{0.4098} \\
\bottomrule
\end{tabular}%
}\vspace{-1em}
\end{table}

\begin{figure}[!htb]
\centering
\setlength{\abovecaptionskip}{4pt}
\setlength{\belowcaptionskip}{-2pt}

\includegraphics[width=0.32\textwidth]{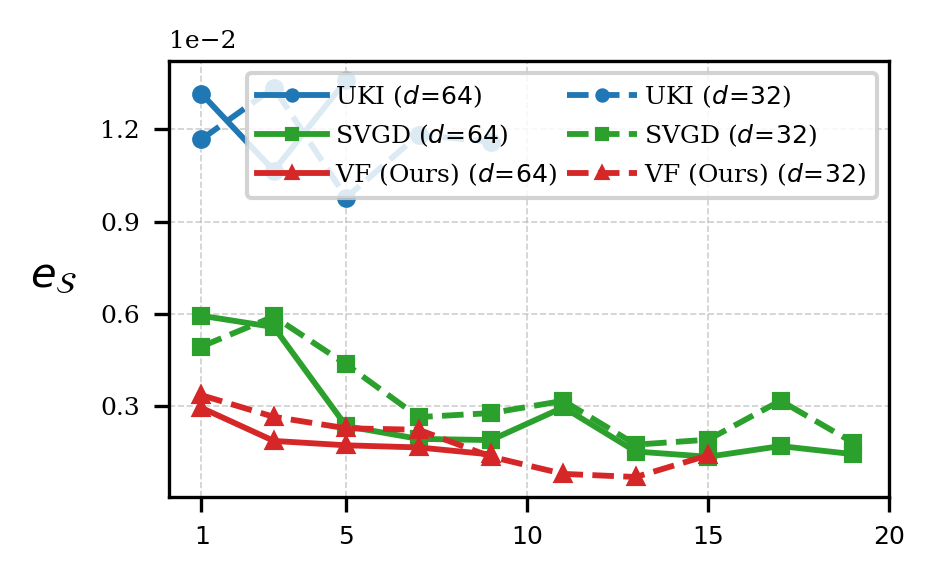}
\includegraphics[width=0.32\textwidth]{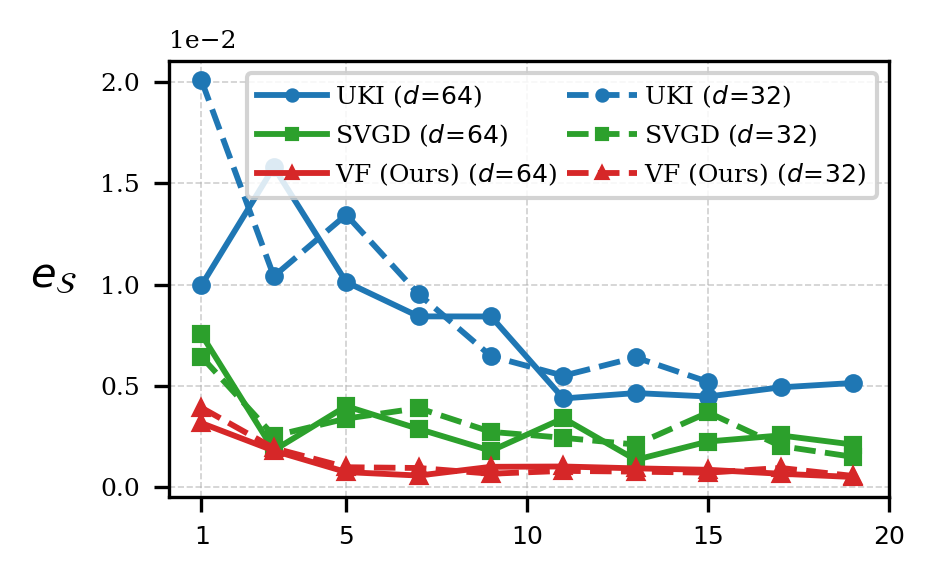}
\includegraphics[width=0.32\textwidth]{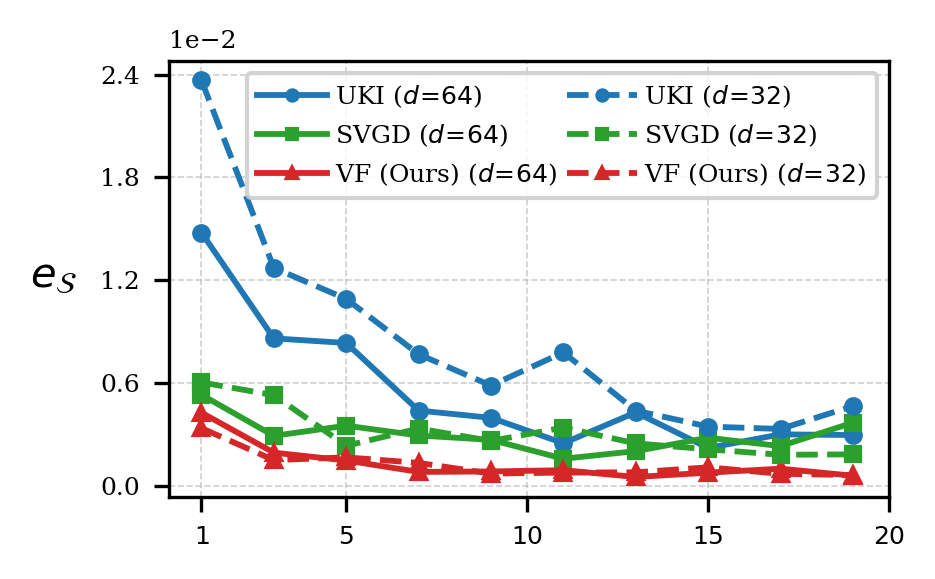}

\caption{Surrogate fitting error $e_{\mathcal{S}}$ across adaptive stages of 1D Darcy Flow problem. Columns: noise levels $\delta \in \{1\%, 5\%, 10\%\}$. Solid/dashed lines: $d=64/32$. Overall, our method converges to lower errors in fewer stages.}
\label{fig:fitting_error_1d_darcy}
\end{figure}

\begin{figure}[htbp]
\centering
\includegraphics[width=0.196\textwidth]{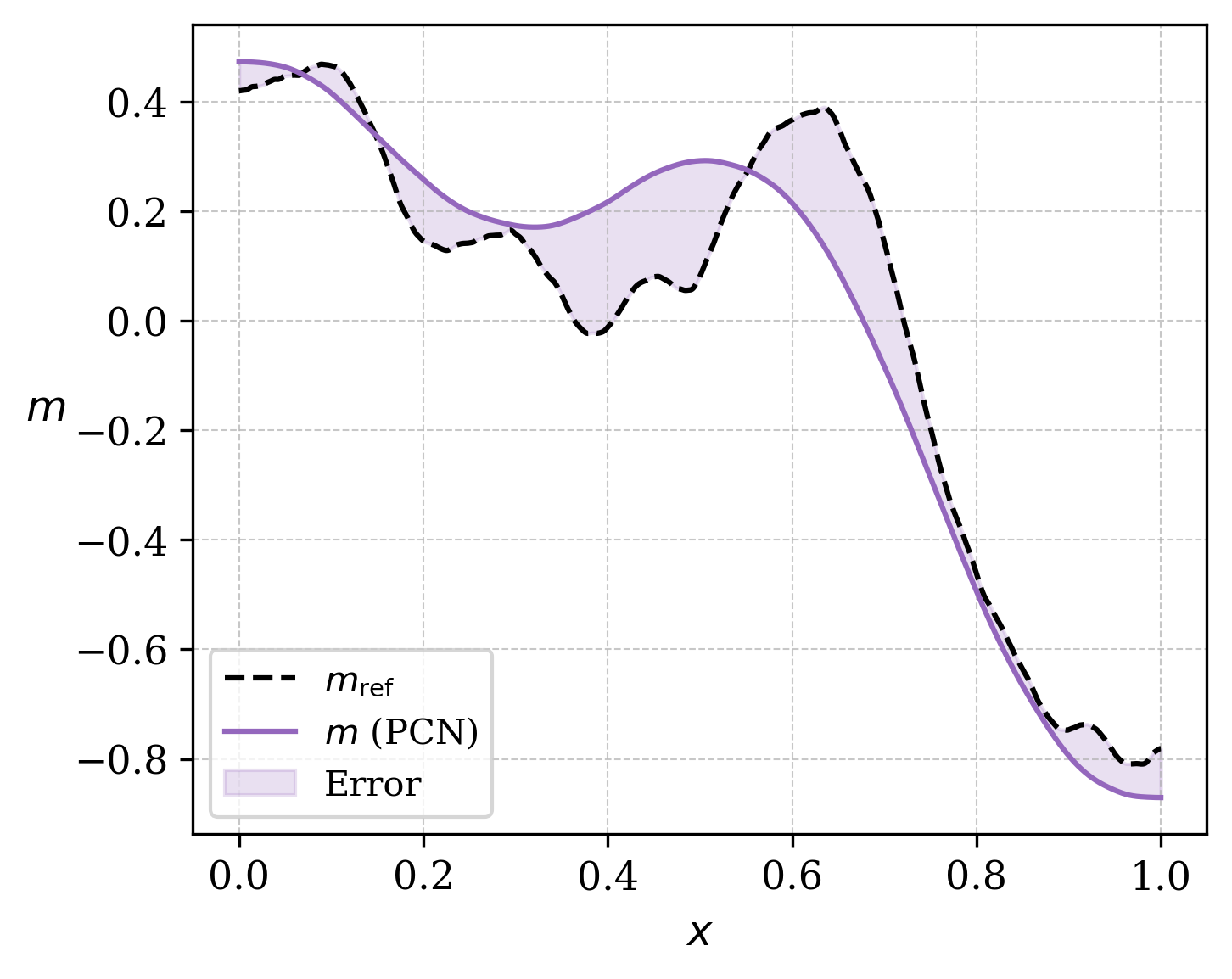}%
\hspace{0.5mm}%
\includegraphics[width=0.196\textwidth]{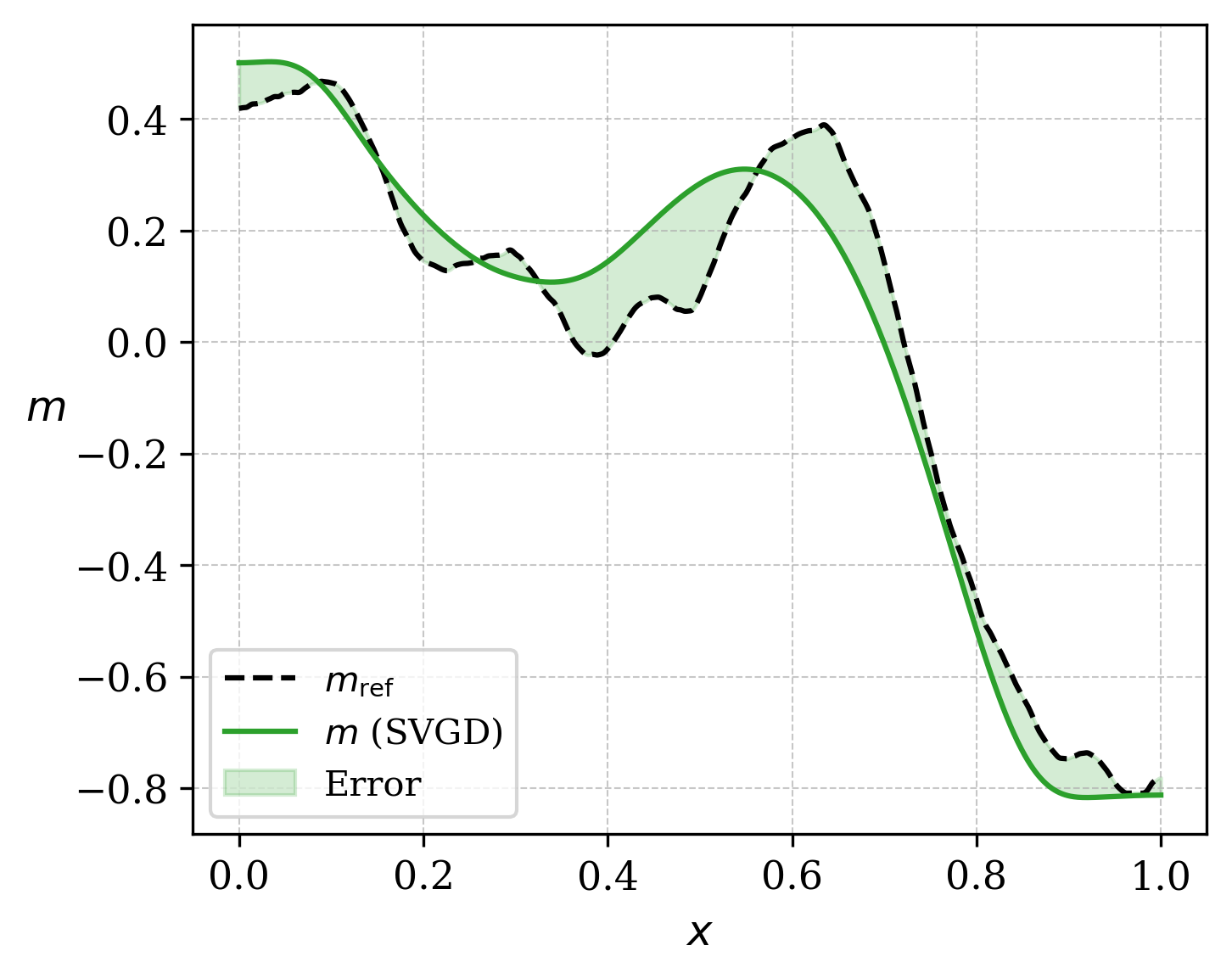}%
\hspace{0.5mm}%
\includegraphics[width=0.196\textwidth]{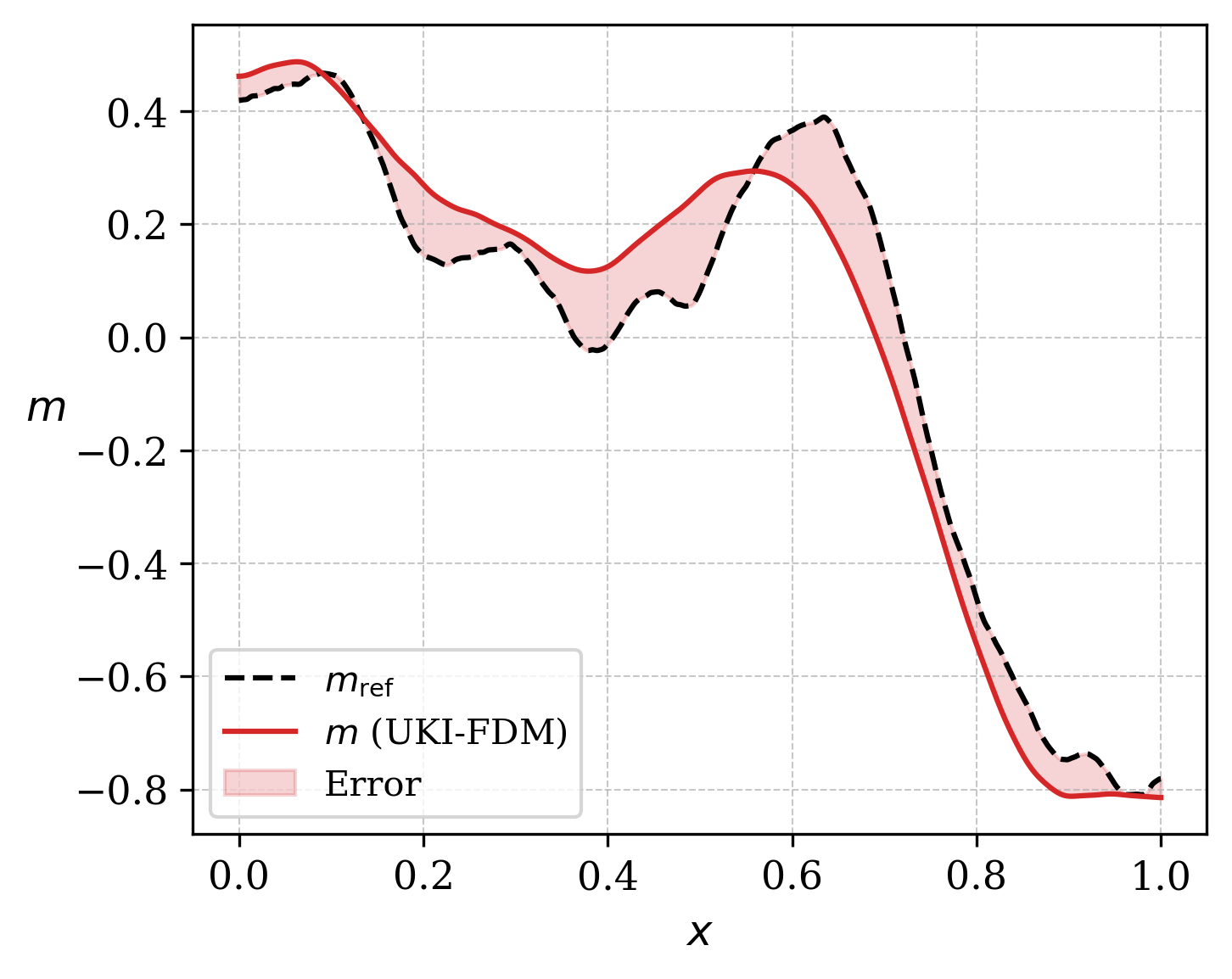}%
\hspace{0.5mm}%
\includegraphics[width=0.196\textwidth]{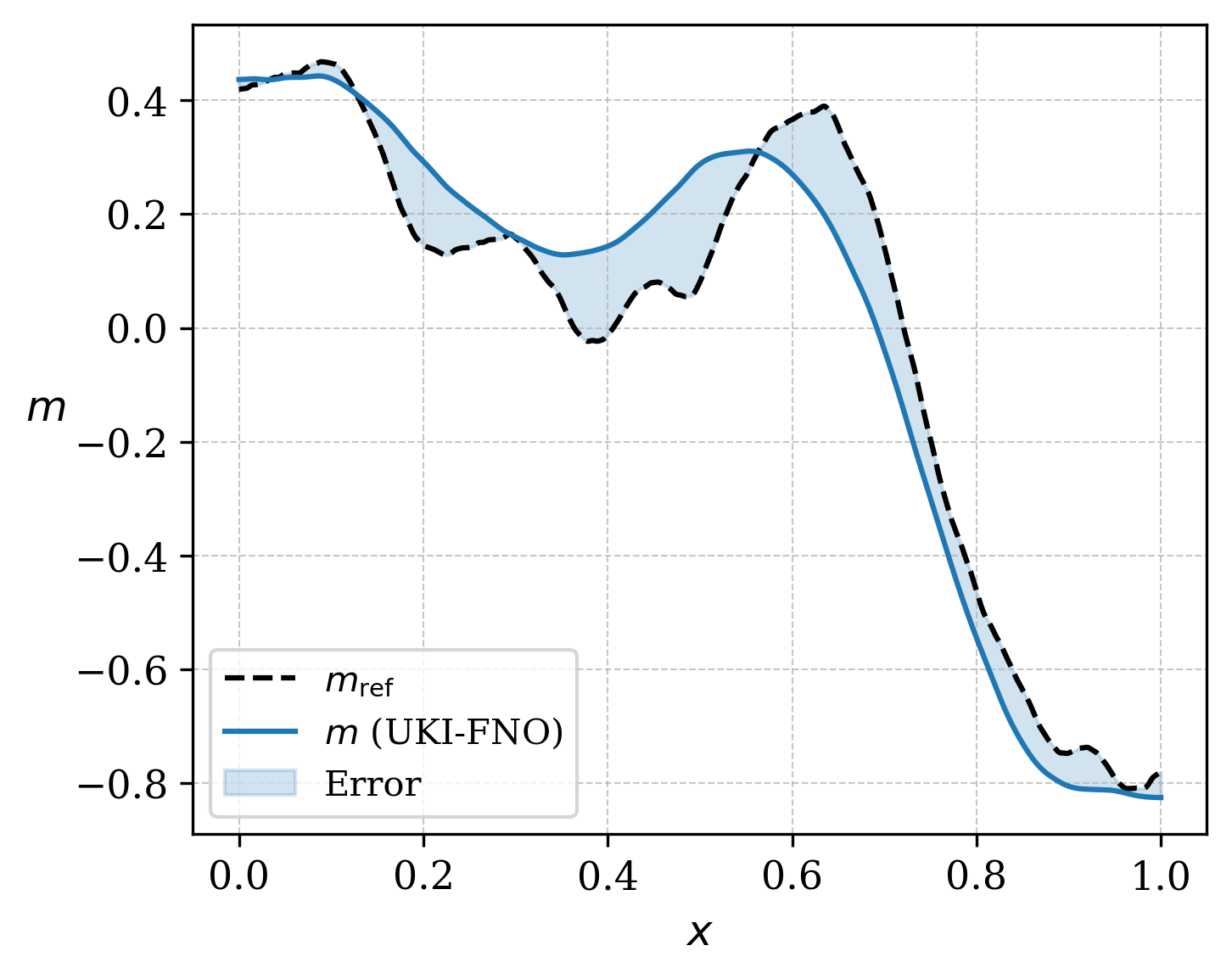}%
\hspace{0.5mm}%
\includegraphics[width=0.196\textwidth]{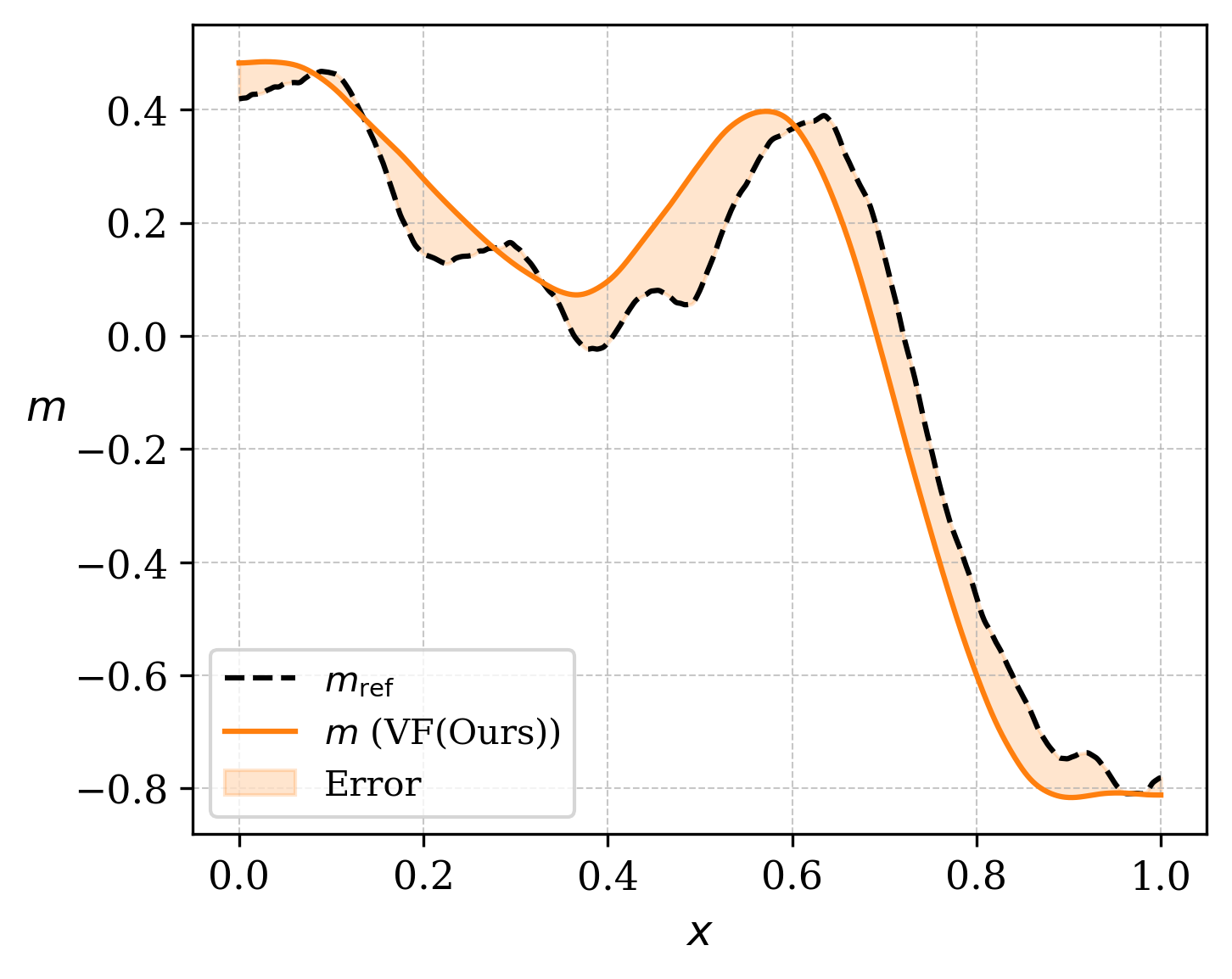}\\
\vspace{0.5mm}

\includegraphics[width=0.196\textwidth]{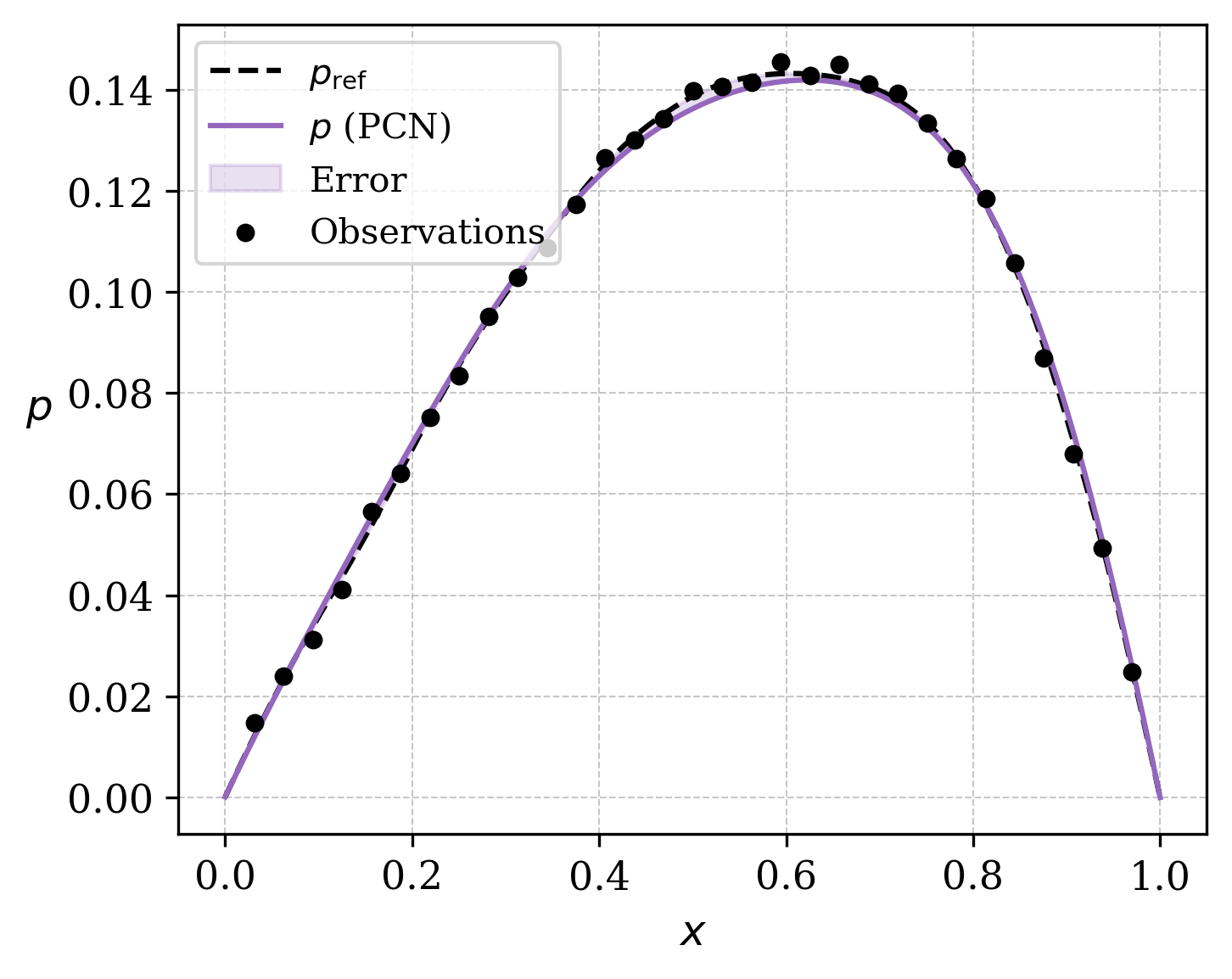}%
\hspace{0.5mm}%
\includegraphics[width=0.196\textwidth]{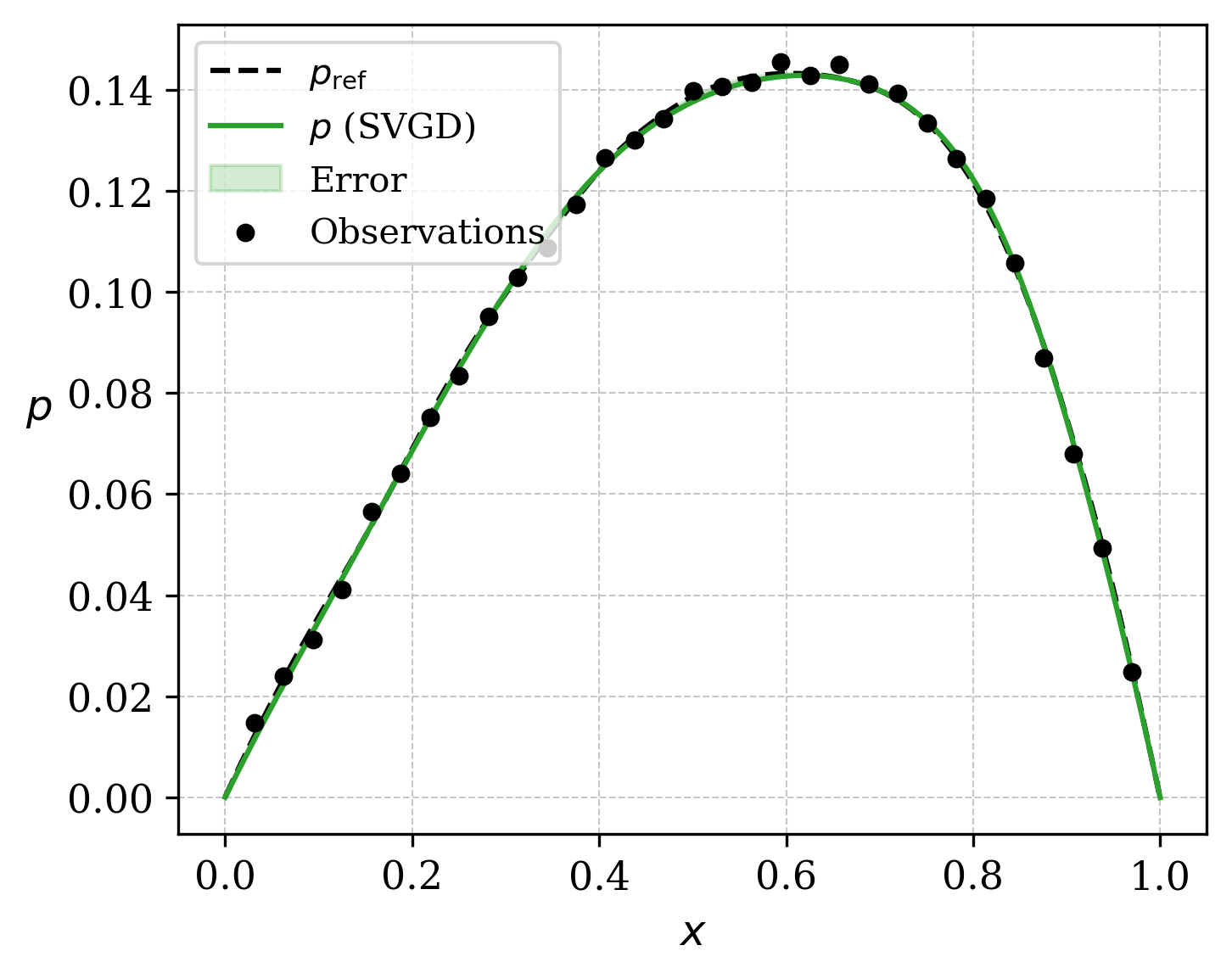}%
\hspace{0.5mm}%
\includegraphics[width=0.196\textwidth]{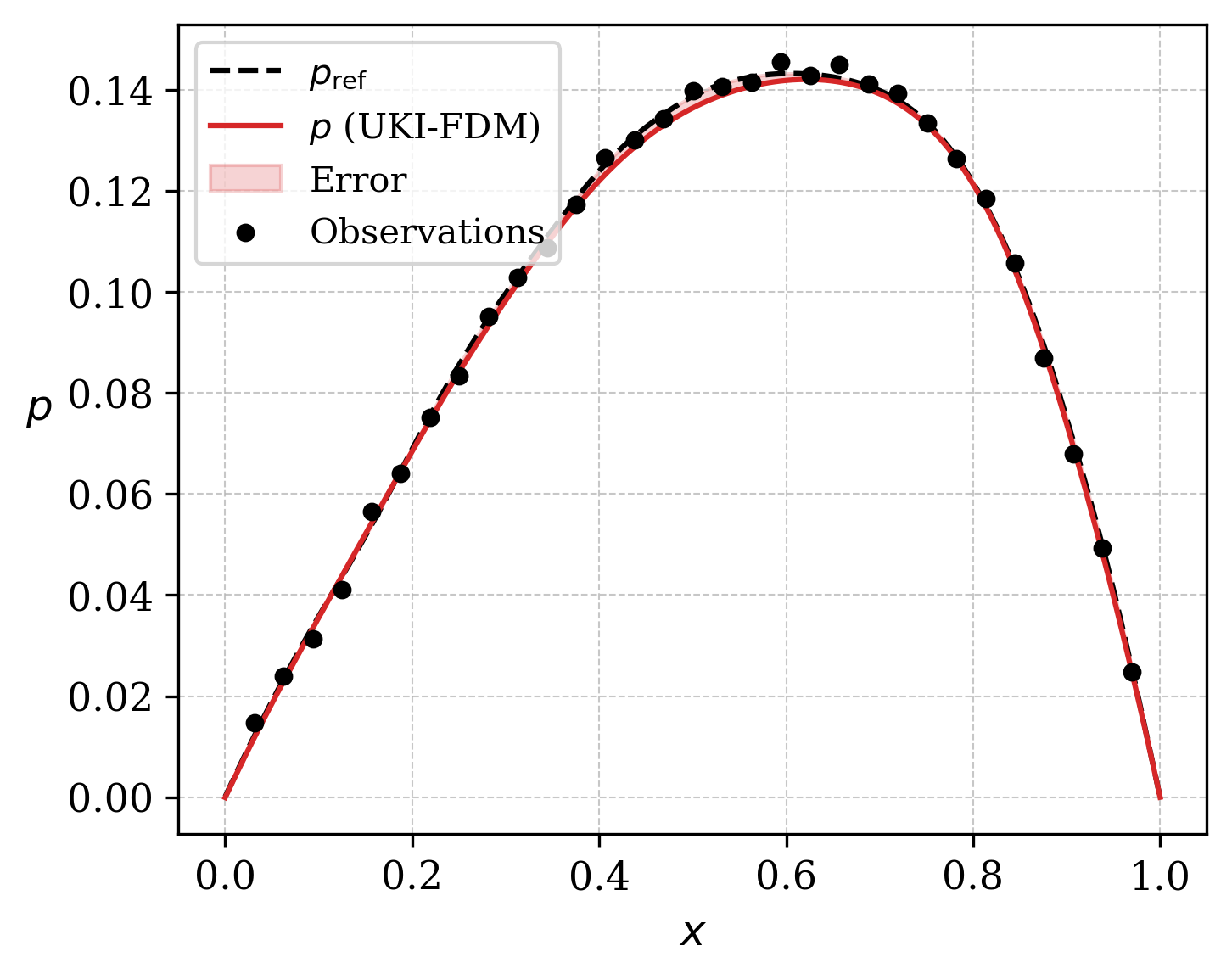}%
\hspace{0.5mm}%
\includegraphics[width=0.196\textwidth]{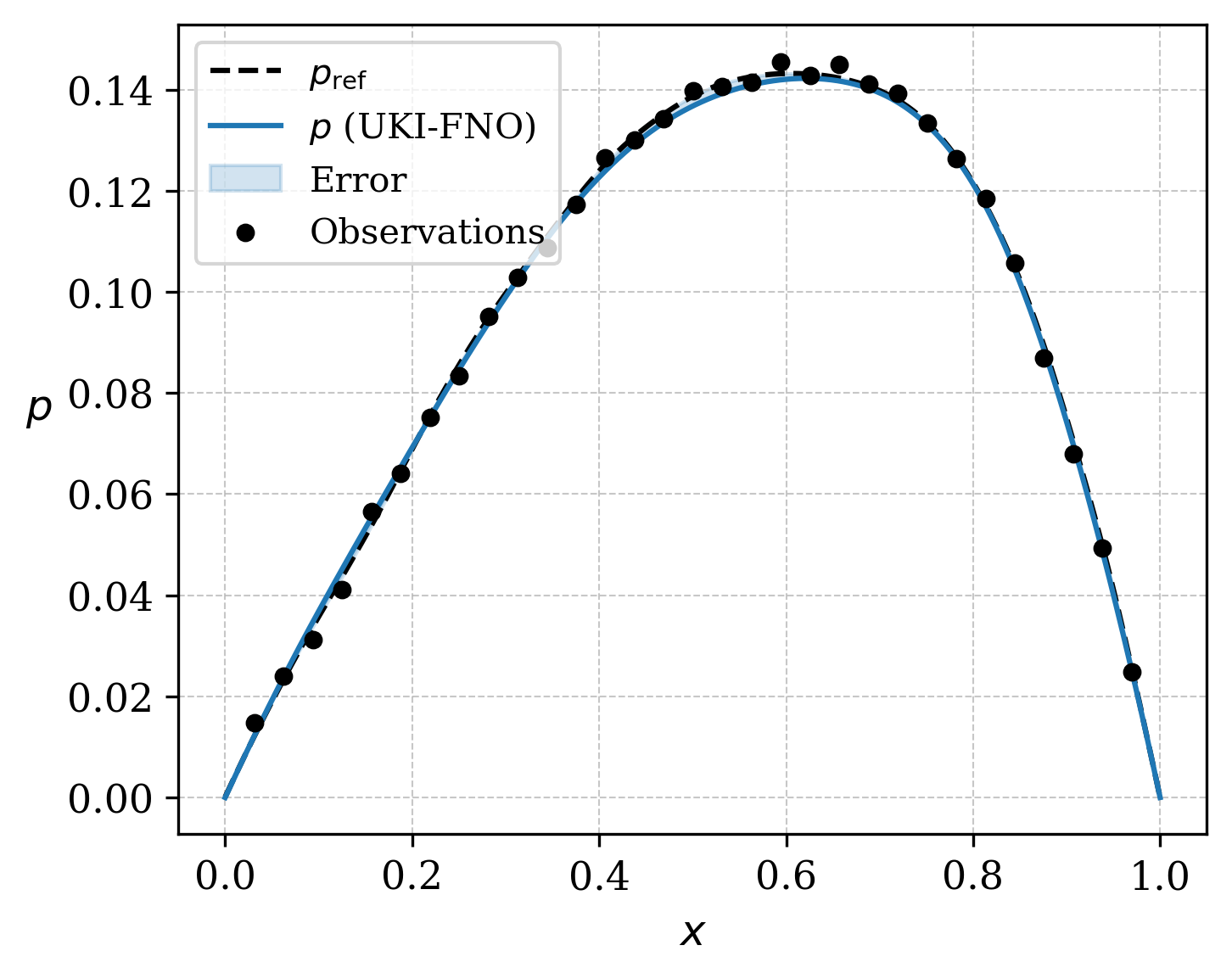}%
\hspace{0.5mm}%
\includegraphics[width=0.196\textwidth]{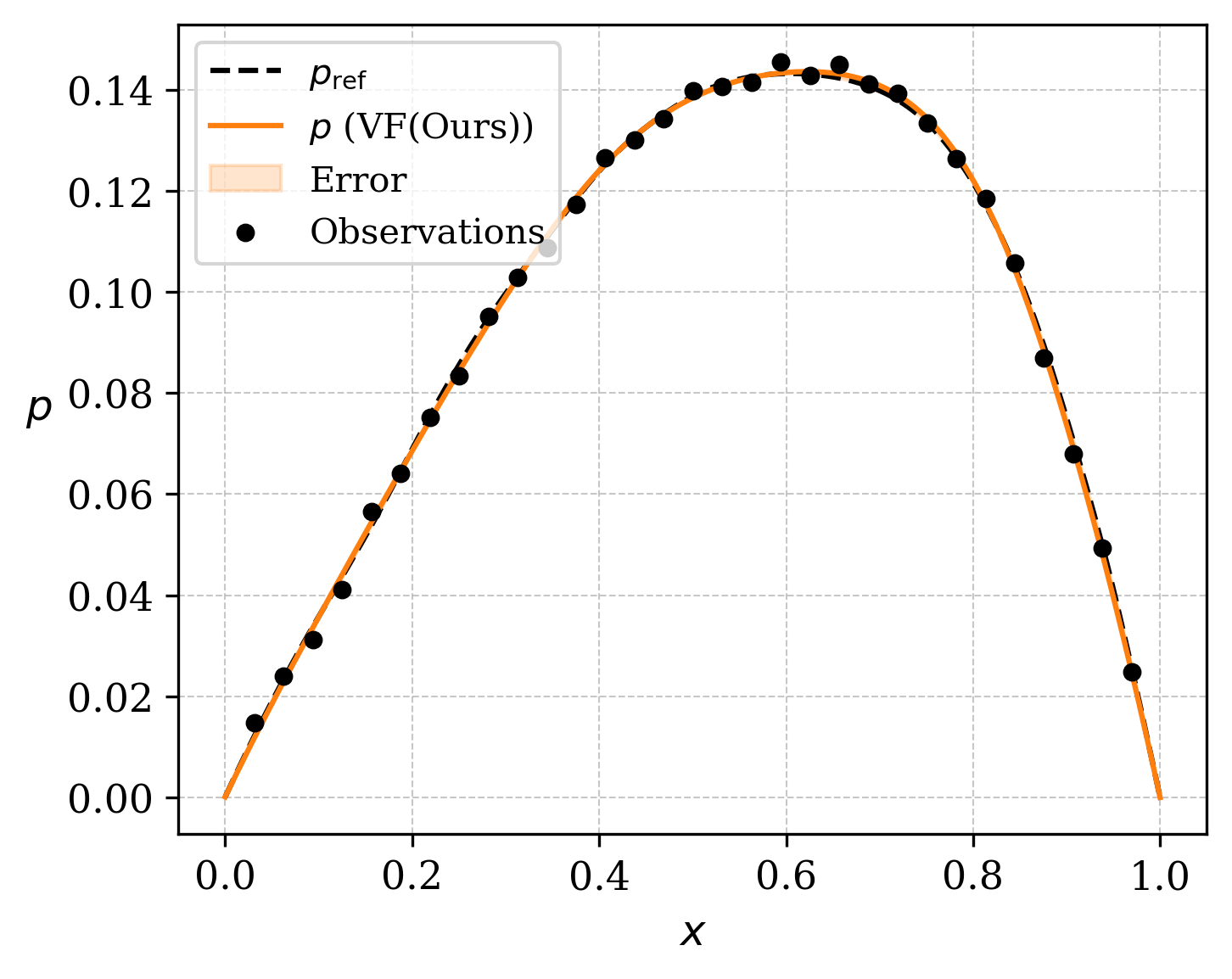}\\

\vspace{-1mm}
\centerline{\small (a) Noise level $\delta=1\%$}
\vspace{1.5mm}

\includegraphics[width=0.196\textwidth]
{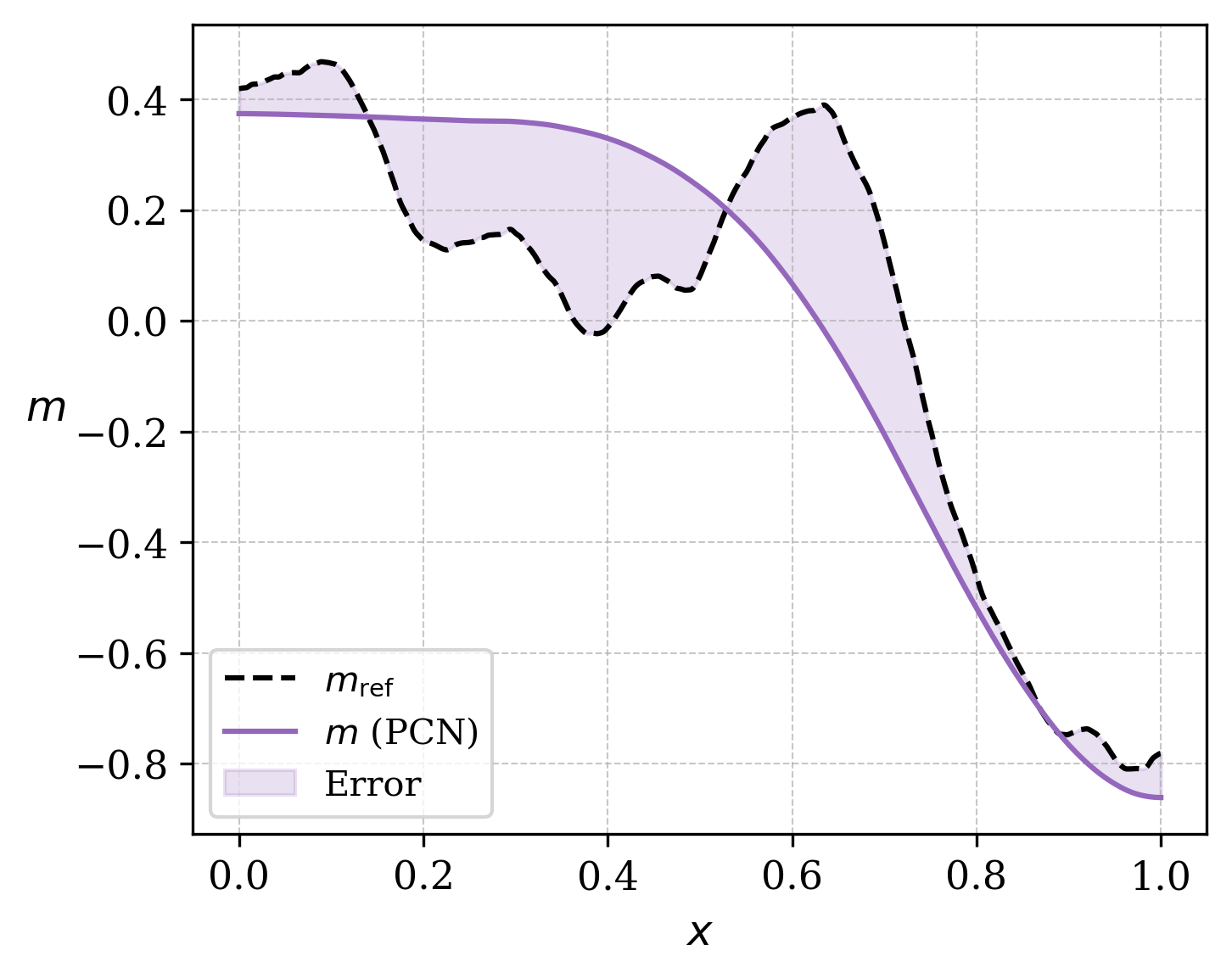}%
\hspace{0.5mm}%
\includegraphics[width=0.196\textwidth]{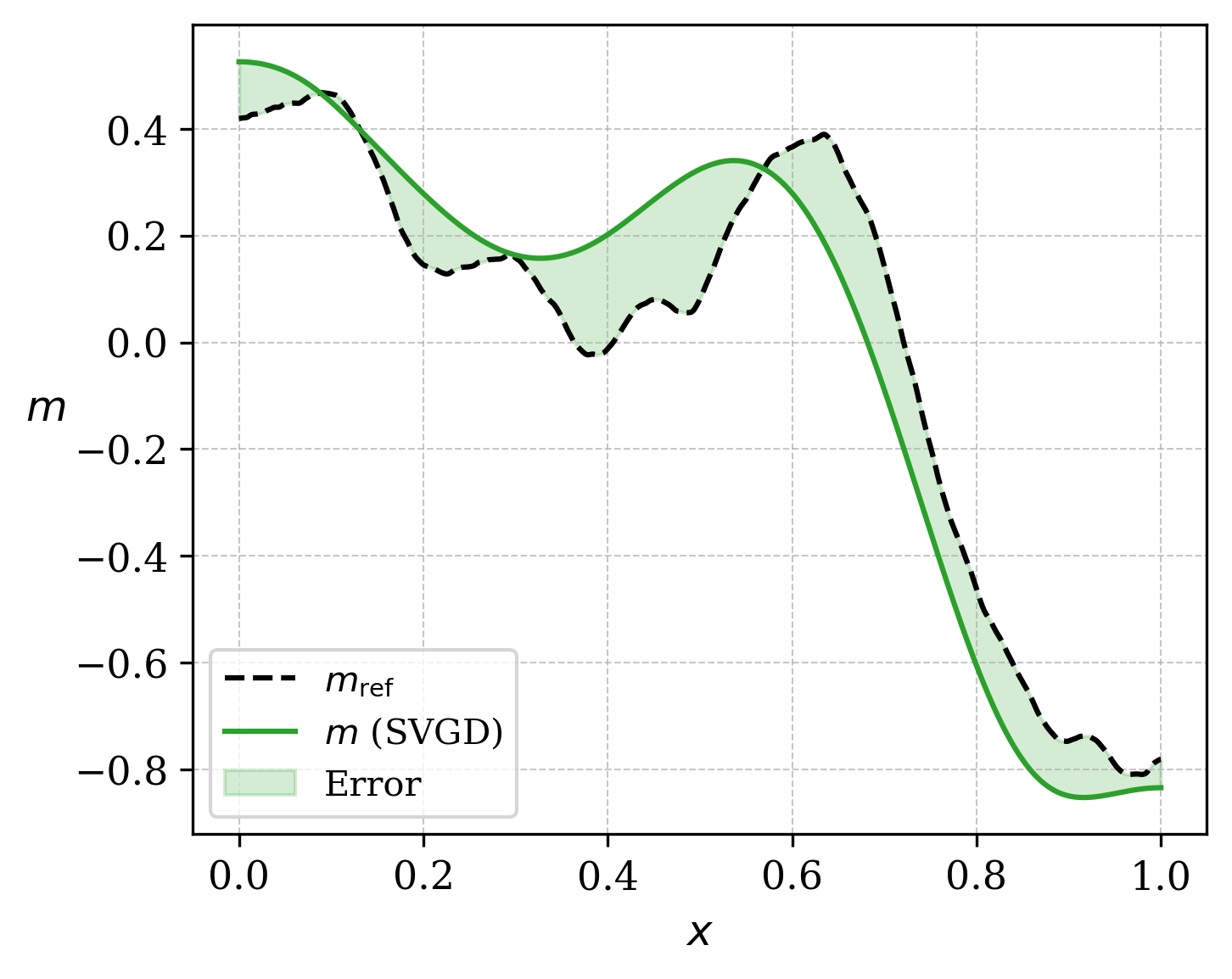}%
\hspace{0.5mm}%
\includegraphics[width=0.196\textwidth]{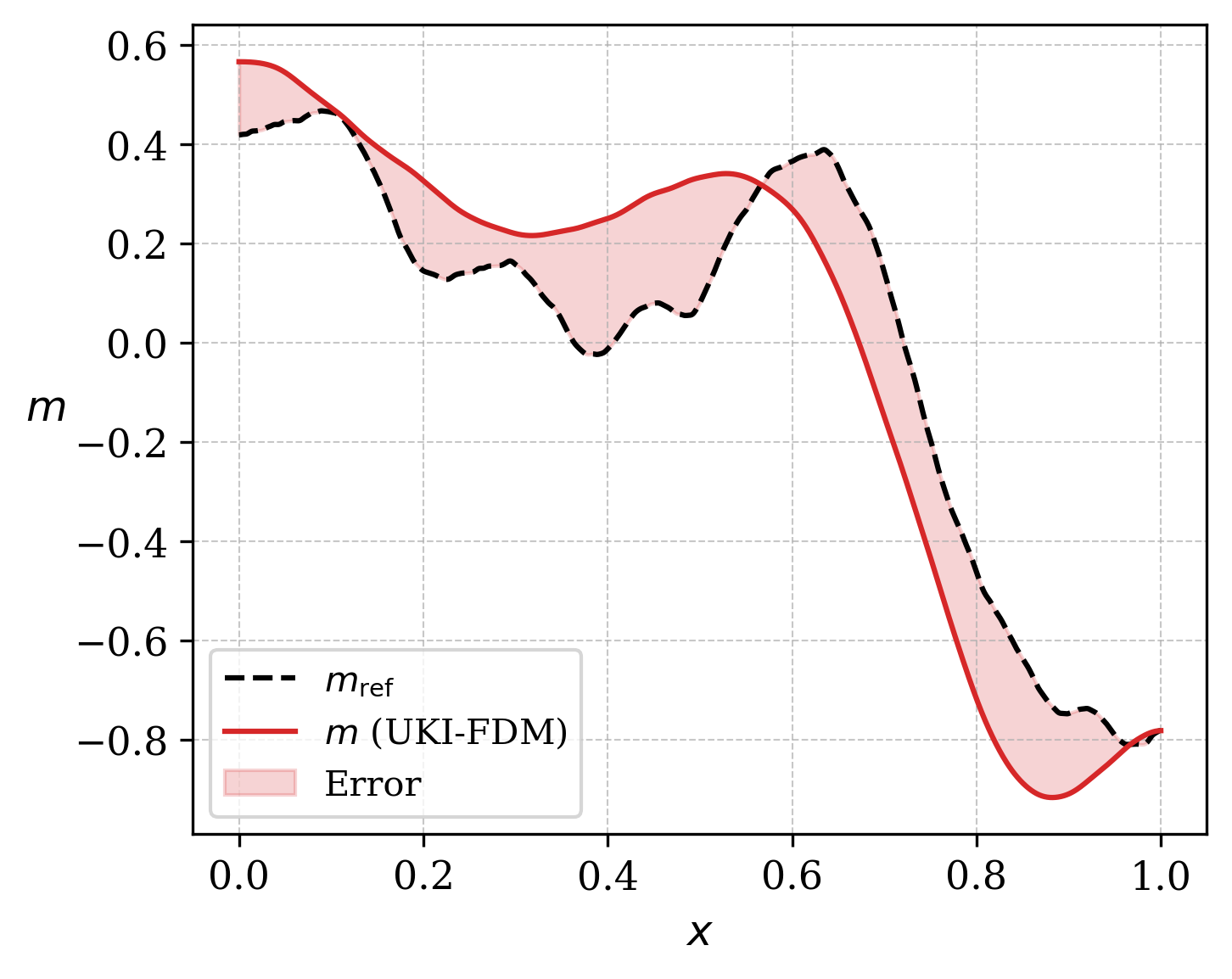}%
\hspace{0.5mm}%
\includegraphics[width=0.196\textwidth]{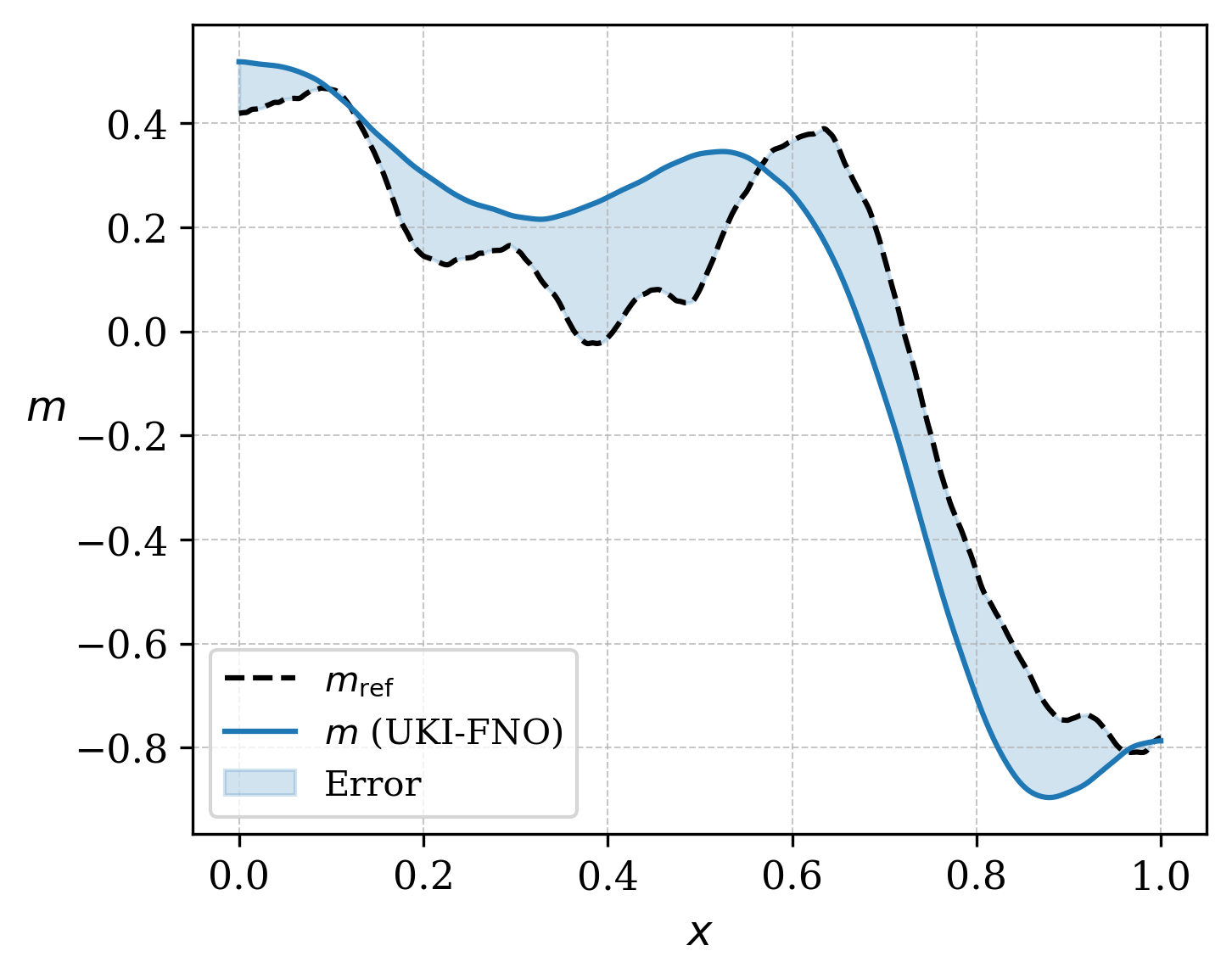}%
\hspace{0.5mm}%
\includegraphics[width=0.196\textwidth]{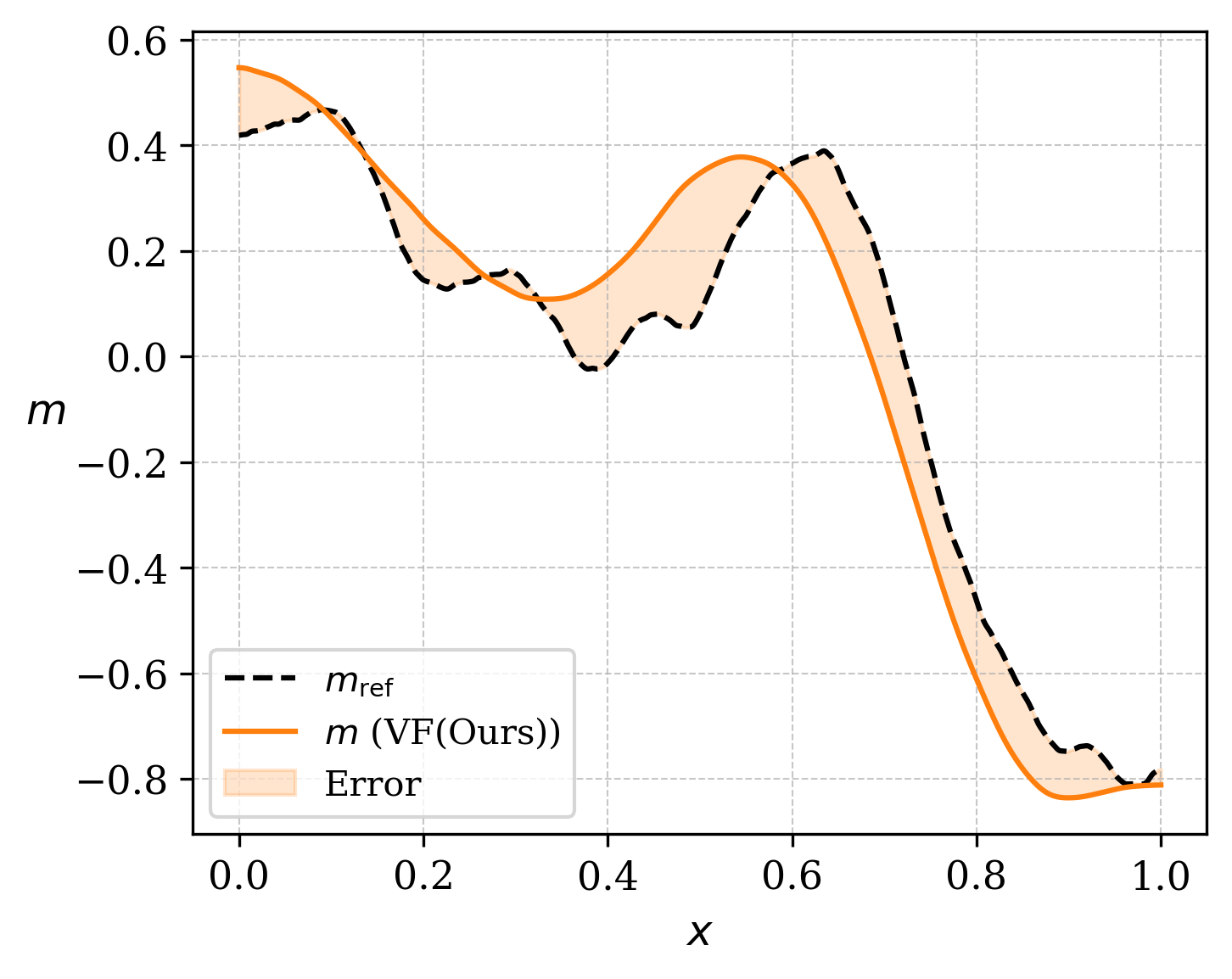}\\
\vspace{0.5mm}

\includegraphics[width=0.196\textwidth]{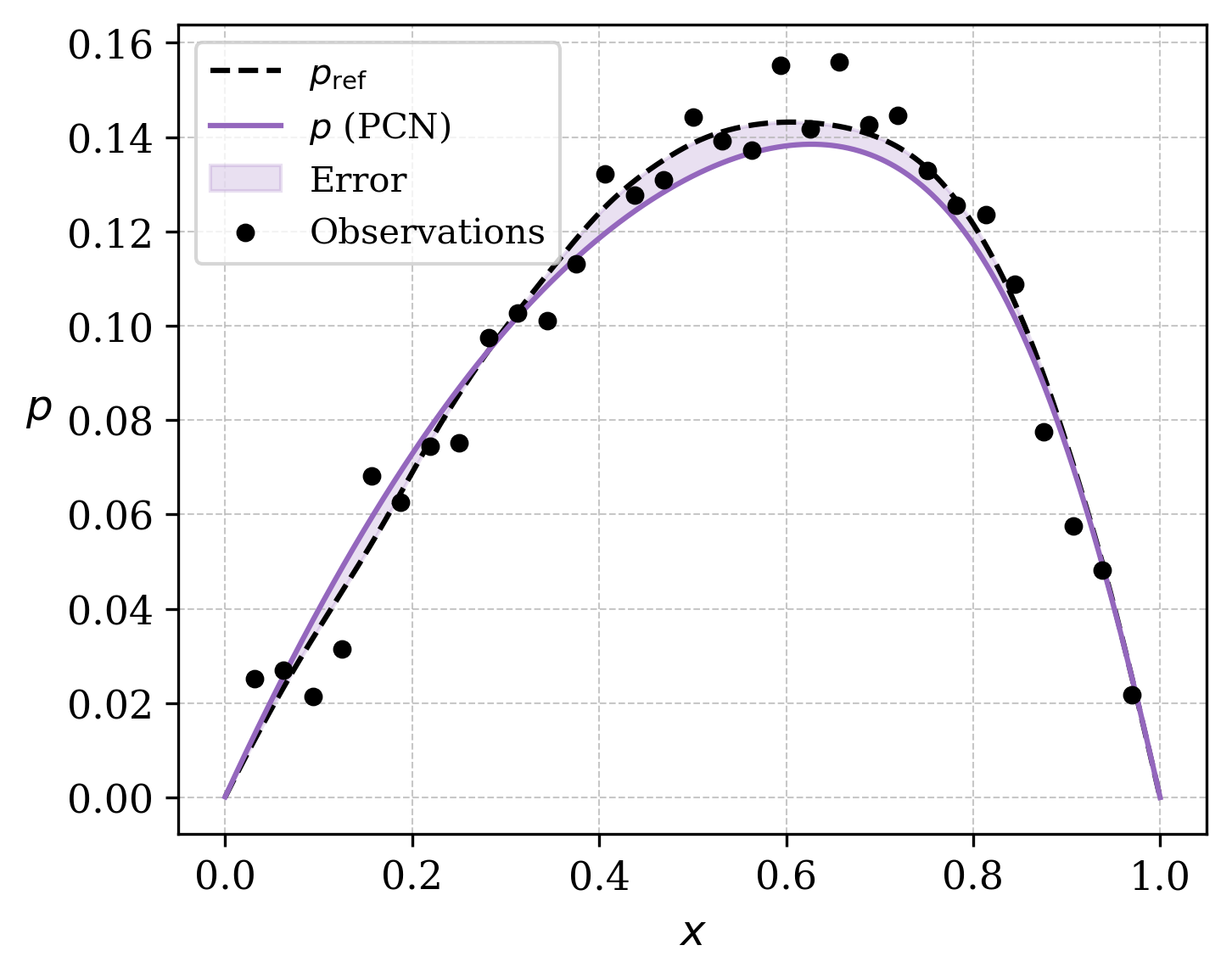}%
\hspace{0.5mm}%
\includegraphics[width=0.196\textwidth]{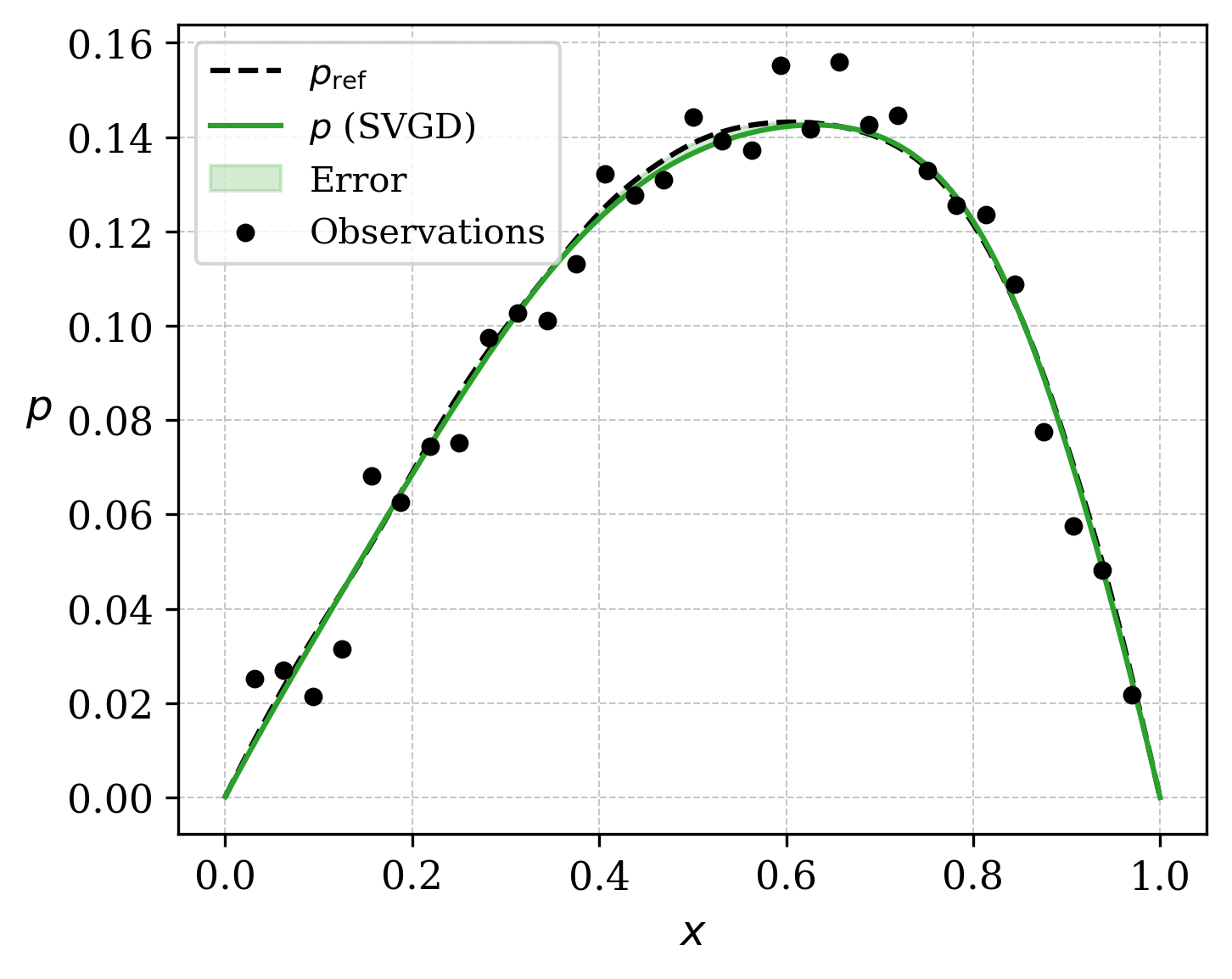}%
\hspace{0.5mm}%
\includegraphics[width=0.196\textwidth]{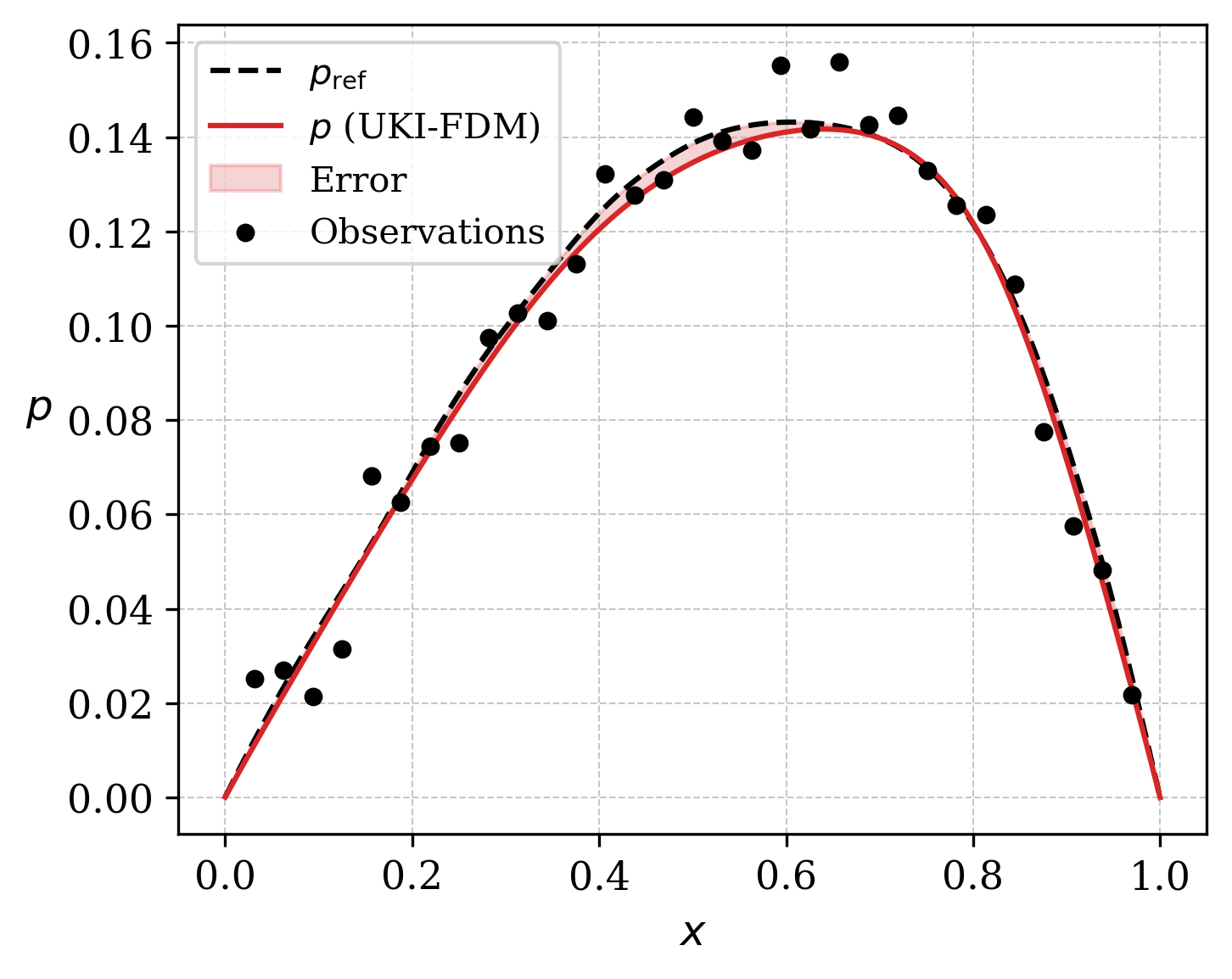}%
\hspace{0.5mm}%
\includegraphics[width=0.196\textwidth]{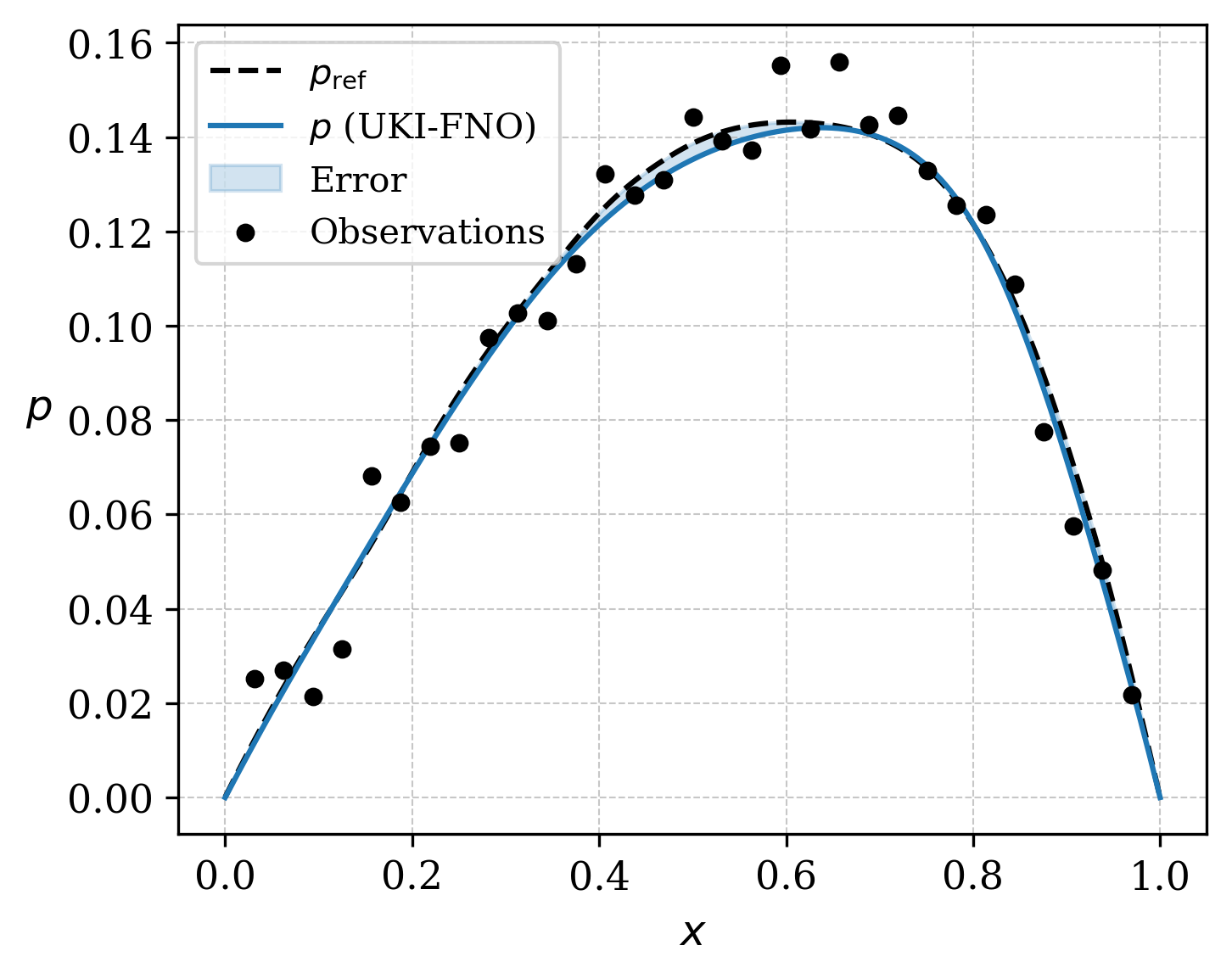}%
\hspace{0.5mm}%
\includegraphics[width=0.196\textwidth]{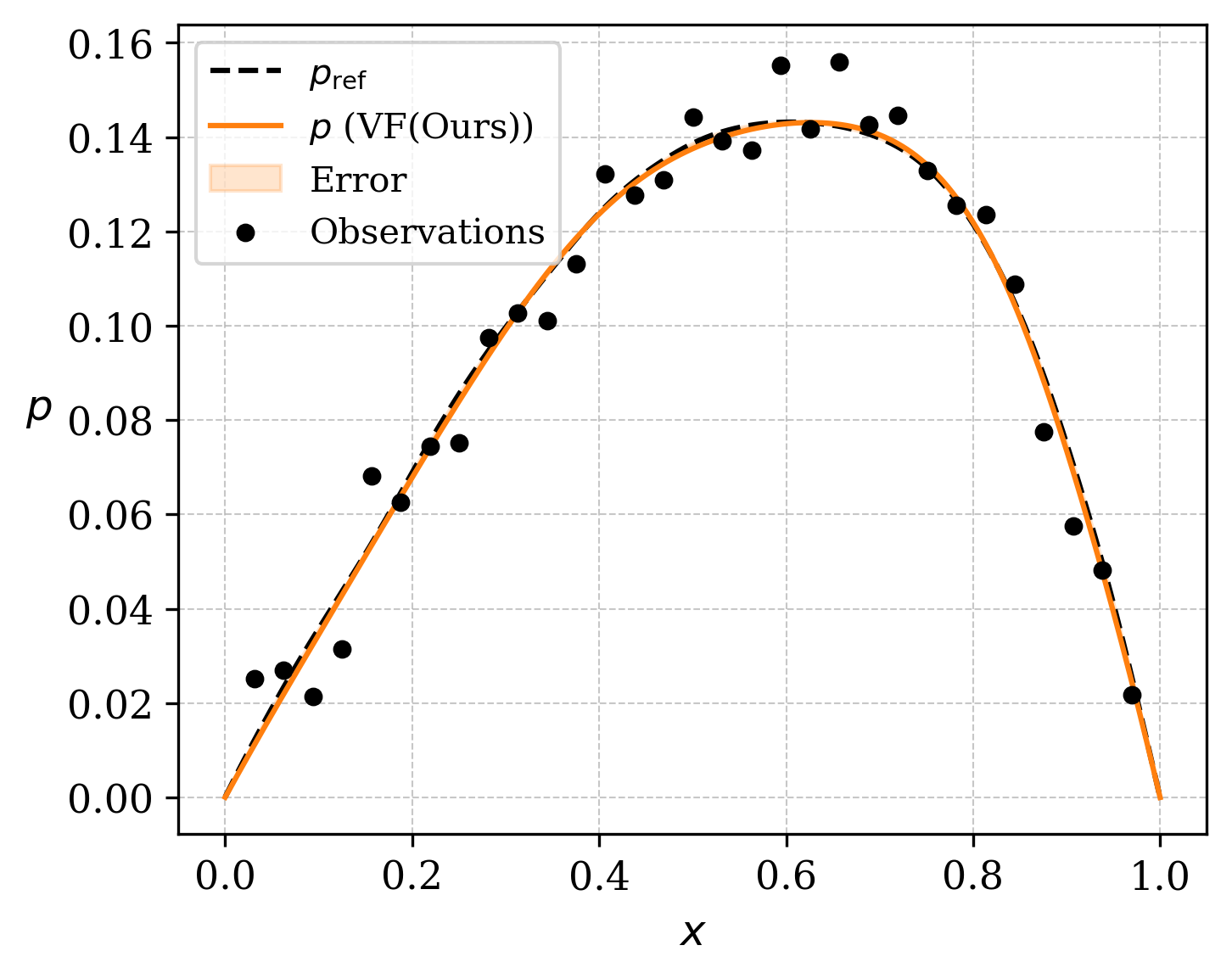}

\vspace{-1mm}
\centerline{\small (b) Noise level $\delta=5\%$}
\vspace{1.5mm}

\includegraphics[width=0.196\textwidth]{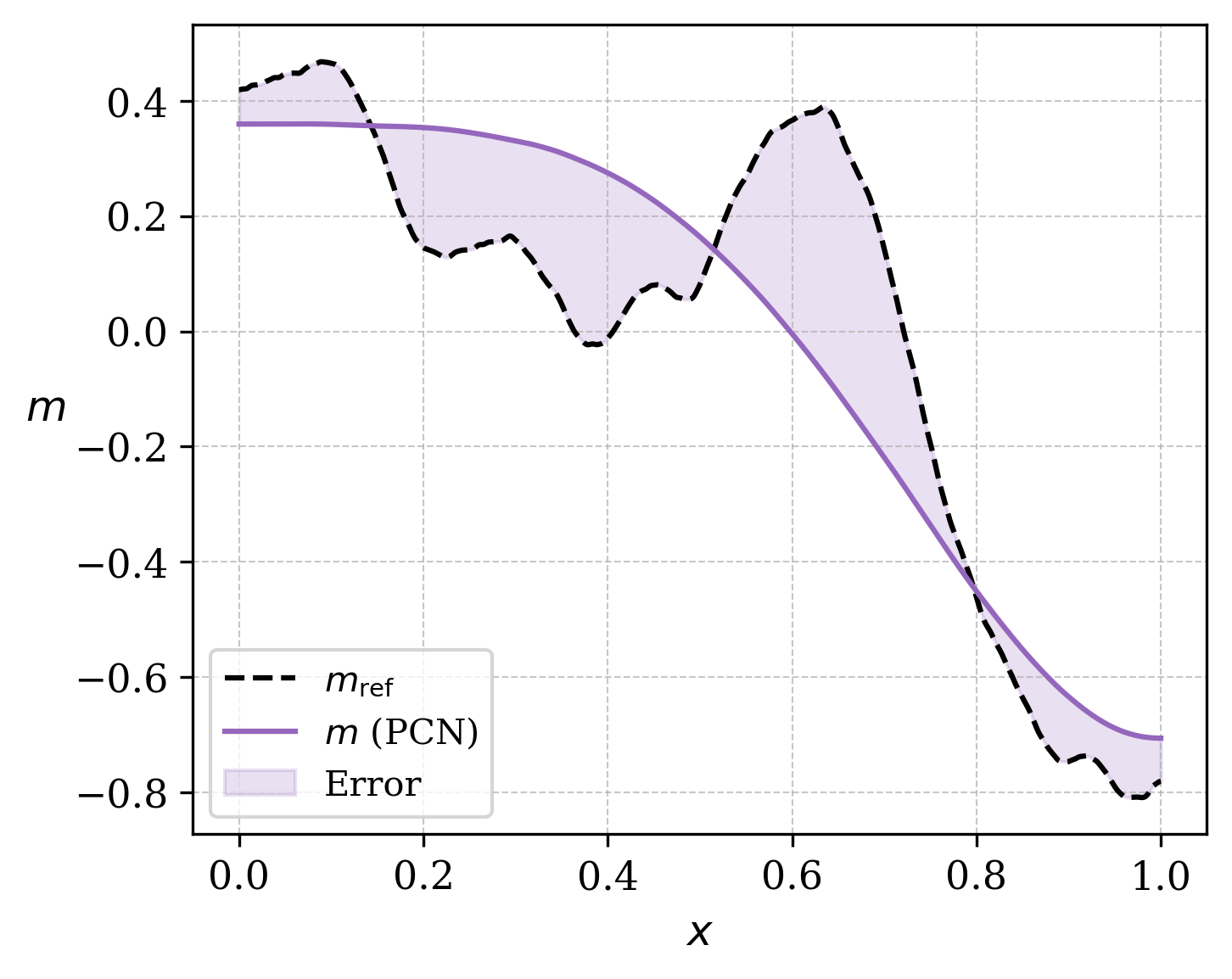}%
\hspace{0.5mm}%
\includegraphics[width=0.196\textwidth]{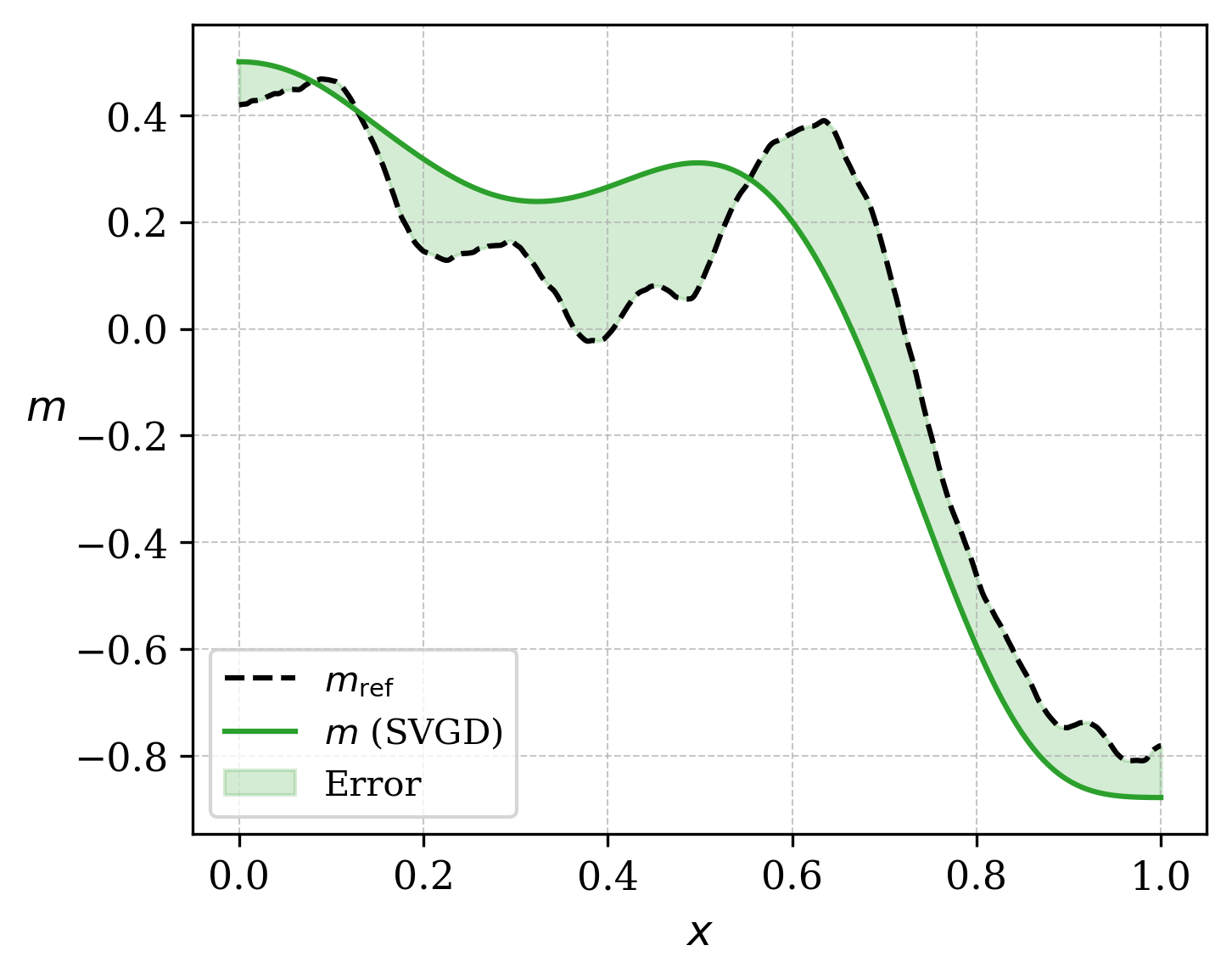}%
\hspace{0.5mm}%
\includegraphics[width=0.196\textwidth]{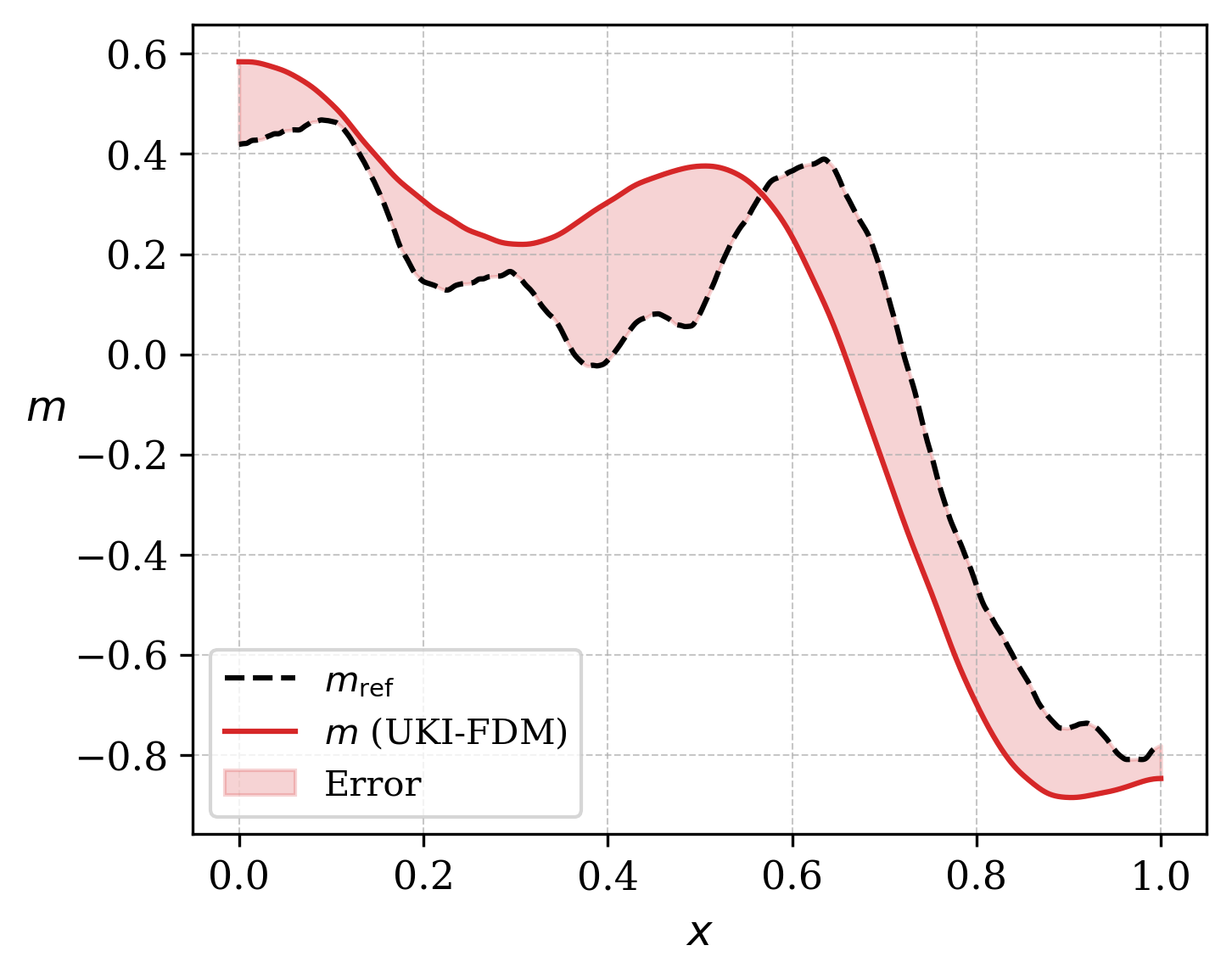}%
\hspace{0.5mm}%
\includegraphics[width=0.196\textwidth]{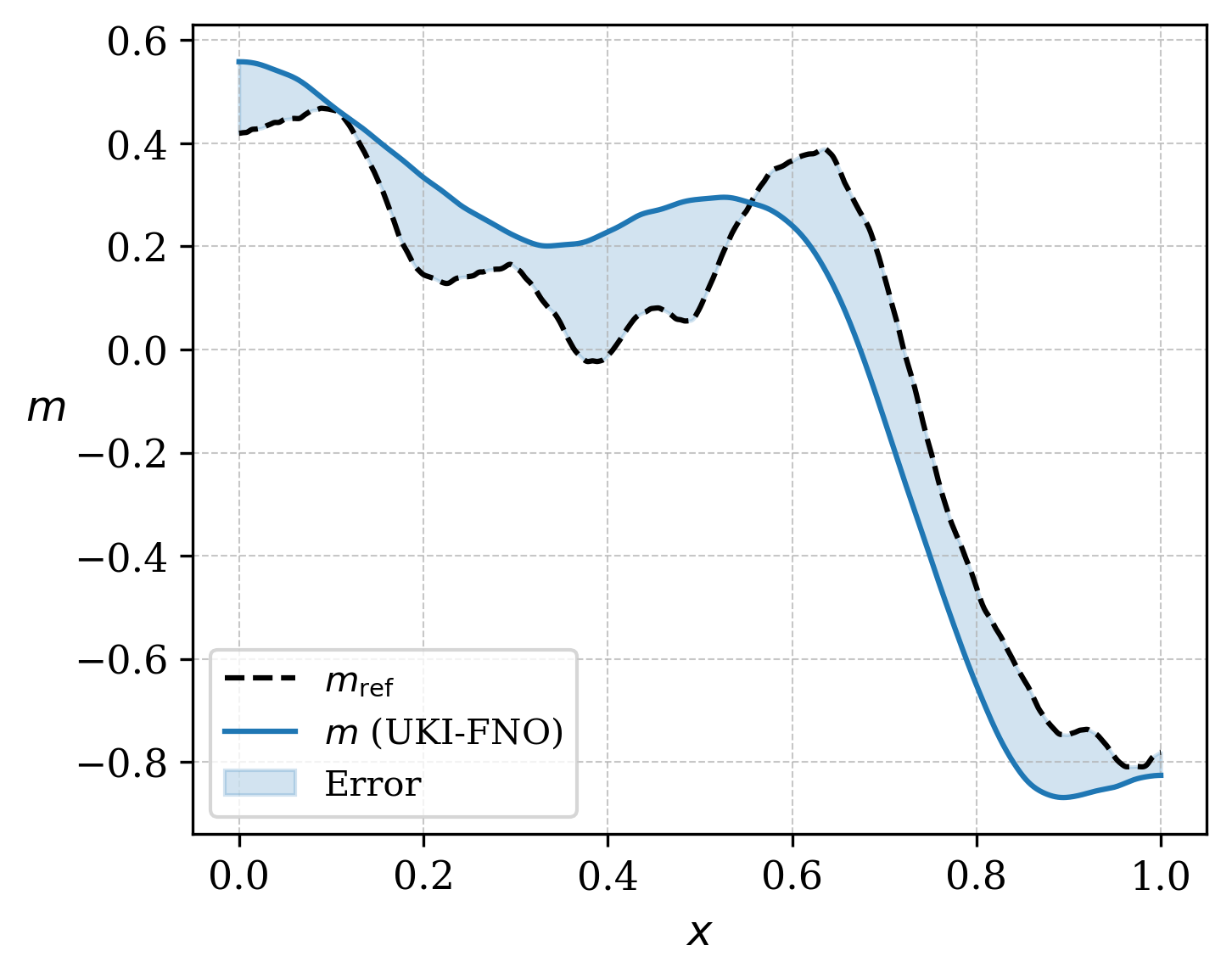}%
\hspace{0.5mm}%
\includegraphics[width=0.196\textwidth]{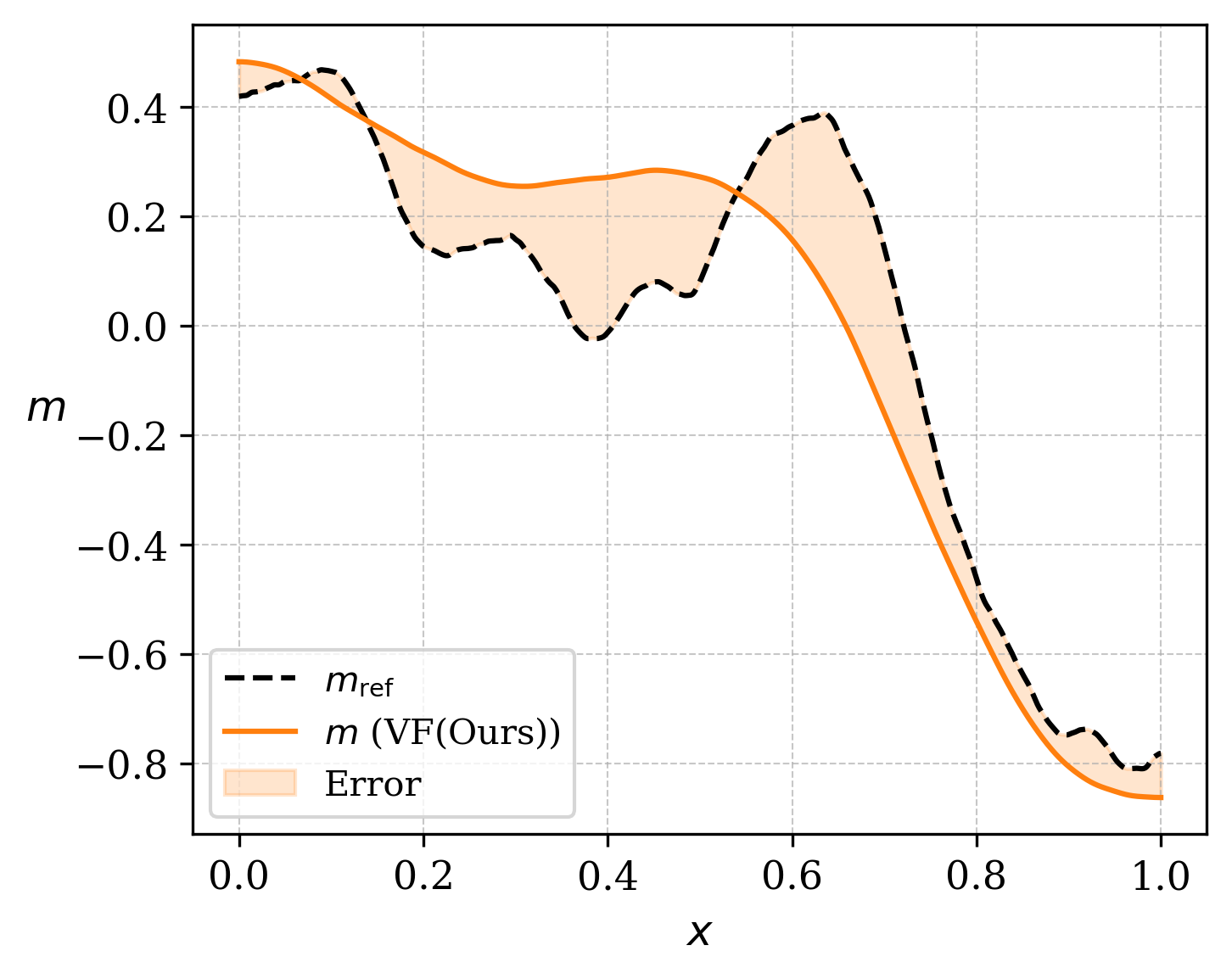}\\
\vspace{0.5mm}

\includegraphics[width=0.196\textwidth]{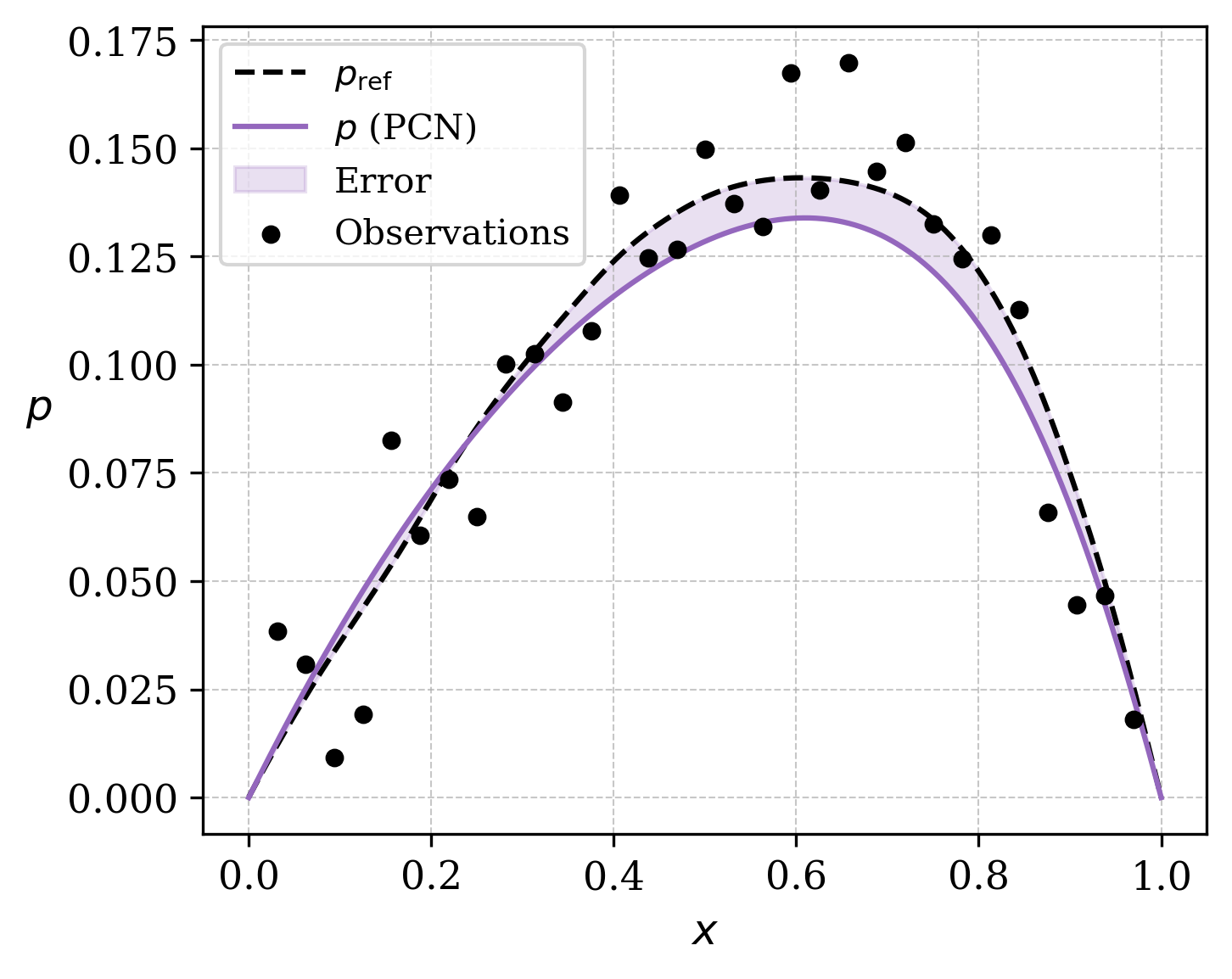}%
\hspace{0.5mm}%
\includegraphics[width=0.196\textwidth]{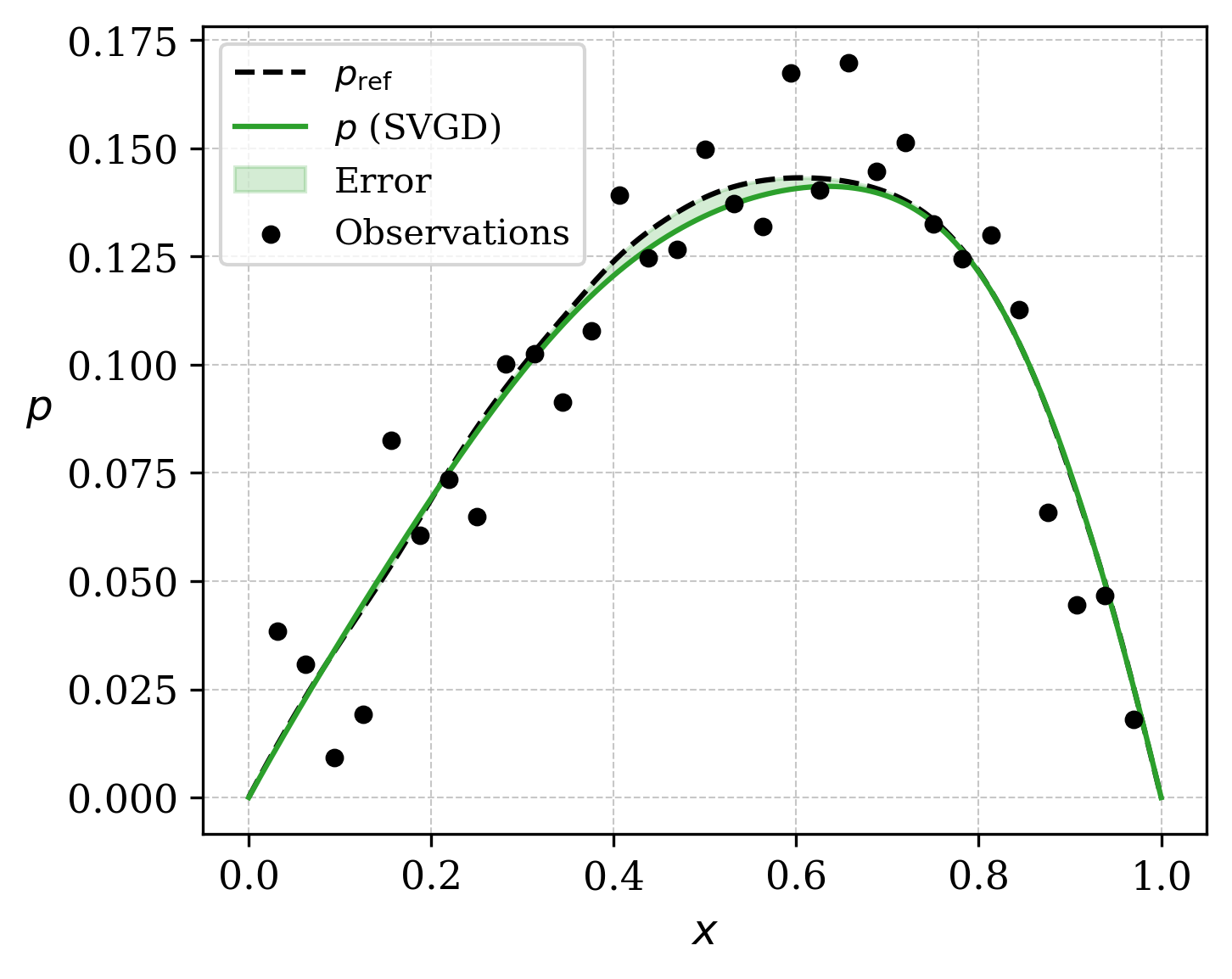}%
\hspace{0.5mm}%
\includegraphics[width=0.196\textwidth]{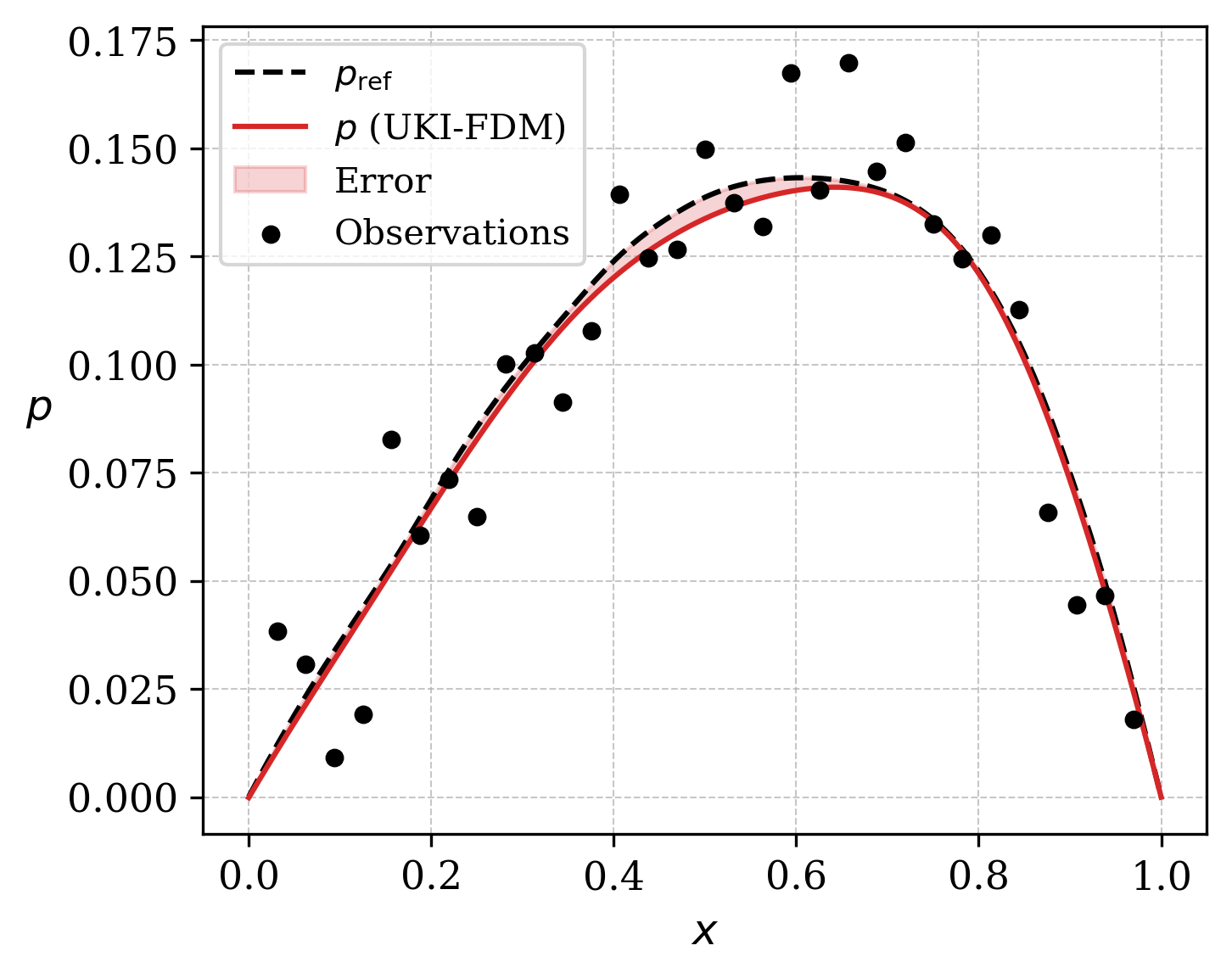}%
\hspace{0.5mm}%
\includegraphics[width=0.196\textwidth]{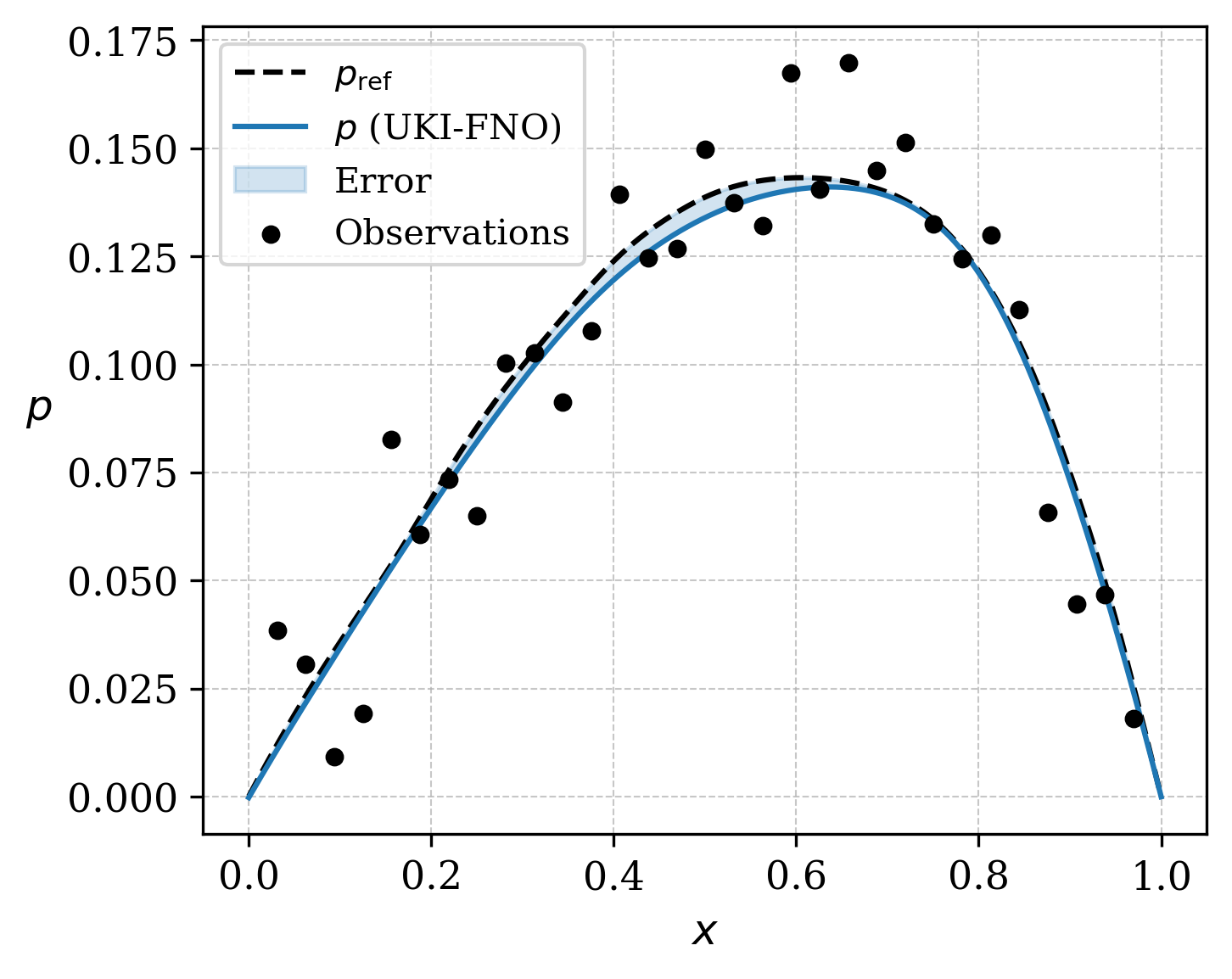}%
\hspace{0.5mm}%
\includegraphics[width=0.196\textwidth]{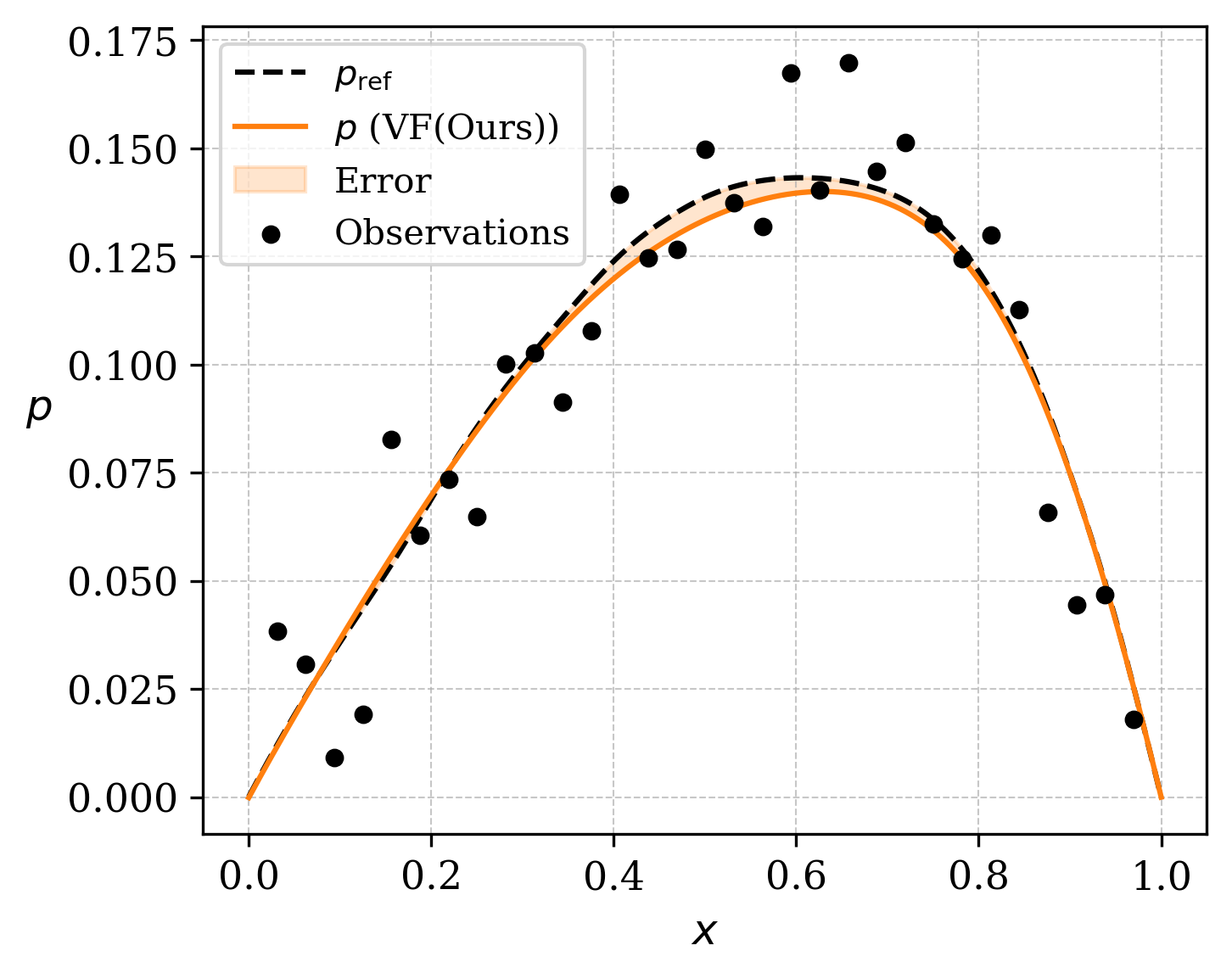}

\vspace{-1mm}
\centerline{\small (c) Noise level $\delta=10\%$}

\caption{Qualitative comparison for the 1D Darcy Flow problem with $d=64$ under different noise levels.
For each noise level, the first row shows the reconstructed coefficient fields $m$ and the second row shows the corresponding pressure fields $p$.
From left to right: pCN, SVGD, UKI-FDM, UKI-FNO, and our proposed VF method.
}
\label{fig:darcy1d_noise010_qual}
\vspace{-1em}
\end{figure}

Table~\ref{tab:comparison_1d_darcy} reports the relative inversion error $e_{\mathcal{I}}$ under varying noise amplitudes ($\delta \in \{1\%, 5\%, 10\%\}$) and truncation dimensions ($d \in \{32, 64\}$). The proposed framework consistently achieves the lowest inversion error across all scenarios. Notably, the improvement is especially pronounced at moderate-to-high noise levels, such as the 5\% and 10\% cases, demonstrating the robustness of our dual-flow architecture in highly uncertain regimes where standard Gaussian approximations (as in UKI) or particle-based methods (SVGD) struggle.

Figure~\ref{fig:fitting_error_1d_darcy} tracks the surrogate fitting error $e_{\mathcal{S}}$ over the adaptive training stages. Our approach converges to a lower fitting error in fewer stages compared to the baselines, highlighting the efficiency of the adaptive fine-tuning process guided by the VF posterior samples. 

Figure~\ref{fig:darcy1d_noise010_qual} illustrates the qualitative reconstructions for the 1D Darcy Flow problem with $d=64$ under different noise levels. As the noise level increases, the posterior distribution becomes increasingly challenging to characterize. The proposed VF method effectively captures the dominant spatial variations in the log-permeability field, producing pressure fields that closely follow the reference solutions. In comparison, the baseline methods, including pCN, SVGD, and UKI, show larger deviations from the reference, and overall the VF reconstructions achieve higher accuracy than all baselines, particularly under higher noise amplitudes.

\subsection{2D Darcy Flow}
\label{subsec:2d_darcy}

We extend the steady-state Darcy Flow to the 2D domain $\Omega = [0,1]^2$:
\begin{align}
    -\nabla \cdot \left( \exp(m_{\mb{\xi}}(\mb{x})) \nabla p(\mb{x}) \right) &= f(\mb{x}), \quad \mb{x} \in \Omega, \\
    p(\mb{x}) &= 0, \quad \mb{x} \in \partial \Omega,
\end{align}
where the source term is piecewise constant to induce strong heterogeneity:
\begin{equation}
    f(x_1, x_2) =
    \begin{cases}
        1000, & 0 \le x_2 \le \frac{4}{6}, \\
        2000, & \frac{4}{6} < x_2 \le \frac{5}{6}, \\
        3000, & \frac{5}{6} < x_2 \le 1.
    \end{cases}
\end{equation}

The log-permeability field $m_{\mb{\xi}}(\mb{x})$ is modeled as a Gaussian random field with covariance operator $\mathcal{C} = \sigma^2\,(-\Delta + \tau^2)^{-l}$, where $\Delta$ is the 2D Laplacian with homogeneous Neumann boundary conditions. We set $\sigma = 1$, $\tau = 2$, and $l = 3$. 

The truncated Karhunen-Loève expansion of $m_{\mb{\xi}}(\mb{x})$ is
\begin{equation}
    m_{\mb{\xi}}(\mb{x}) = \sum_{k} \sqrt{\lambda_k}\, \xi_k\, \psi_k(\mb{x}),
\end{equation}
with eigenfunctions and eigenvalues given explicitly by
\begin{equation}
    \psi_k(\mb{x}) =
    \begin{cases}
        \sqrt{2} \cos(\pi k_1 x_1), & k_2 = 0, \\
        \sqrt{2} \cos(\pi k_2 x_2), & k_1 = 0, \\
        2 \cos(\pi k_1 x_1) \cos(\pi k_2 x_2), & \text{otherwise},
    \end{cases}
    \quad
    \lambda_k = \sigma^2 \bigl( \pi^2 (k_1^2 + k_2^2) + \tau^2 \bigr)^{-l}.
\end{equation}

The vector $\boldsymbol{\xi} \in \mathbb{R}^d$ forms the finite-dimensional parameter space for inference. Here, we infer the log-permeability field from 36 uniformly spaced pressure observations. The forward PDE is solved numerically using a $71\times 71$ spatial grid.

\begin{table}[t]
\centering
\small
\renewcommand{\arraystretch}{1.12}
\setlength{\tabcolsep}{7pt}
\caption{Relative inversion error $e_{\mathcal{I}}$ of 2D Darcy Flow problem under different noise amplitudes and $d$ values. Lower is better. Results are averaged over 3 runs per experiment.}
\label{tab:comparison_2d_darcy}

\resizebox{0.9\textwidth}{!}{%
\begin{tabular}{lccc@{\hspace{1.2em}}ccc}
\toprule
\multirow{2}{*}{\textbf{Method}}
& \multicolumn{3}{c}{$d=32$}
& \multicolumn{3}{c}{$d=64$} \\
\cmidrule(lr){2-4} \cmidrule(lr){5-7}
& \textbf{1\%} & \textbf{5\%} & \textbf{10\%}
& \textbf{1\%} & \textbf{5\%} & \textbf{10\%} \\
\midrule
pCN        & 0.2623 & 0.4075 & 0.6044 & 0.2650 & 0.4206 & 0.6116 \\
SVGD-FNO   & 0.2602 & 0.4115 & 0.6100 & 0.2620 & 0.4115 & 0.6110 \\
UKI-FDM    & \textbf{0.2119} & 0.3891 & 0.4161 & 0.2169 & 0.3757 & 0.4067 \\
UKI-FNO    & 0.2153 & 0.3851 & 0.4225 & \textbf{0.2151} & 0.3869 & 0.4332 \\
\textbf{Ours} & 0.2238 & \textbf{0.3331} & \textbf{0.3578} & 0.2206 & \textbf{0.3372} & \textbf{0.3504} \\
\bottomrule
\end{tabular}%
}\vspace{-1em}
\end{table}

\begin{figure}[!htb]
\centering
\setlength{\abovecaptionskip}{4pt}
\setlength{\belowcaptionskip}{-2pt}
\includegraphics[width=0.32\textwidth]{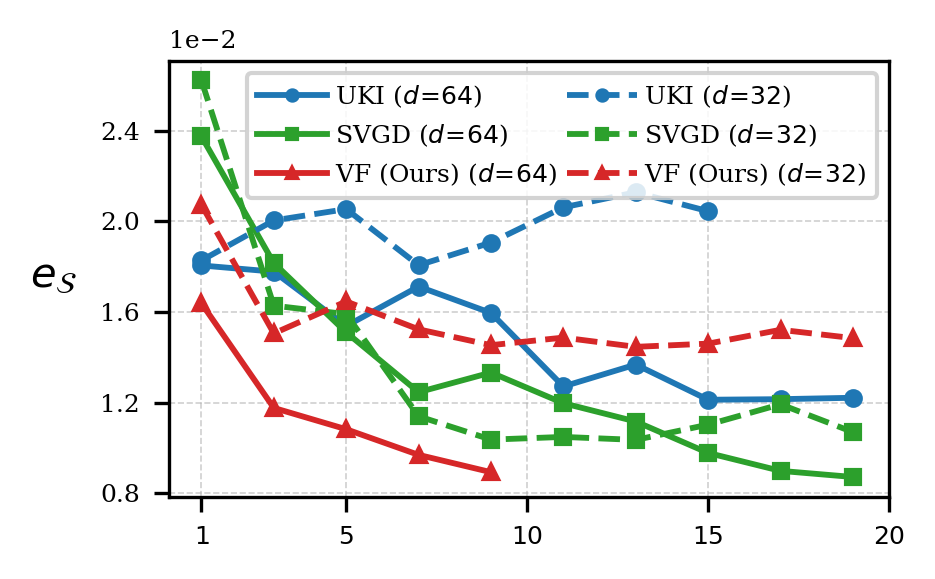}
\includegraphics[width=0.32\textwidth]{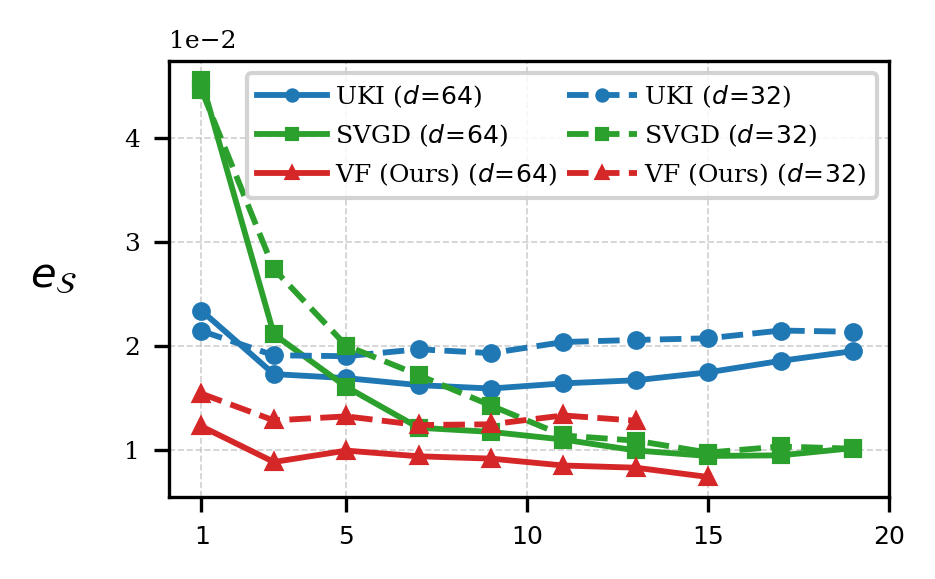}
\includegraphics[width=0.32\textwidth]{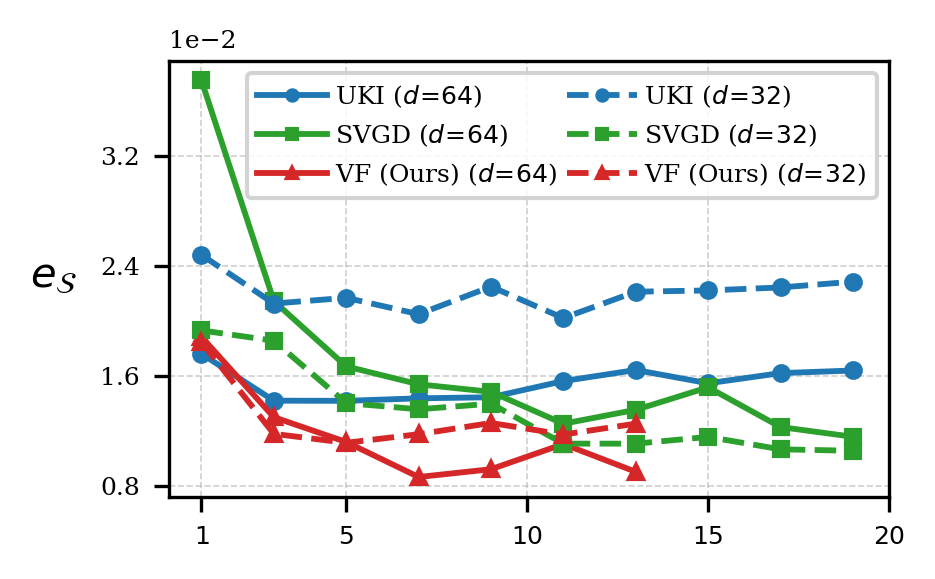}

\caption{Surrogate fitting error $e_{\mathcal{S}}$ across adaptive stages of 2D Darcy Flow problem. Columns: noise levels $\delta \in \{1\%, 5\%, 10\%\}$. Solid/dashed lines: $d=64/32$. Overall, our VF method converges to lower errors in fewer stages.}
\label{fig:fitting_error_2d_darcy}
\end{figure}

\begin{figure}[htbp]
\centering

\includegraphics[width=0.95\textwidth]{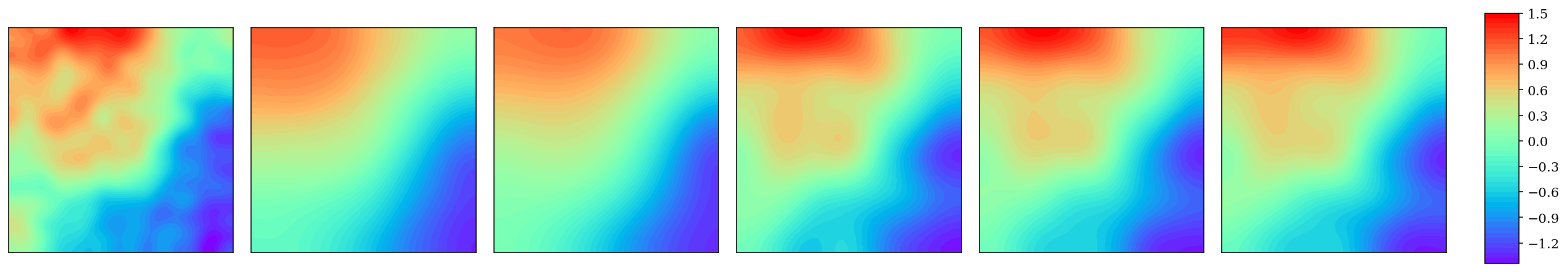}\\
\includegraphics[width=0.95\textwidth]{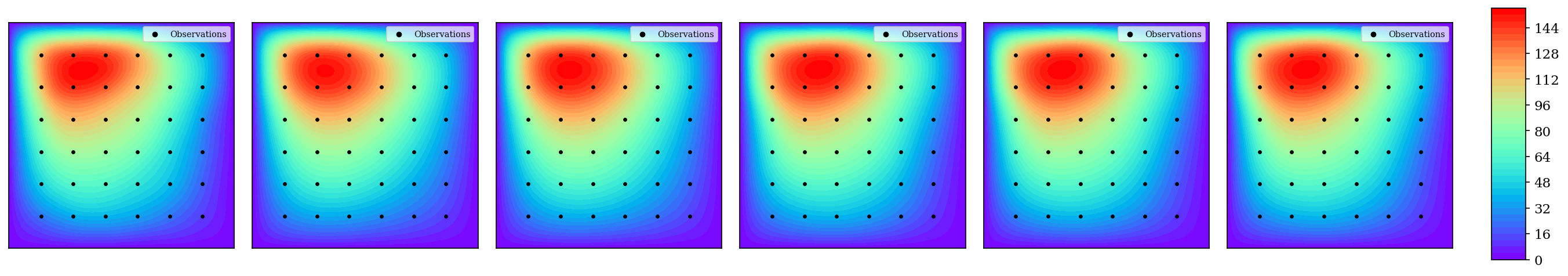}\\

\vspace{-1mm}
\centerline{\small (a) Noise level $\delta=1\%$}
\vspace{1mm}

\includegraphics[width=0.95\textwidth]{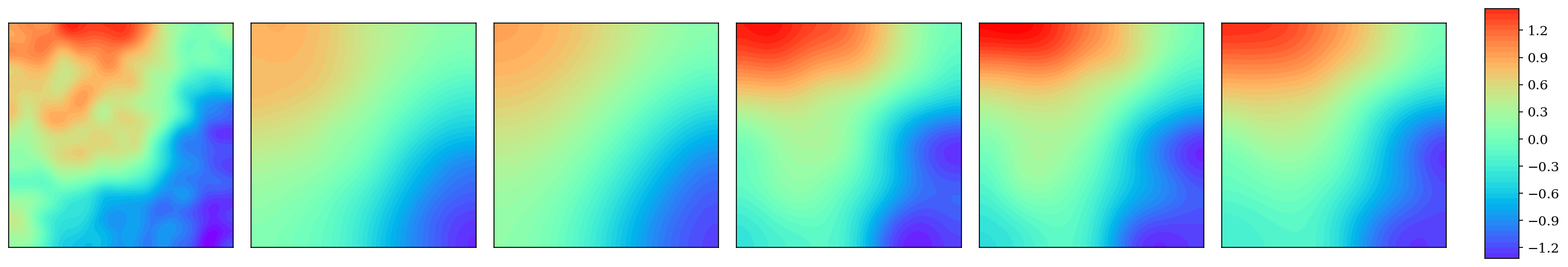}\\
\includegraphics[width=0.95\textwidth]{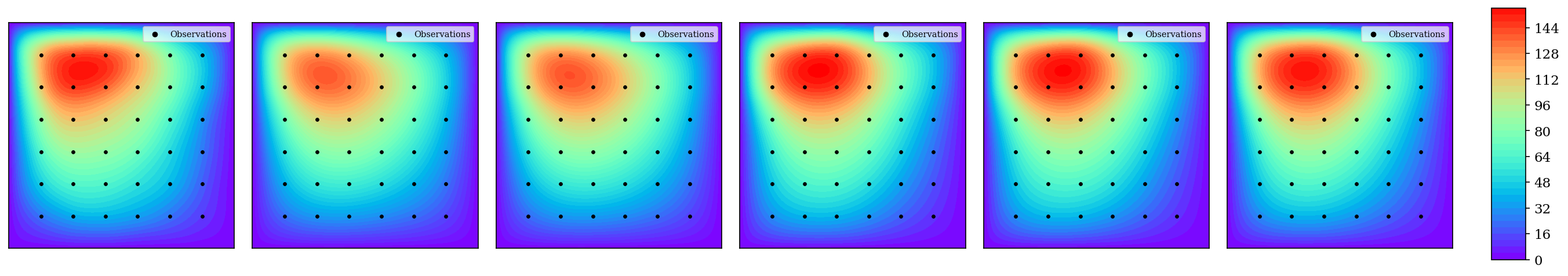}\\

\vspace{-1mm}
\centerline{\small (b) Noise level $\delta=5\%$}
\vspace{1mm}

\includegraphics[width=0.95\textwidth]{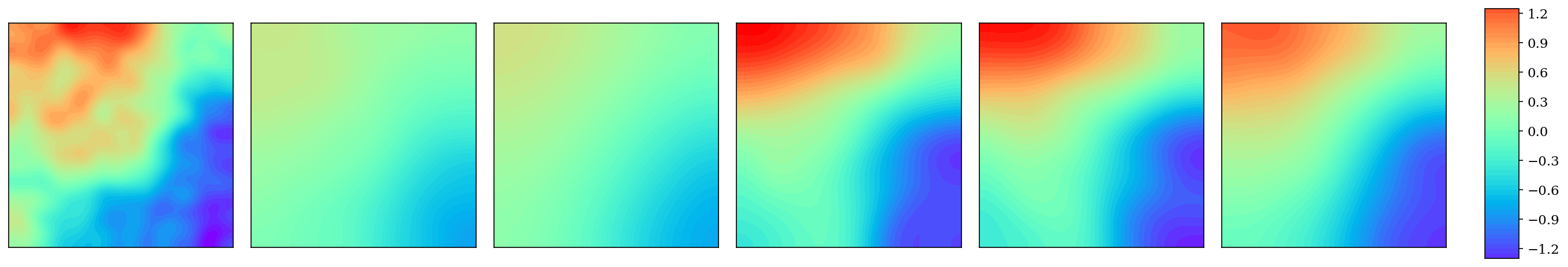}\\
\includegraphics[width=0.95\textwidth]{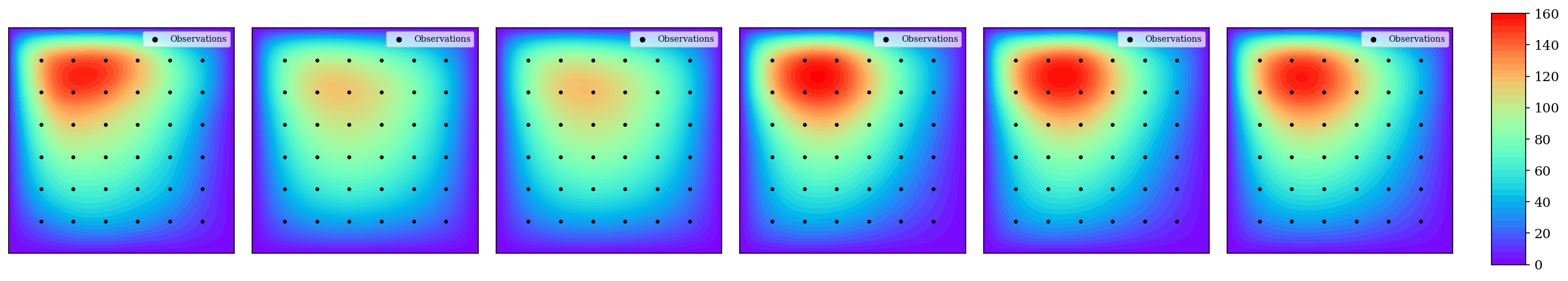}\\

\vspace{-1mm}
\centerline{\small (c) Noise level $\delta=10\%$}

\caption{
Qualitative comparison for the 2D Darcy Flow problem with $d=64$ under different noise levels.
For each noise level, the first row shows the reconstructed coefficient fields $m$ and the second row shows the corresponding pressure fields $p$.
From left to right: reference ground truth, pCN, SVGD, UKI-FDM, UKI-FNO, and our proposed VF method.
}
\label{fig:darcy2d_qual}
\vspace{-1em}
\end{figure}

The quantitative inversion results for the 2D Darcy problem are summarized in Table~\ref{tab:comparison_2d_darcy}. In this highly heterogeneous setting, The proposed method achieves the best performance in most medium- and high-noise cases, while remaining competitive in the low-noise regime. The iterative prior updating scheme effectively mitigates prior misspecification, allowing the model to locate the true posterior even when the reference field is drawn from an out-of-distribution uniform distribution.

The evolution of the surrogate fitting error $e_{\mathcal{S}}$, depicted in Figure~\ref{fig:fitting_error_2d_darcy}, further confirms the stability of our method. The VF-guided adaptive training rapidly reduces the surrogate error, outperforming both UKI-FDM and UKI-FNO in terms of convergence speed and final accuracy.

Figure~\ref{fig:darcy2d_qual} presents qualitative comparisons for the 2D Darcy Flow problem with $d=64$ under different noise levels. Compared with the 1D case, the 2D coefficient field exhibits substantially more complex spatial structures and heterogeneous patterns. The proposed VF method consistently reconstructs the high-contrast regions of the permeability field while preserving high-frequency spatial features and avoiding the spurious artifacts observed in the pCN and SVGD results. In particular, the reconstructions produced by pCN and SVGD exhibit significant discrepancies in scale compared with the reference field, as these methods lack the iterative prior updating mechanism and therefore tend to generate parameter fields concentrated around the initial prior distribution $\mathcal{N}(\mb{0},\mb{I})$. In contrast, UKI-based methods and our proposed framework are able to better capture the overall magnitude and spatial range of the reference coefficient field. However, UKI tends to overestimate the extreme values and appears to exhibit more fine-scale details than our framework. In reality, these apparent details correspond to errors introduced by the Gaussian-based assumptions in UKI, resulting in lower reconstruction accuracy compared with our VF method. Consequently, the pressure fields produced by VF remain highly consistent with the reference solutions, whereas the baseline methods -- including UKI, pCN, and SVGD -- either produce oversmoothed reconstructions or distorted local structures.

\subsection{2D Navier-Stokes}
\label{subsec:2d_ns}

We consider the vorticity formulation of the 2D NS equations on $\Omega = [0,1]^2$ with periodic boundary conditions:
\begin{align}
    \partial_t \omega + \mb{u} \cdot \nabla \omega &= \nu \Delta \omega + f(\mb{x}), && \mb{x} \in \Omega,\; t \in (0,T], \\
    \nabla \cdot \mb{u} &= 0, && \mb{x} \in \Omega,\; t \in [0,T], \\
    \omega(\mb{x}, 0) &= m_{\mb{\xi}}(\mb{x}), && \mb{x} \in \Omega,
\end{align}
where the viscosity is $\nu = 10^{-2}$, final time $T = 1$, and the forcing term is
\begin{equation}
    f(\mb{x}) = 0.1\bigl(\sin(2\pi(x_1+x_2)) + \cos(2\pi(x_1+x_2))\bigr).
\end{equation}

The initial vorticity field $m_{\mb{\xi}}(\mb{x})$ is modeled as a Gaussian random field with covariance operator $\mathcal{C} = \sigma^2 (-\Delta + \tau^2)^{-l}$,
where $\Delta$ is the 2D Laplacian with periodic boundary conditions. We set the parameters as $\sigma = 25$, $\tau = 2$, and $l = 2.5$.
The truncated Karhunen-Loève expansion of $m_{\mb{\xi}}(\mb{x})$ is
\begin{equation}
    m_{\mb{\xi}}(\mb{x}) = \sum_{k \in K} \sqrt{\lambda_k}\, \xi^c_{(k)}\, \psi^c_k(\mb{x}) + \sqrt{\lambda_k}\, \xi^s_{(k)}\, \psi^s_k(\mb{x}),
\end{equation}
where $K = \{ (k_x, k_y) \mid k_y > 0 \text{ or } (k_y = 0 \text{ and } k_x > 0)\}$, and the basis functions and eigenvalues are
\begin{equation}
\begin{aligned}
    \psi^c_k(\mb{x}) = \sqrt{2}\, &\cos(2\pi \mb{k} \cdot \mb{x}), \quad
    \psi^s_k(\mb{x}) = \sqrt{2}\, \sin(2\pi \mb{k} \cdot \mb{x}), \\
    & \lambda_k = \sigma^2 \bigl( 4 \pi^2 (k_x^2 + k_y^2) + \tau^2 \bigr)^{-l}.
\end{aligned}
\end{equation}
The vector of coefficients $\boldsymbol{\xi} \in \mathbb{R}^d$ forms the finite-dimensional parameter space for inference. The goal is to recover the initial vorticity field $\omega_0$ from 36 uniformly spaced sparse observations of the final state $\omega(\mb{x}, 1)$ at $T=1$. The forward solver operates on a $128\times 128$ spatial grid.

\begin{table}[t]
\centering
\small
\renewcommand{\arraystretch}{1.12}
\setlength{\tabcolsep}{7pt}
\caption{Relative inversion error $e_{\mathcal{I}}$ of 2D Navier-Stokes problem under different noise amplitudes and $d$ values. Lower is better. Results are averaged over 3 runs per experiment.}
\label{tab:comparison_2d_ns}

\resizebox{0.9\textwidth}{!}{%
\begin{tabular}{lccc@{\hspace{1.2em}}ccc}
\toprule
\multirow{2}{*}{\textbf{Method}}
& \multicolumn{3}{c}{$d=32$}
& \multicolumn{3}{c}{$d=64$} \\
\cmidrule(lr){2-4} \cmidrule(lr){5-7}
& \textbf{1\%} & \textbf{5\%} & \textbf{10\%}
& \textbf{1\%} & \textbf{5\%} & \textbf{10\%} \\
\midrule
pCN        & 0.4910 & 0.4560 & 0.5387 & 0.4686 & 0.4634 & 0.5514 \\
SVGD-FNO   & 0.3720 & 0.4591 & 0.5413 & 0.3723 & 0.4572 & 0.5407 \\
UKI-FDM    & 0.2948 & 0.3777 & 0.4571 & 0.3082 & 0.3756 & 0.4598 \\
UKI-FNO    & 0.2864 & 0.3778 & 0.4586 & 0.2951 & 0.3784 & 0.4618 \\
\textbf{Ours} & \textbf{0.2849} & \textbf{0.3647} & \textbf{0.4207} & \textbf{0.2879} & \textbf{0.3644} & \textbf{0.4184} \\
\bottomrule
\end{tabular}%
}\vspace{-1em}
\end{table}

\begin{figure}[ht]
\centering
\setlength{\abovecaptionskip}{4pt}
\setlength{\belowcaptionskip}{-2pt}

\includegraphics[width=0.32\textwidth]{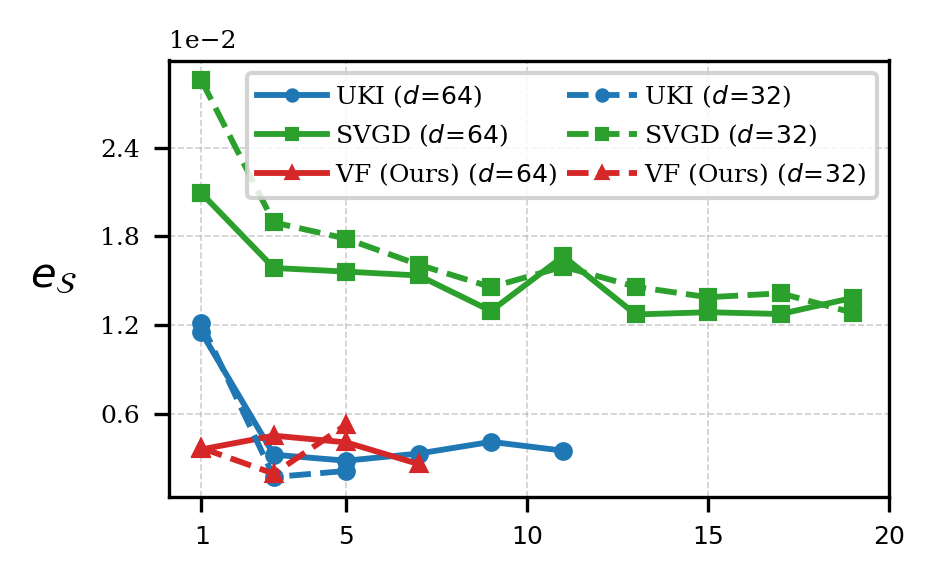}
\includegraphics[width=0.32\textwidth]{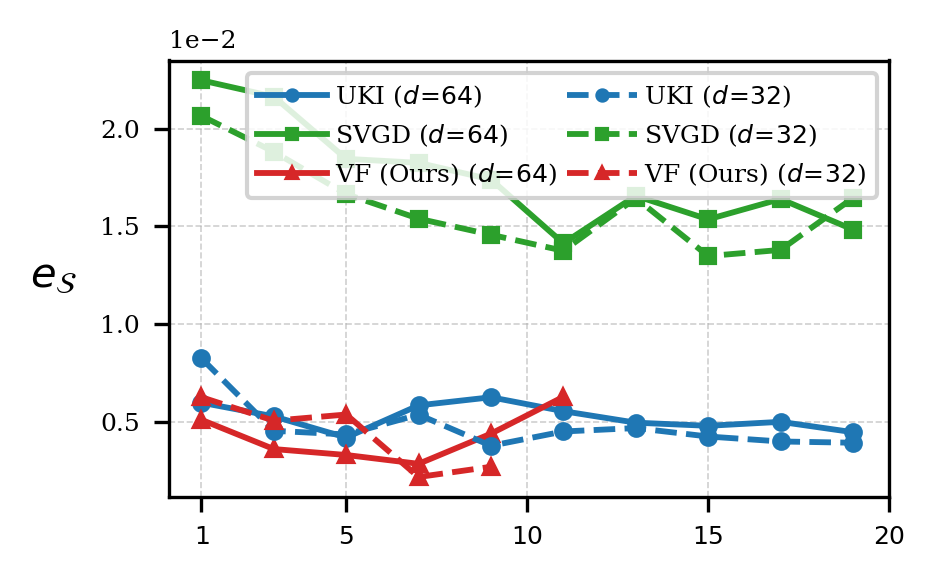}
\includegraphics[width=0.32\textwidth]{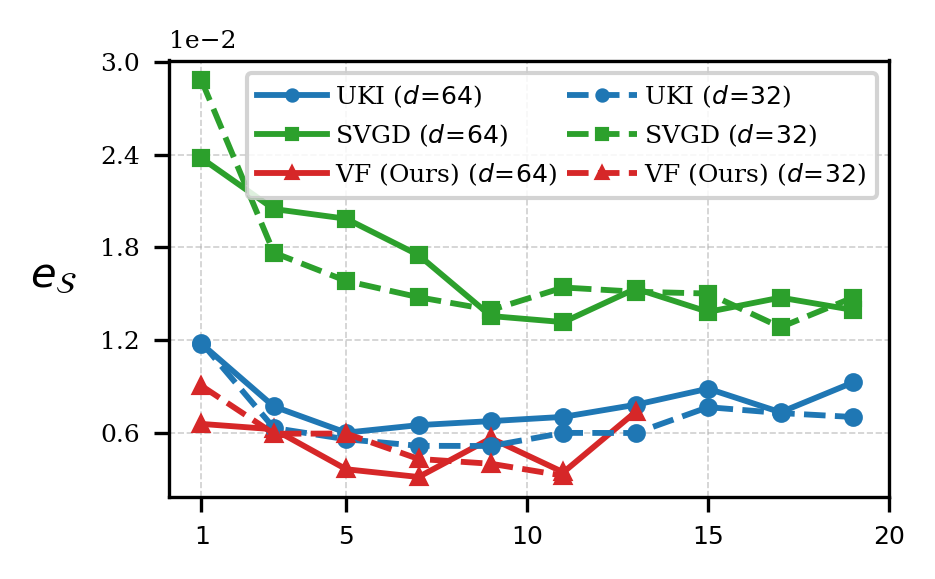}

\caption{Surrogate fitting error $e_{\mathcal{S}}$ across adaptive stages of 2D Navier-Stokes problem. Columns: noise levels $\delta \in \{1\%, 5\%, 10\%\}$. Solid/dashed lines: $d=64/32$. Overall, our VF method converges to lower errors in fewer stages.}
\label{fig:fitting_error_2d_ns}
\end{figure}

\begin{figure}[ht]
\centering

\includegraphics[width=0.95\textwidth]{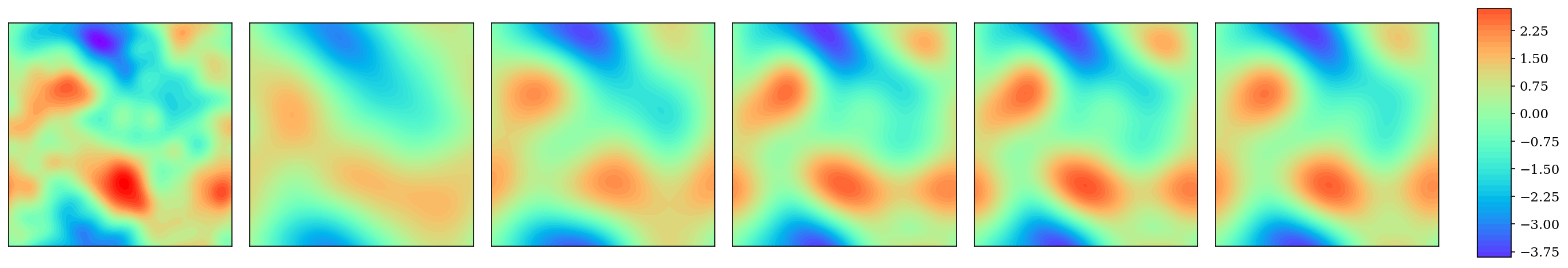}\\
\includegraphics[width=0.95\textwidth]{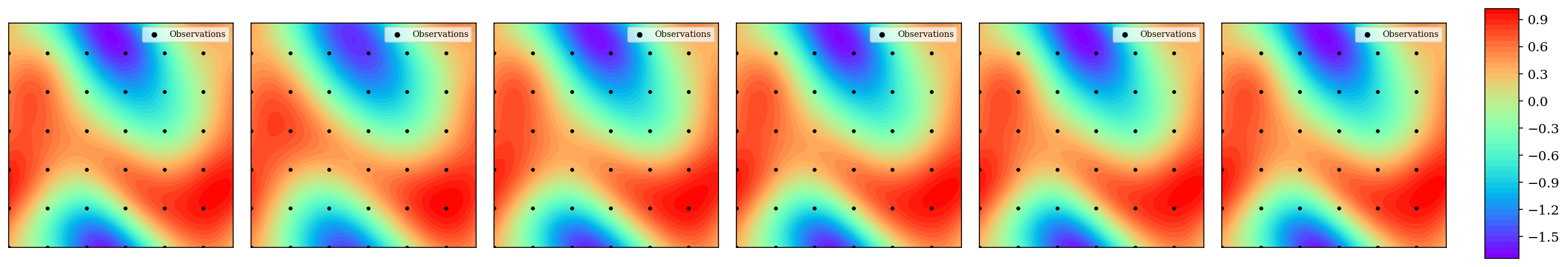}\\

\vspace{-1mm}
\centerline{\small (a) Noise level $\delta=1\%$}
\vspace{1mm}

\includegraphics[width=0.95\textwidth]{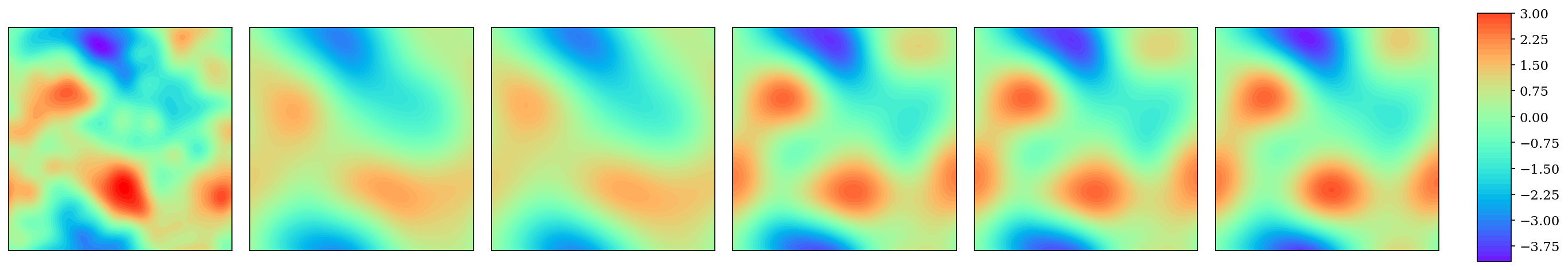}\\
\includegraphics[width=0.95\textwidth]{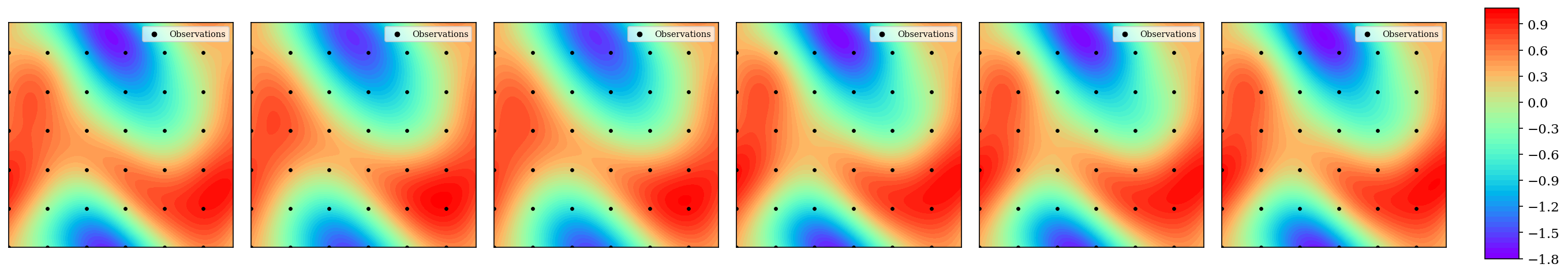}\\

\vspace{-1mm}
\centerline{\small (b) Noise level $\delta=5\%$}
\vspace{1mm}

\includegraphics[width=0.95\textwidth]{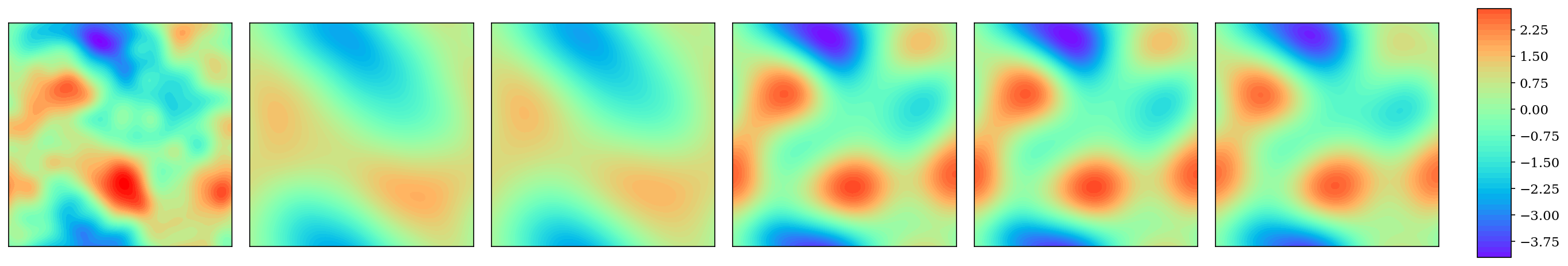}\\
\includegraphics[width=0.95\textwidth]{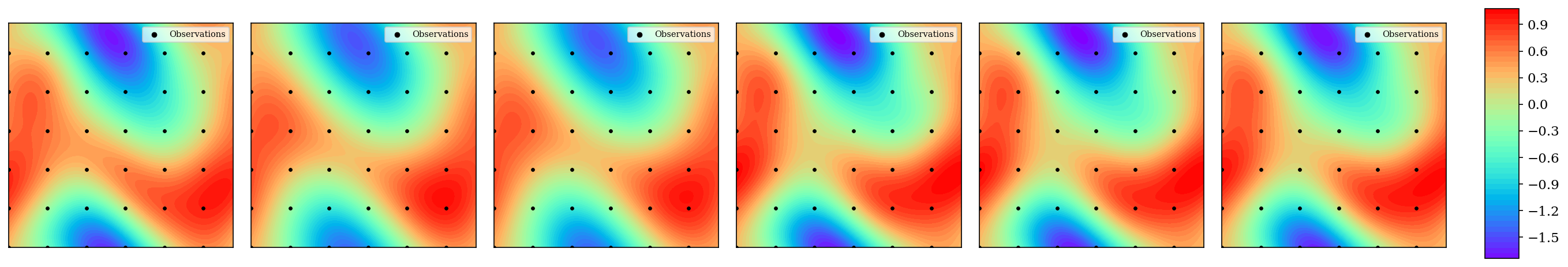}\\

\vspace{-1mm}
\centerline{\small (c) Noise level $\delta=10\%$}

\caption{
Qualitative comparison for the 2D Navier-Stokes problem with $d=64$ under different noise levels.
For each noise level, the first row shows the reconstructed initial vorticity fields $\omega_0$ and the second row shows the corresponding vorticity fields $\omega(\cdot,T)$ at $T=1$.
From left to right: reference ground truth, pCN, SVGD, UKI-FDM, UKI-FNO, and our proposed VF method.
}
\label{fig:ns2d_qual}
\vspace{-1em}
\end{figure}

Table~\ref{tab:comparison_2d_ns} details the relative inversion errors for the 2D Navier-Stokes equations. This problem is particularly challenging due to the highly nonlinear time-dependent dynamics over time. Despite these difficulties, the VF framework achieves the lowest $e_{\mathcal{I}}$ across all tested truncation dimensions and noise levels. 

Figure~\ref{fig:fitting_error_2d_ns} illustrates the surrogate fitting error $e_{\mathcal{S}}$. The FNO surrogate, when fine-tuned using the posterior-concentrated samples generated by the VF model, adapts rapidly to the complex fluid dynamics, achieving a lower error plateau than the UKI-driven adaptive surrogates.

Visual comparisons of the reconstructed initial vorticity fields and the corresponding final states under different noise levels are shown in Figure~\ref{fig:ns2d_qual}. Similar to the 2D Darcy case, the VF method accurately captures the complex vortex structures in $\omega_0$ and preserves their spatial distributions, resulting in final vorticity fields $\omega(\cdot,T)$ at $T=1$ that closely match the reference solutions. In contrast, the baseline methods either oversmooth the vortices or misestimate their magnitudes, with errors amplified under higher noise levels due to the lack of adaptive prior updating and expressive posterior modeling.

\section{Conclusions}
\label{sec:conclusions}


We present a deep adaptive dimension-reduction Bayesian inference framework for high-dimensional PDE-governed inverse problems, addressing non-Gaussian posteriors, surrogate out-of-distribution (OOD) errors, and prior misspecification. Our core Variational Flow (VF) model integrates VAE-based dimensionality reduction with dual normalizing flows, achieving a strictly higher ELBO than VAE and capturing multimodal posteriors. Coupled with an iterative prior updating scheme and an adaptively fine-tuned FNO surrogate, the framework establishes a mutually reinforcing loop for accurate posterior approximation. Empirically, our method achieves state-of-the-art accuracy on standard PDE benchmarks, particularly excelling under high noise and large dimensions. Future work will address the current lack of theoretical convergence guarantees and extend to multi-physics systems.


\section*{Acknowledgments}
This work was supported by the National Natural Science Foundation of China (grants 12131002, 12501598, 12288201, and 12461160275), NSF Grant DMS-2513234, the high-level talent research start-up project funding of Henan Academy of Sciences (No.~232019024), and the Science Challenge Project (No.~TZ2025006).

\appendix
\section{Karhunen-Lo\`{e}ve Expansion with Covariance Operator}
\label{appendix:kl}
The Karhunen-Lo\`eve (KL) expansion provides a systematic way to represent a spatially dependent Gaussian random field (GRF) as a spectral series. After truncation, it yields a finite-dimensional parameterization of the random field, forming the basis for representing the coefficient field in our PDE-governed inverse problems.

\paragraph{Covariance operator}
Let $m(\mb{x})$ be a GRF with mean $\bar{m}(\mb{x})$ and covariance kernel $k(\mb{x}, \mb{x}')$. We define the corresponding covariance operator $\mathcal{C}: L^2(\Omega_s) \to L^2(\Omega_s)$ as
\begin{equation}
    (\mathcal{C}\,\phi)(\mb{x}) = \int_{\Omega_s} k(\mb{x}, \mb{x}')\,\phi(\mb{x}')\,\mathrm{d}\mb{x}'.
\end{equation}
In many PDE-governed inverse problems, we parameterize the covariance via an elliptic operator, e.g.
\begin{equation} 
\mathcal{C} = \sigma^2\,(-\Delta + \tau^2)^{-l}, 
\end{equation} 
where $\Delta$ is the Laplacian on $\Omega_s$ with appropriate boundary conditions, $\sigma > 0$ controls the overall variance (amplitude) of the Gaussian random field, $\tau > 0$ controls the inverse correlation length, and $l > 0$ controls the regularity of the field. This operator is symmetric, positive definite, and compact, ensuring the existence of a countable set of eigenpairs.

Solving the eigenvalue problem $\mathcal{C}\,\phi_j = \lambda_j \,\phi_j$ yields eigenvalues $\lambda_1 \ge \lambda_2 \ge \dots \ge 0$ and orthonormal eigenfunctions $\{\phi_j\}_{j=1}^\infty$ in $L^2(\Omega_s)$. The GRF can then be represented as an infinite series $m_{\boldsymbol{\xi}_\infty}(\mb{x}) = \bar{m}(\mb{x}) + \sum_{j=1}^{\infty} \sqrt{\lambda_j}\,\xi_j\,\phi_j(\mb{x})$, where $\xi_j \sim \mathcal{N}(0,1)$. For numerical tractability, we truncate the series to the first $d$ leading modes:
\begin{equation}
\label{eq:kl-trunc}
    m_{\mb{\xi}}(\mb{x}) = \bar{m}(\mb{x}) + \sum_{j=1}^{d} \sqrt{\lambda_j}\,\xi_j\,\phi_j(\mb{x}), \quad 
    \boldsymbol{\xi} = (\xi_1,\dots,\xi_d)^\top \in \mathbb{R}^d.
\end{equation}
This finite-dimensional parameterization defines the exact inference space $\Omega_p$ over which our Variational Flow (VF) posterior approximation is constructed.

\section{VF for Standard Generative Modeling}
\label{appendix:vf_case1}

Section~\ref{sec:vf_posterior} addresses the setting where the target density $\hat{p}(\mb{x})$ is known up to a normalization constant. Here we describe the complementary setting of standard generative modeling, where only i.i.d.\ samples $\{\mb{x}^{(n)}\}_{n=1}^{N} \sim p_{\mb{x}}$ are available but $p_{\mb{x}}$ cannot be evaluated.

In this case, we train VF by maximizing the ELBO of $\log p_{\mb{x}}$:
\begin{equation}
L_{\mb{\theta},\mb{\alpha},\mb{\beta}}(\mb{x})
= \mathbb{E}_{q_{\mb{z}|\mb{x},\mb{\alpha}}}\!\left[\log \frac{p_{\mb{x}|\mb{z},\mb{\theta}}\,p_{\mb{z},\mb{\beta}}}{q_{\mb{z}|\mb{x},\mb{\alpha}}}\right]
\approx \frac{1}{M}\sum_{i=1}^{M}\log \frac{p_{\mb{x}|\mb{z},\mb{\theta}}(\mb{x}|\mb{z}^{(i)})\,p_{\mb{z},\mb{\beta}}(\mb{z}^{(i)})}{q_{\mb{z}|\mb{x},\mb{\alpha}}(\mb{z}^{(i)}|\mb{x})},
\end{equation}
where $\{\mb{z}^{(i)}\}_{i=1}^{M}$ are drawn from $q_{\mb{z}|\mb{x},\mb{\alpha}}(\mb{z}|\mb{x})$. The total training objective is 
\begin{equation}
\frac{1}{N}\sum_{n=1}^{N} L_{\mb{\theta},\mb{\alpha},\mb{\beta}}(\mb{x}^{(n)}).
\end{equation}

\section{Pre-training of the Fourier Neural Operator Surrogate}
\label{appendix:fno}
We adopt the Fourier Neural Operator (FNO)~\cite{li2020fourier} as the surrogate. The FNO learns mappings between infinite-dimensional function spaces in the spectral domain, achieving discretization-invariant representations that generalize across different grids.

\paragraph{Dataset construction} Training samples $\{\mb{\xi}^{(i)}\}_{i=1}^N$ are drawn from the prior $\pi_0 = \mathcal{N}(\mb{\mu}_0, \mb{\Sigma}_0)$, and the corresponding coefficient fields are generated via the truncated KL expansion~\eqref{eq:kl-trunc} (see Appendix~\ref{appendix:kl} for details). For each sample, the PDE system~\eqref{spdexi1}-\eqref{spdexi2} is solved numerically to obtain the state field $u(\cdot; m_{\mb{\xi}^{(i)}})$, yielding the training dataset:
\begin{equation}
    \mathcal{D} = \bigl\{\, m_{\mb{\xi}^{(i)}}(\mb{x}),\; u\bigl(\mb{x}; m_{\mb{\xi}^{(i)}}\bigr) \,\bigr\}_{i=1}^N.
\end{equation}

\paragraph{Training objective} Let $\mb{\vartheta}$ denote the trainable parameters of the FNO, denoted $\mathcal{F}_{\mb{\vartheta}}: m_{\mb{\xi}} \mapsto u(\cdot; m_{\mb{\xi}})$. The model is trained by minimizing the relative $L^2$ loss:
\begin{equation}
\label{eq:fno_loss}
    \mathcal{L}_{\mathrm{FNO}}(\mb{\vartheta};\mathcal{D}) = \frac{1}{N} \sum_{i=1}^{N}
    \frac{\bigl\|\mathcal{F}_{\mb{\vartheta}}(m_{\mb{\xi}^{(i)}}) - u(\cdot; m_{\mb{\xi}^{(i)}})\bigr\|_{L^2(\Omega_s)}}
         {\bigl\|u(\cdot; m_{\mb{\xi}^{(i)}})\bigr\|_{L^2(\Omega_s)}}.
\end{equation}
Once trained, the composite map $\mb{\xi} \mapsto \mathcal{F}_{\mb{\vartheta}}(m_{\mb{\xi}})$ serves as the surrogate for the forward map throughout Bayesian inference. Specific hyperparameter settings (dataset size, learning rate schedule, training epochs) for each experiment are reported in Appendix~\ref{appendix:exp_details}.
\section{Implementation Details}
\label{appendix:exp_details}
\textbf{Computational Resources.} All experiments were conducted on a single workstation equipped with an NVIDIA RTX 3090 GPU (24GB VRAM). The total computational time for a single run of each PDE-governed inverse problem (1D Darcy flow, 2D Darcy flow, and 2D Navier-Stokes) is approximately 20 minutes. The Rosenbrock inverse problem takes approximately 2 hours to complete.

\subsection{PDE-governed Inverse Problems}
The stopping threshold of the overall framework is set to a tolerance of $\epsilon = 0.01$. For training the FNO as a surrogate model, the initial dataset size is set to 2,000,
and the model is trained for 1,000 epochs with an initial learning rate of 0.001, using a decay rate of 0.5 applied every 50 epochs. During the adaptive fine-tuning of the FNO, 500 samples are drawn at each stage with a perturbation strength of $\gamma = 3$ (as defined in Eq.~\eqref{eq:perturbation}), and the model is then
fine‑tuned for 100 epochs using a batch size of 25. The learning rate is initialized at 0.001 and decreased by a factor of 0.5 every 25 epochs, using the Adam optimizer~\cite{kingma2014adam}.

For Variational Flow (VF) model, the data dimension is set to $d$ and the latent dimension is 16 throughout all experiments. The flow-based prior employs 6 affine coupling layers, while the conditional normalizing flow encoder uses 2 affine coupling
layers. Each flow subnetwork has a hidden width of 32, and the multiscale squeezing step is fixed at 4. The accompanying VAE-style decoder contains 5 hidden layers of width 64. In terms of VF training, we run at most 20 adaptive stages. At each stage the VF model is trained on 1024 samples with a batch size of 32 and a learning rate of 0.001 for 10 epochs.

\subsection{Rosenbrock Problem}\label{sec:rosen_details}
In this problem, we compare our proposed model with a vanilla VAE, MCMC, SVGD, and UKI. Our method and UKI do not employ iterative priors in this problem. The MCMC method is implemented using the emcee Python package with 1,000 chains, a burn-in period of 50,000 steps, and 10,000 sampling steps. The SVGD method utilizes 20,000 particles and is run for 10,000 iterations. Our model and VAE model are trained with a learning rate of 0.001 on a dataset of 100,000 samples, using a batch size of 10,000, for 20,000 epochs. The UKI method is simulated for 500 steps. For all methods, the results are presented based on 20,000 samples.

\subsection{Preconditioned Crank-Nicolson (pCN) MCMC}\label{sec:pcn}
We implement the preconditioned Crank-Nicolson (pCN) MCMC method, a standard sampler whose efficiency does not degrade with parameter dimension. Assuming a Gaussian prior $\mb{m} \sim \mathcal{N}(\mb{m}_{0}, \mb{C}_{0})$, the proposal rule is
\begin{equation}
    \mb{m}^* = \sqrt{1 - \beta^2}\,(\mb{m}_{n} - \mb{m}_{0}) + \mb{m}_{0} + \beta \mb{\eta}, \quad \mb{\eta} \sim \mathcal{N}(\mb{0}, \mb{C}_0),
    \label{eq:pcn_proposal}
\end{equation}
with step size $\beta \in (0,1]$. Since the proposal preserves the prior measure, the acceptance probability depends only on the likelihood:
\begin{equation}
    a(\mb{m}_{n}, \mb{m}^{*}) = \min\bigl(1,\; \exp\bigl(\Phi(\mb{m}_{n}) - \Phi(\mb{m}^*)\bigr)\bigr),
    \label{eq:pcn_accept}
\end{equation}
where $\Phi$ is the negative log-likelihood. We set $\beta = 0.1$ and run 5{,}000 iterations of the high-fidelity FDM solver, discarding the first $20\%$ as burn-in and thinning by a factor of 10.

\subsection{UKI}\label{uki}
We implement the UKI algorithm of \cite{gao2024adaptive}, based on the stochastic dynamical system
\begin{align}
    \text{Evolution:} \quad & \mb{m}_{n+1} = \mb{r}_{0} + \alpha(\mb{m}_{n} - \mb{r}_{0}) + \mb{\omega}_{n+1}, \quad &\mb{\omega}_{n+1} \sim \mathcal{N}(\mb{0}, \mb{\Sigma_{\omega}}), \nonumber \\
    \text{Observation:} \quad & \mb{y}_{n+1} = \mathcal{G}(\mb{m}_{n+1}) + \mb{\eta}_{n+1}, \quad &\mb{\eta}_{n+1} \sim \mathcal{N}(\mb{0}, \mb{\Sigma_{\eta}}),
    \label{eq:uki_model}
\end{align}
where $\mb{m}_{n+1}$ is the unknown parameter vector, $\mb{y}_{n+1}$ the observation, $\mb{\omega}_{n+1}$ and $\mb{\eta}_{n+1}$ are mutually independent zero-mean Gaussian errors, $\mb{r}_{0}$ is an arbitrary initial anchor, and $\alpha \in (0,1]$ is the regularization parameter, set to $\alpha = 0.5$.

\subsection{SVGD with Adaptive Surrogate Fine-tuning}\label{SVGD_adap}
We implement Stein Variational Gradient Descent (SVGD)~\cite{liu2016stein, liu2017stein} with the radial basis function (Gaussian) kernel
\begin{equation}
k(\mb{x}, \mb{x}') = \exp\!\left(-\|\mb{x} - \mb{x}'\|^2 / h\right),
\end{equation}
whose bandwidth $h$ is set via the median heuristic $h = \mathrm{med}^2 / \log n$, with $\mathrm{med}$ the median of pairwise particle distances and $n$ the number of particles. We run 20 adaptive stages, alternating 50 SVGD updates with FNO fine-tuning at each stage; FNO fine-tuning uses perturbation strength $\gamma = 3$ (Eq.~\eqref{eq:perturbation}) for adequate local coverage.

\bibliographystyle{siamplain}
\bibliography{tang}
\end{document}